%% file: ijcai23-paper.tex
\def\thetitle{Tractable Diversity: \\ Scalable Multiperspective Ontology Management via Standpoint $\EL$}
\def\theauthors{%
	\mbox{Lucía Gómez Álvarez} and \mbox{Sebastian Rudolph} and \mbox{Hannes Strass}
}

\def\theaffiliation{Computational Logic Group, TU~Dresden, Germany}
\newcommand{\email}[1]{\href{mailto:#1}{\texttt{#1}}}
\def\theemails{%
	\email{\{lucia.gomez\_alvarez},
	\email{sebastian.rudolph},
	\email{hannes.strass\}@tu-dresden.de}%
}

\documentclass{article}
\usepackage[switch,displaymath,mathlines]{lineno}
\usepackage{ijcai23}

\usepackage{times}
\usepackage{soul}
\usepackage{url}
\usepackage[hidelinks]{hyperref}
\usepackage[utf8]{inputenc}
\usepackage[small]{caption}
\usepackage{booktabs}
\usepackage{algorithm}
\usepackage{algorithmic}

\usepackage[utf8]{inputenc}
\usepackage[british]{babel}
\usepackage[usenames,x11names]{xcolor}
\usepackage{relsize,setspace}
\usepackage{stmaryrd}
\usepackage{url}
\usepackage[font=it]{caption}
\usepackage{tikz,pgf}\usetikzlibrary{shapes,backgrounds}
\usepackage{graphicx}
\usepackage{tikz}

\usepackage{amsmath,amssymb}
\usepackage[mathscr]{euscript}

\usepackage{pifont}
\usepackage[amsthm,thmmarks]{ntheorem}

\usepackage{cleveref}  % allows to ref without putting the environment name
\usepackage{multirow}
\usepackage{comment}
\usepackage{xspace}

\usepackage{enumitem}
\setlist{leftmargin=*,noitemsep,parsep=0.2ex,topsep=0.3ex}
\usepackage{scalerel}

\usepackage[normalem]{ulem}

\usepackage{listings}% Allows to write code sections
\usepackage{mdframed}% Frames for the tableau rules

\usepackage{nimbusmononarrow} % a better distinguishable \mathtt

\newcommand{\shadydiamond}{\mbox{{\rotatebox[origin=c]{45}{\small\ding{113}}}}}

\theoremseparator{.}
\theoremheaderfont{\normalfont\bfseries}
\theorembodyfont{\slshape}
\newtheorem{theorem}{Theorem}
\newtheorem{lemma}[theorem]{Lemma}
\newtheorem{proposition}[theorem]{Proposition}

\newtheorem{numberclaim}{Claim}
\theorembodyfont{\upshape}
\theoremsymbol{\ensuremath{\shadydiamond}}
\newtheorem{definition}{Definition}
\newtheorem{example}{Example}
\theoremstyle{nonumberplain}
\theoremheaderfont{\normalfont\bfseries}
\theorembodyfont{\slshape}

\theoremstyle{nonumberplain}
\theoremheaderfont{\normalfont\itshape}
\theoremsymbol{\ensuremath{\Box}}
\newtheorem{proofsketch}{Proof sketch} % note that this does not care \iflong
\theoremstyle{nonumberplain}
\theoremheaderfont{\normalfont\itshape}
\theoremsymbol{\ensuremath{\shadydiamond}}
\newtheorem{claimproof}{Proof of the claim}

\usepackage{ifthen}
\usepackage{environ}
\newif\iflong
\NewEnviron{shortproof}{\iflong\else\expandafter\begin{proof}\BODY\end{proof}\fi}
\NewEnviron{shortproofsketch}{\iflong\else\expandafter\begin{proofsketch}\BODY\end{proofsketch}\fi}
\NewEnviron{longproof}[1][]{\iflong\expandafter\begin{proof}[#1]\BODY\end{proof}\fi}
\NewEnviron{supplementary}{\iflong\expandafter\color{teal}\BODY\color{black}\fi}
\NewEnviron{shortonly}{\iflong\else\expandafter\BODY\fi}
\NewEnviron{longonly}{\iflong\expandafter\BODY\fi}

\usepackage[textwidth=1.5cm]{todonotes}
\input{macros/davidMacros}

\input{macros/ThesisSymbols}
\input{macros/mymacros}

\input{macros/dlmacros}

\title{\thetitle}

\author{
\theauthors
\affiliations
\theaffiliation
\emails
\theemails
}

\longfalse

\begin{document}

\maketitle

% main content goes here
% approx 1 page including normalisation

\newenvironment{aligna}{\linenomathNonumbers\begin{align*}}{\end{align*}\endlinenomath}

\input{sections/abstract.tex}

% approx 2 pages
\section{Introduction}
%Already discussing what Standpoint logic is, plus the motivation to have Standpoint EL specifically. SNOMED...

\input{sections/introduction.tex}

\section{Syntax, Semantics, and Normalisation}\label{sec:syntax-semantics}

\input{sections/syntax-semantics.tex}

\input{sections/reasoning-problems.tex}

\input{sections/normal-form.tex}

%%%%% This content is provisionally here until the section is written. The tableau section assumes that this has been introduced in this section.

%We use $C, D, \ldots \in \CC$ to denote concepts, $\phi, \psi, \ldots \in \forC$ to denote formulas, $a, b, \ldots \in \objC$ to denote individuals and $\Phi, \Psi, \ldots \in \forC\cup\CC\cup\objC$ to denote either.

%\begin{definition}[Size]
%    Let us define the size of a knowledge base $\kb$, denoted by $\size{\kb}$ as
%    $$ \size{\kb} \eqdef \card{\forC} + \card{\CC} + \card{\SC} $$
%\end{definition}

%\begin{definition}[Normal form]
%
%    Let us define the assumed normal form of the formulas of a knowledge base $\kb$ as
%    \begin{align*}
%        \st          & \preceq \sp      \\
%        \standb{s}	[	C & \sqsubseteq  D ] \\
%        \standb{s}	[A & (x)]             \\
%        \standb{s}	[R & (x,y)]
%    \end{align*}

%    Where $C$ is of the form $A$, $\exists R.A$ or $A \sqcap B$ with $A, B \in\BCC$ and $D$ is of the form $E$,  $\exists R.E$, $\standds E$ or $\standbs E$ with $E \in\BCC\cup\{\bot\}$
%\end{definition}

\section{A Tableau Algorithm for Standpoint \EL}\label{sec:tableau}

\input{sections/tableau-algorithm.tex}

% approx 3/4 page
\section{Intractable Extensions}\label{sec:intractable-extensions}

\input{sections/intractable-extensions.tex}

\section{Conclusion and Future Work}\label{sec:conclusion}

\input{sections/conclusion.tex}

%approx 1/4
\bibliographystyle{named}
\bibliography{bib/references}
\end{document}

%% file: macros/davidMacros.tex
\newcommand{\FormatArithmeticComplexityClass}[1]{\ensuremath{\textsc{#1}}\xspace}

\newcommand{\PTime}{\FormatArithmeticComplexityClass{PTime}}
\newcommand{\NP}{\FormatArithmeticComplexityClass{NP}}

\newcommand{\ExpTime}{\FormatArithmeticComplexityClass{ExpTime}}
\newcommand{\NExpTime}{\FormatArithmeticComplexityClass{NExpTime}}

\newcommand{\TwoExpTime}{\mbox{\sc{2\ExpTime}}\xspace}

\newcommand{\coNP}{\mbox{\sc{co\NP}}\xspace}

%% file: macros/ThesisSymbols.tex
\newcommand{\pred}[1]{\mbox{\small\tt {#1}}}

\newcommand{\spform}[1]{\mathsf{#1}}
\newcommand{\spsub}[2]{\spform{#1}_{#2}}

\newcommand{\mathcom}[3]{ \newcommand{#1}[#2]{\mbox{$#3$}}}
\mathcom{\imp}{0}{\ \rightarrow\ }            % implication arrow
\mathcom{\rimp}{0}{\ \leftarrow\ }            % implication arrow

\mathcom{\con}{0}{\ \wedge\ }                 % conjunction
\mathcom{\dis}{0}{\ \vee\ }                   % disjunction
\mathcom{\n}{0}{\neg}                     % negation
\mathcom{\dimp}{0}{\ \leftrightarrow\ }       % mat equiv
\mathcom{\corresponds}{0}{\ \Lleftarrow\! \! \Rrightarrow\ }
\mathcom{\A}{0}{\forall}                  % universal quantifier
\mathcom{\E}{0}{\exists}     % existential quant

\def\Box{\mathop\square}
\def\Diamond{\mathop\lozenge}

\mathcom{\tuple}{1}{\left\langle #1 \right\rangle}
\mathcom{\stuple}{1}{\langle #1 \rangle}

\def\eqdef{\mathrel{\ =_{\mbox{\em \tiny def}}\ }}

\def\sp{\hbox{$\spform{s'}$}\xspace}

\def\st{\hbox{$\spform{s}$}\xspace}

\def\star{\hbox{$*$}\xspace}

\def\pr{\pi}

\def\Precs{\Pi}

\newcommand{\standb}[1]{\mathord{\Box\nolimits_{\spform{#1}}}}
\newcommand{\standd}[1]{\mathord{\Diamond\nolimits_{\spform{#1}}}}
\newcommand{\standdsub}[2]{\mathord{\Diamond\nolimits_{\spsub{#1}{#2}}}}

\def\standv#1{\mathord{\mathop{\odot}\nolimits_{\spform{#1}}}}

\def\standbx#1{\mathord{\Box\nolimits_{\scaleobj{0.8}{\spform{#1}}}}}
\def\standdx#1{\mathord{\Diamond\nolimits_{\scaleobj{0.8}{\spform{#1}}}}}

\def\standbs{\standb{s}}
\def\standbu{\standb{u}}
\def\standbv{\standb{v}}
\def\standbsp{\standb{s'}}

\def\standds{\standd{s}}
\def\standdu{\standd{u}}

\def\standdvo{\standdsub{v}{0}}
\def\standdvi{\standdsub{v}{1}}
\def\standdsp{\standd{s'}}

\def\standvs{\standv{s}}
\def\standvu{\standv{u}}
\def\standvsp{\standv{s'}}

\def\standball{\standb{*}}

\def\standdall{\standd{*}}

\xspace
\def\f\xspacestandtopre{\hbox{$\sigma\,$}\xspace}
\def\fpretov\xspacealue{\hbox{$\delta\,$}\xspace}

\def\ModSat#1||-#2{#1\models #2}
\def\NotModSat#1||-#2{#1\nvDash #2}

%% file: macros/mymacros.tex
\newcommand{\define}[1]{\emph{#1}}
\newcommand{\N}{\mathbb{N}}
\newcommand{\set}[1]{\left\{#1\right\}}
\newcommand{\guard}{\;\middle\vert\;}

\newcommand{\size}[1]{\left\lVert#1\right\rVert}

\renewcommand{\eqdef}{\mathrel{\,:=\,}}
\newcommand{\iffdef}{\mathrel{\;:\mathrel{\mkern-6mu}{\Longleftrightarrow}\;}}
\newcommand{\pheq}{\mathrel{\phantom{=}}}
\newcommand{\phiff}{\mathrel{\phantom{\iff}}}
\newcommand{\phimplies}{\mathrel{\phantom{\implies}}}
\newcommand{\ebnfeq}{\mathrel{::=}}
\newcommand{\ebnfalt}{\mathrel{\;\mid\;}}

\renewcommand{\land}{\mathrel{\wedge}}

\renewcommand{\S}{\mathcal{S}}

\def\Stands{N_{\mathsf{S}}}
\def\Concepts{N_{\mathsf{C}}}
\def\Roles{N_{\mathsf{R}}}
\def\Individuals{N_{\mathsf{I}}}
\def\sts{\spform{s}} % standpoint symbol
\def\stsp{\spform{s'}} % standpoint symbol
\def\stu{\spform{u}} % standpoint symbol
\def\stv{\spform{v}} % standpoint symbol
\def\stvb{\spsub{v}{b}} % standpoint symbol with subscript
\def\stvo{\spsub{v}{0}} % standpoint symbol with subscript
\def\stvi{\spsub{v}{1}} % standpoint symbol with subscript
\def\prvo{\pr'_{\stvo}}
\def\prvi{\pr'_{\stvi}}
\def\prvb{\pr'_{\stvb}}
\def\prvy{\pr'_{\spsub{v}{y}}}

\def\E{\mathcal{E}}

\def\struct{\mathcal{I}}
\def\intf{\cdot^{\struct}}
\def\Dom{\Delta}
\def\de{\delta}
\def\ve{\varepsilon}
\def\ze{\zeta}
\newcommand{\interpret}[2]{#1^{#2}}
\newcommand{\interprets}[1]{\interpret{#1}{\struct}}

\newcommand{\interpretgp}[1]{\interpret{#1}{\gamma(\pr)}}
\newcommand{\interpretgps}[1]{\interpret{#1}{\gamma(\pr_*)}}
\newcommand{\interpretpgp}[1]{\interpret{#1}{\gamma'(\pr)}}

\newcommand{\interpretpgpp}[1]{\interpret{#1}{\gamma'(\pr')}}
\newcommand{\interpretpgppp}[1]{\interpret{#1}{\gamma'(\pr'')}}

\def\dland{\sqcap}

\def\dlsub{\sqsubseteq}
\newcommand{\T}{\mathcal{T}}
\renewcommand{\S}{\mathcal{S}}
\newcommand{\C}{\mathcal{C}}
\newcommand{\CT}{\mathcal{C}_\T}
\newcommand{\CTp}{\mathcal{C}_{\T'}}
\newcommand{\CTpp}{\mathcal{C}_{\T''}}
\renewcommand{\A}{\mathcal{A}}
\newcommand{\K}{\mathcal{K}}
\newcommand{\dlstruct}{\mathfrak{D}}

\def\EL{\ensuremath{\mathcal{E\!L}}\xspace}

\newcommand{\clause}{\mathbf{C}}

\newcommand{\goesto}{\mathbin{\qquad\longrightarrow\qquad}}

\newenvironment{firstexample}[2]{
\ifthenelse{\equal{#2}{}}{\begin{example}}{\begin{example}[#2]}
    \label{#1}
    }{
  \end{example}
}

\newenvironment{cexample}[1]{
  \begin{example}[Continued from \Cref{#1}]
    }{
  \end{example}
}

%% file: macros/dlmacros.tex
\newcommand{\SROIQbs}{\ensuremath{\mathcal{SROIQ}b_s}\xspace}

\RequirePackage{graphicx}

\newcommand{\kb}{\text{\rm{KB}}} % the generic knowledge base
\newcommand{\KB}{\text{$\mathcal{K}$}\xspace}
\usepackage{ upgreek }

\def\kb{\mathcal{K}}
\def\SEL{\mathbb{S}_{\EL}}
\def\SC{\mathsf{ST}_{\kb}}
\def\forC{\mathsf{SF}_{\kb}}
\def\SFCT{\mathsf{SC}_{\T}}
\def\sub{\mathsf{sub}}
\def\objC{\mathsf{IN}_{\kb}}
\def\BCC{\mathsf{BC}_{\kb}}
\def\CC{\mathsf{C}_{\kb}}
\def\int{^{\mathcal{I}}}
\def\qelem{\varepsilon}
\def\qelemg{g}
\newcommand{\cgfont}[1]{\mathscr{#1}}
\def\consisfq{\cgfont{S}^{q}}
\def\consisf{\cgfont{S}}
\def\consisfu{\cgfont{S}}
\newcommand{\consis}[1]{\cgfont{S}(#1)}

\newcommand{\consisu}[1]{\cgfont{S}(#1)}
\newcommand{\consisq}[1]{{\cgfont{S}^{q}}(#1)}
\newcommand{\elabel}[1]{\cgfont{L}(#1)}
\newcommand{\elabelq}[1]{\cgfont{L}^{q}(#1)}
\def\elabelf{\cgfont{L}}
\def\runsf{\Upgamma}
\def\rolesf{\cgfont{R}}
\def\initialcs{S_0^{\kb}}
\newcommand{\stlabel}[2]{\mathsf{st}_{#1}(#2)}

\def\crel{\mu}
\newcommand{\cfunc}[2]{\crel_{#1}(#2)}
\newcommand{\cfuncf}[1]{\crel_{#1}}

\newcommand{\col}{{\hspace{1.5pt}:\hspace{2pt}}}

%% file: sections/abstract.tex
\begin{abstract}
    The tractability of the lightweight description logic
    $\EL$ has allowed for the construction of large and widely used ontologies that support semantic interoperability.
    However, comprehensive domains with a broad user base are often at odds with strong axiomatisations otherwise useful for inferencing, since these are usually context dependent and subject to diverging perspectives.

    In this paper we introduce \emph{Standpoint $\EL$}, a multi-modal extension of $\EL$ that allows for the integrated representation of domain knowledge relative to diverse, possibly conflicting \emph{standpoints} (or contexts), which can be hierarchically organised and put in relation with each other. We establish that \emph{Stand\-point $\EL$} still exhibits $\EL$'s favourable \PTime standard reasoning, whereas introducing additional features like empty standpoints, rigid roles, and no-\\minals makes standard reasoning tasks %\hsn{More precisely, we only show this for knowledge base satisfiability and haven't shown reducibility from other standard reasoning tasks for the language extensions.}\sru{It is not too hard to see that the established reductions work the same for the extensions. So I'd vote for keeping it as is.}\hsn{OK.}
    intractable.
\end{abstract}

%% file: sections/introduction.tex
%Schema: 
% EL is very useful
In many subfields of artificial intelligence, ontologies are used to provide a formal representation of a shared vocabulary, give meaning to its terms, and describe the relations between them.
To this end, one of the most prominent and successful class of logic-based knowledge representation formalisms are \emph{description logics} %\hsn{I would capitalise either both or none, same for standpoint logic.}\sru{I'd say lower case for classes of logics, but upper for concrete logics, i.e., ``Standpoint \EL''}\hsn{Agree, will do.}
(DLs) \cite{baader_horrocks_lutz_sattler_2017,Rudolph11}, which provide the formal basis for most recent version of the Web Ontology Language OWL~2 \cite{owl2-overview}.

Among the most widely used families of DLs used today is $\EL$ \cite{Baader05ELenvelope}, which is the formal basis of OWL~2~EL \cite{owl2-profiles}, a popular tractable profile of OWL~2.
One of the main appeals of $\EL$ is that basic reasoning tasks can be performed in polynomial time with respect to the size of the ontology, enabling reasoning-supported creation and maintenance of very large ontologies. An example of this is the healthcare ontology SNOMED CT \cite{donnelly2006snomed}, with worldwide adoption %, including institutions such as the National Health Service of the United Kingdom (NHS), 
and a broad user base comprising clinicians, patients, and researchers.

However, when modelling comprehensive ontologies like SNOMED CT, one is usually facing issues related to context or perspective-dependent knowledge as well as ambiguity of language \cite{schulz2017lexical}. % that preserve consistency and generality, which limits the exploitation of the reasoning power of knowledge representation systems. (Word it with context/perspective dependence) In practice, different applications of these terminologies often require different interpretations of their meanings, which may involve different (possibly conflicting) axiomatisations. 
For instance, the concept $\pred{Tumour}$ might denote a process or a piece of tissue; $\pred{Allergy}$ may denote an allergic reaction or just an allergic disposition.

%%% USE THIS SENTENCE ->
%ontologies can be used only as long as consensus about their contents is reached
In a similar vein, the decentralised nature of the Semantic Web has led to the generation of various ontologies of overlapping knowledge that inevitably reflect different points of view. For instance, an initiative has attempted to integrate the FMA1140 (Foundational Model of Anatomy), SNOMED-CT, and the NCI (National Cancer Institute Thesaurus) into a single combined version called LargeBio and reported ensuing challenges \cite{Osman2021OntologyIssues}. In this context, frameworks supporting the integrated representation of multiple perspectives seem preferable to recording the distinct views in a detached way, but also to entirely merging them at the risk of causing inconsistencies or unintended consequences.

To this end, \citeauthor{gomez2021standpoint} [\citeyear{gomez2021standpoint}] proposed \emph{standpoint logic}, a formalism inspired by the theory of supervaluationism \cite{Fine1975} and rooted in modal logic, which allows for the simultaneous representation of multiple, potentially contradictory, viewpoints in a unified way and the establishment of alignments between them. This is achieved by extending the base language with labelled modal operators, where
propositions  $\standbx{S}\phi$ and $\standdx{S}\phi$ express information relative to the \emph{standpoint} $\mathsf{S}$ and read, respectively: ``according to $\mathsf{S}$, it is \emph{unequivocal/conceivable} that $\phi$''.
Semantically, standpoints are represented by sets of \emph{precisifications},\footnote{Precisifications are analogous to the \emph{worlds} of modal-logic frameworks with possible-worlds semantics.} such that $\standbx{S}\phi$ and $\standdx{S}\phi$ hold if $\phi$ is true in all/some of the precisifications associated with $\mathsf{S}$. Consider the following example.
%%%%%%%%%%%%%%%%%%%%%%%%%%%%%%%%%%%%%%% 3 -- For Example:  %%%%%%%%%%%%%%%%%%%

\begin{firstexample}{example:fol}{Tumour Disambiguation}

    Two derivatives of the SNOMED CT ontology {\small$(\mathsf{SN})$} model tumours differently. According to {\small$\mathsf{TP}$}, a $\pred{Tumour}$ is a process by which abnormal or damaged cells grow and multiply (\ref{formula:TP-main}), yet according to {\small$\mathsf{TT}$}, a $\pred{Tumour}$ is a lump of tissue (\ref{formula:TT-main}).
    \begin{linenomath*}
        \begin{align}
            \standbx{TP}[\pred{Tumour}          & \sqsubseteq\pred{AbnormalGrowthProcess}]\label{formula:TP-main} \\
            \small   \standbx{TT}[\pred{Tumour} & \sqsubseteq\pred{Tissue}]\label{formula:TT-main}
        \end{align}
    \end{linenomath*}
    \noindent
    Both interpretations inherit the axioms of the original SNOMED CT (\ref{formula:SN-TT-TP-preceqs}) and are such that if according to {\small$\mathsf{SN}$} something is arguably both a $\pred{Tumour}$ and a $\pred{Tissue}$, then it (unequivocally) is a $\pred{Tumour}$ according to {\small$\mathsf{TT}$} (\ref{formula:SN-diamond-to-TT-box}). The respective assertion is made for {\small$\mathsf{TP}$} (\ref{formula:SN-diamond-to-TP-box}). But $\pred{Tissue}$ and $\pred{Process}$ are disjoint categories according to  {\small$\mathsf{SN}$} (\ref{formula:SN-tissue-process-disjoint}).

    \vspace{-2ex}
    \begin{linenomath*}
        {\small
            \begin{align}
                (\mathsf{TP}\preceq\mathsf{SN}) \quad                  & \quad (\mathsf{TT}\preceq\mathsf{SN})\label{formula:SN-TT-TP-preceqs}      \\
                \standdx{SN}[\pred{Tumour}\sqcap\pred{PhysicalObject}] & \sqsubseteq\standbx{TT}[\pred{Tumour}]\label{formula:SN-diamond-to-TT-box} \\
                \standdx{SN}[\pred{Tumour}\sqcap\pred{Process}]        & \sqsubseteq\standbx{TP}[\pred{Tumour}]\label{formula:SN-diamond-to-TP-box} \\
                \standbx{SN}[\pred{Tissue}\sqcap\pred{Process}         & \sqsubseteq\bot]\label{formula:SN-tissue-process-disjoint}
            \end{align}}
    \end{linenomath*}%   
    \noindent
    While clearly incompatible, both perspectives are semantically close and we can establish relations between them. For instance, we might assert that something is unequivocally the product of a $\pred{Tumour}$ (process) according to {\small$\mathsf{TP}$} if and only if it is arguably a $\pred{Tumour}$ (tissue) according to {\small$\mathsf{TT}$} (\ref{formula:TT-TP-bridge}).
    Or we may want to specify a subsumption between the classes of unequivocal instances of $\pred{Tissue}$ according to {\small$\mathsf{TT}$} and to   \begin{linenomath*}
        {\small$\mathsf{TP}$} (\ref{formula:TT-TP-tissue-bridge}).
        \begin{align}
            \standbx{TP}[\exists\pred{ProductOf.Tumour}] \equiv \standdx{TT}[\pred{Tumour}]\label{formula:TT-TP-bridge} \\
            \standbx{TT}[\pred{Tissue}]\sqsubseteq\standbx{TP}[\pred{Tissue}]\label{formula:TT-TP-tissue-bridge}
        \end{align}
    \end{linenomath*}
    \noindent
    When recording clinical findings, clinicians may use ambiguous language, so an automated knowledge extraction service may obtain the following from text and annotated scans:%\hsn{Why semicolons? Or is this not an ABox? Maybe we should also briefly justify using the diamond in front of the whole expression, as this topic is only mentioned in the conclusion.}\sru{I'd prefer commas, maybe followed by a bit of extra space for better visual distiction from the in-atom commas.}
    \begin{linenomath*}
        \begin{align}
            \small \standbx{SN}\! & \small \set{ \pred{Patient}(p1),\,\pred{HasPart}(p1,a),\,\pred{Colon}(a) }\label{formula:instance-patient1}     \\
            \small \standdx{SN}\! & \small \set{ \pred{HasPart}(a,b),\,\pred{Tumour}(b),\,\pred{PhysicalObject}(b) }\label{formula:instance-tumour}
        \end{align}
    \end{linenomath*}
\end{firstexample}
%%%%%%%%%%%%%%%%%%%%%%%%%%%%%%%%%%%%%%% 5 -- The Standpoint Approach:  %%%%%%%%%%%%%%%%%%%

%\emph{Standpoint logic} \cite{gomez2021standpoint} is a formalism inspired by the theory of supervaluationism \cite{Fine1975} and rooted in modal logic that supports the coexistence of multiple standpoints and the establishment of alignments between them, by extending the base language with labelled modal operators.
%Propositions  $\standbx{SN}\phi$ and $\standdx{SN}\phi$ express information relative to the \emph{standpoint} $\mathsf{SN}$ and read, respectively: ``according to $\mathsf{SN}$, it is \emph{unequivocal/conceivable} that $\phi$''.
%In the semantics, standpoints are represented by sets of \emph{precisifications},\footnote{Precisifications are analogous to the \emph{worlds} of frameworks with possible-worlds semantics.} such that $\standbx{SN}\phi$ and $\standdx{SN}\phi$ hold if $\phi$ is true in all/some of the precisifications associated with $\mathsf{SN}$.

The logical statements {(\ref{formula:TP-main})--(\ref{formula:instance-tumour})}, which formalise \Cref{example:fol} by means of a stand\-point-enhanced \EL description logic, are not inconsistent, so all axioms can be jointly represented. Let us now illustrate the use of standpoint logic for reasoning with and across individual perspectives.

\begin{cexample}{example:fol}
    In this case, we can disambiguate the information given by Axiom (\ref{formula:instance-tumour}) using Axiom (\ref{formula:SN-TT-TP-preceqs}) and Axiom (\ref{formula:SN-diamond-to-TT-box}), which entail that according to $\small\mathsf{TT}$, $b$ is unequivocally a tumour, $\standbx{TT}\pred{Tumour}(b)$, and with Axiom (\ref{formula:TT-main}) also a tissue, $\standbx{TT}\pred{Tissue}(b)$.
    Moreover, we can use the ``bridges'' to switch to another perspective. From Axiom (\ref{formula:TT-TP-tissue-bridge}), it is clear that according to $\small\mathsf{TP}$, $b$ is also a tissue, $\standbx{TP}\pred{Tissue}(b)$, and from Axiom (\ref{formula:TT-TP-bridge}) $b$ is the product of a tumour, $\standbx{TP}\exists\pred{ProductOf.Tumour}(b)$. Then Axiom (\ref{formula:TP-main}) yields
    \begin{linenomath*}
        $$\standbx{TP}\exists\pred{ProductOf}.(\pred{Tumour}\sqcap \pred{AbnormalGrowthProcess})(b).$$
    \end{linenomath*}
    The statement $\standbx{SN}[\pred{Tumour}\sqcap\pred{Process}](d)$, in contrast, will trigger an inconsistency thanks to Axiom (\ref{formula:SN-tissue-process-disjoint}), which prevents the evaluation of $\pred{Tumour}$ simultaneously as a $\pred{Tissue}$ and a $\pred{Process}$ and Axiom (\ref{formula:TT-main}), which states that according to some interpretations, a $\pred{Tumour}$ is a $\pred{Tissue}$.
    Finally, a user (e.g.\ a specific clinic, $\small\mathsf{CL}$) may inherit the SNOMED CT $(\small\mathsf{CL}\preceq\mathsf{SN})$ and establish further axioms, e.g,
    \begin{linenomath*}
        \begin{align*}
            \standbx{CL}[\pred{Patient}\sqcap\exists\pred{HasPart.}(\pred{Colon}\sqcap\standdx{SN}\exists\pred{HasPart.}\pred{Tumour})\sqsubseteq & \\
            \exists\pred{AssociatedWith.ColonCancerRisk}],                                                                                        &
        \end{align*}
    \end{linenomath*}
    \noindent to identify patients with cancer risk.
    Here, one can infer with Ax.~(\ref{formula:instance-patient1}) that $\standbx{CL}\exists\pred{AssociatedWith.ColonCancerRisk}(p1)$.
\end{cexample}

The need of handling multiple perspectives in the Semantic Web has led to several (non-modal) logic-based pro\-posals.
The closest regarding goals are multi-viewpoint ontologies \cite{Hemam2011MVP-OWL:Web,Hemam2018Probabilistic}, which model the intuition of viewpoints in a tailored extension of OWL for which no complexity bounds are given.
Similar problems are also addressed in the more extensive work on contextuality (e.g.\ C-OWL and Distributed ontologies \cite{Bouquet2003C-OWL:Ontologies,Borgida2003DistributedSources} and the Contextualised Knowledge Repository (CKR) \cite{SERAFINI201264}).
These frameworks focus on contextual and distributed reasoning and range between different levels of expressivity for modelling the structure of contexts and the bridges between them.
%While most of them target more expressive description logics, CKR is suitable for scalable (distributed) reasoning and counts with implementations that provide support for OWL2-RL based CKR defeasible reasoning \cite{BOZZATO201872}.
In the context of scalable reasoning, one should highlight the implementations that provide support for OWL2-RL based CKR defeasible reasoning \cite{BOZZATO201872}.

As for modal logics, their suitability to model perspectives and contexts in a natural way is obvious \cite{Klarman2016DescriptionContext,gomez2021standpoint}, they are well-known in the community and their semantics is well-understood. Yet, the interplay between DL constructs and modalities is often not well-behaved and can easily endanger the decidability of reasoning tasks or increase their complexity \cite{baader1995multi,mosurovic1999complexity,WolterMultiDimensionalDescriptionLogics}. Notable examples are $\NExpTime$-completeness of the multi-modal description logic $\mathbf{K}_\mathcal{ALC}$ \cite{LutzSWZ02} % , and it should be possible to adapt it to $\kdfourfive$ and $\sfive$, which are more suitable systems of modal logics for multiperspective representations. 
and $\TwoExpTime$-completeness of $\mathcal{ALC}_{\mathcal{ALC}}$ \cite{Klarman2016DescriptionContext}, a modal contextual logic framework in the style proposed by McCarthy \cite{mccarthy1997context}.

In this work, we focus on the framework of \emph{standpoint logics} \cite{gomez2021standpoint}, %, specifically tailored for multiperspective knowledge representation,
which are modal logics, too, but come with a simplified Kripke semantics. % that  
%Gómez Álvarez and Rudolph introduced the standpoint framework over a 
%of First Order Standpoint Logic (FOSL)% on a propositional logic base \cite{gomez2021standpoint} (standard reasoning tasks are \textsc{NP}-complete just like for plain propositional logic). 
Recently, \citeauthor{gomez-alvarez22howtoagree}~[\citeyear{gomez-alvarez22howtoagree}] introduced \emph{First-Order} Standpoint Logic (FOSL) and showed favourable complexity results for its \emph{sentential} fragments,\footnote{This includes the sentential standpoint variant of the expressive DL $\SROIQbs$, a logical basis of OWL~2~DL \cite{owl2-semantics}.} which disallow modal operators being applied to formulas with free variables.
In particular, adding sentential standpoints does not increase the complexity for fragments that are \emph{$\NP$-hard}.
Yet, a fine-grained terminological alignment between different perspectives requires concepts preceded by modal operators, as in Axiom~(\ref{formula:TT-TP-bridge}), leading to non-sentential fragments of FOSL.

Our paper is structured as follows.
After introducing the syntax and semantics of Standpoint $\EL$ ($\SEL$) and a suitable normal form (\Cref{sec:syntax-semantics}), we establish our main result: satisfiability checking in $\SEL$ is $\PTime$.
We show this by providing a worst-case optimal tableau-based algorithm (\Cref{sec:tableau}) that takes inspiration in the \emph{quasi-model} based methods \cite{Wolter98SatisfiabilityDescriptionLogicsModalOperators} as used for $K_\mathcal{ALC}$ \cite{LutzSWZ02}, but differs in its specifics.
Our approach builds a \emph{quasi-model} from a graph of \emph{(quasi) domain elements}, which are annotated with various constraints, to then reconstruct the worlds or, in our case, precisifications.
We also show that introducing additional features such as empty standpoints, rigid roles, and nominals make standard reasoning tasks intractable (\Cref{sec:intractable-extensions}).
In \Cref{sec:conclusion}, we conclude the paper with a discussion of future work, including efficient approaches for reasoner implementations.
Altogether, this paper provides a clear pathway for making scalable multiperspective ontology management possible.

An extended version of the paper with proofs of all results is available as a technical appendix.

%% file: sections/syntax-semantics.tex
We now introduce syntax and semantics of Standpoint $\EL$ (referred to as $\SEL$) and propose a normal form that is useful for subsequent algorithmic considerations.

\subsubsection{Syntax}

A \emph{Standpoint DL vocabulary} is a traditional DL vocabulary consisting of
sets $\Concepts$ of \define{concept names},
$\Roles$ of \define{role names}, and
$\Individuals$ of \define{individual names}, extended it by an additional set $\Stands$ of \define{standpoint names} with \mbox{$\star\in\Stands$}.
A \emph{standpoint operator} is of the form  $\standds$ (``diamond'') or $\standbs$ (``box'') with \mbox{$\sts\in\Stands$}; we use $\standvs$ to refer to either.
A \define{concept term} is defined via
\begin{linenomath*}    
\[
    C \ebnfeq \top \ebnfalt \bot \ebnfalt A \ebnfalt C_1\dland C_2 \ebnfalt \exists R.C \ebnfalt \standvs[C]
\]
\end{linenomath*}    
where \mbox{$A\in\Concepts$} and \mbox{$R\in\Roles$}.
A \define{general concept inclusion (GCI)} is of the form
\mbox{$\standvs[C \dlsub D]$},
where $C$ and $D$ are concept terms.\footnote{The square brackets $[\ldots]$ indicate the scope of the modality, as the same modalities may be used inside concept terms.}
%We sometimes abbreviate the GCI $\standball[C\dlsub D]$ as $C\dlsub D$.
%
A \define{concept assertion} is of of the form $\standvs[C(a)]$ while
a \define{role assertion} is of the form $\standvs[R(a,b)]$,
where \mbox{$a,b\in\Individuals$}, $C$ is a concept term, and \mbox{$R\in\Roles$}.
A \define{sharpening statement} is of the form
\mbox{$\sts \preceq \sts'$}
where \mbox{$\sts,\sts'\in\Stands$}.

A \define{$\SEL$ knowledge base} is a tuple \mbox{$\K=\tuple{\S,\T,\A}$},
where
$\T$ is a set of GCIs, called \emph{TBox};
$\A$ is a set of (concept or role) assertions, called \emph{ABox}; and
%additionally,
$\S$ is a set of sharpening statements, called \emph{SBox}.
We refer to arbitrary statements from $\K$ as \emph{axioms}.
Since the axiom types in $\S$, $\T$, and $\A$ are syntactically well-distinguished% 
%and there is no danger of confusion
, we sometimes identify $\K$ as \mbox{$\S \cup \T \cup \A$}.
Note that all axioms except sharpening statements are preceded by modal operators (\emph{``modalised''} for short).
In case the preceding operator happens to be $\standball$, we may omit it.
%
%\begin{definition}[Subformula]
%    Let us define a subformula $\phi$ of a knowledge base $\kb$ in the following way,
%    \begin{itemize}
%        \item For all $\phi\in\kb$, $\phi$ is a subformula of $\kb$, and
%        \item if $\standbs\phi$ is a subformula, then $\phi$ is also a subformula.
%    \end{itemize}
%\end{definition}

\subsubsection{Semantics}
The semantics of standpoint $\EL$ is defined via standpoint structures.
Given a Standpoint DL vocabulary $\tuple{\Concepts,\Roles,\Individuals,\Stands}$, a \define{description logic standpoint structure} is a tuple
\mbox{$\dlstruct = \tuple{\Dom, \Precs, \sigma, \gamma}$} where:
\begin{itemize}
    \item $\Dom$ is a non-empty set, the \define{domain} of $\dlstruct$;
    \item $\Precs$ is a set, called the \define{precisifications} of $\dlstruct$;
    \item $\sigma$ is a function mapping each standpoint symbol to a non-empty subset of $\Precs$;\footnote{As shown in \Cref{sec:intractable-extensions}, allowing for ``empty standpoints'' immediately incurs  intractability, even for an otherwise empty vocabulary.}
    \item $\gamma$ is a function mapping each precisification from $\Precs$ to an ``ordinary'' DL interpretation \mbox{$\struct=\stuple{\Dom,\intf}$} over the domain $\Dom$, where the interpretation function $\intf$ maps\/:
          \begin{itemize}
              \item each concept name \mbox{$A\in\Concepts$} to a set \mbox{$\interprets{A}\subseteq\Dom$},
              \item each role name \mbox{$R{\,\in\,}\Roles$} to a binary relation \mbox{$\interprets{R}{\,\subseteq\,}\Dom{\times}\Dom$},
              \item each individual name \mbox{$a\in\Individuals$} to an element \mbox{$\interprets{a}\in\Dom$},
          \end{itemize}
          and we require \mbox{$a^{\gamma(\pi)} = a^{\gamma(\pi')}$} for all \mbox{$\pi,\pi' \in \Precs$} and \mbox{$a\in\Individuals$}.
\end{itemize}
Note that by this definition, individual names (also referred to as constants) are interpreted rigidly, i.e., each individual name $a$ is assigned the same \mbox{$a^{\gamma(\pi)} \in \Delta$} across all precisifications \mbox{$\pi \in \Precs$}.
We will refer to this uniform $a^{\gamma(\pi)}$ by $a^{\dlstruct}$.
\pagebreak

For each $\pr\in\Precs$, the interpretation mapping $\struct=\gamma(\pr)$ is extended to concept terms via structural induction as follows:
\begin{linenomath*}    
\begin{align*}
    \interprets{\top}            & \eqdef \Dom       \qquad\qquad\quad      \interprets{(\standds C)} \eqdef \textstyle\bigcup_{\pr'\in\sigma(\sts)}C^{\gamma(\pr')} \\[-0.7ex]
    \interprets{\bot}            & \eqdef \emptyset\ \qquad\qquad\quad      \interprets{(\standbs C)} \eqdef \textstyle\bigcap_{\pr'\in\sigma(\sts)}C^{\gamma(\pr')} \\[-0.7ex]
    \interprets{(C_1\dland C_2)} & \eqdef \interprets{C_1} \cap \interprets{C_2}                                                                                     \\[-0.2ex]
    \interprets{(\exists R.C)}   & \eqdef \set{ \delta\in\Dom \guard \tuple{\delta,\varepsilon}\in\interprets{R} \text{ for some } \varepsilon\in\interprets{C} }
\end{align*}
\end{linenomath*}    
We observe that modalised concepts $\standvs C$ are interpreted uniformly across all precisifications \mbox{$\pi \in \Precs$}, which allows us to denote their extensions with $(\standvs C)^{\dlstruct}$.

A DL standpoint structure $\dlstruct$ \define{satisfies a sharpening statement} \mbox{$\sts\preceq\stsp$}, written as \mbox{$\dlstruct\models\sts\preceq\stsp$}, iff \mbox{$\sigma(\sts)\subseteq\sigma(\stsp)$}.
For the other axiom types, satisfaction by $\dlstruct$ is defined as follows:
\begin{linenomath*}    
\begin{align*}
    \dlstruct & \models \standbs[C{\,\dlsub\,} D] & \iffdef & \interpretgp{C}\subseteq\interpretgp{D} \text{ for each } \pr\in\sigma(\sts)                                                              \\[-0.5ex]
    \dlstruct & \models \standds[C{\,\dlsub\,} D] & \iffdef & \interpretgp{C}\subseteq\interpretgp{D} \text{ for some } \pr\in\sigma(\sts)                                                              \\[-0.5ex]
    \dlstruct & \models\standbs[C(a)]             & \iffdef & a^{\dlstruct}\in\textstyle\bigcap_{\pr\in\sigma(\sts)}\interpretgp{C} \ \textcolor{lightgray}{\left( =  (\standbs C)^{\dlstruct} \right)} \\[-0.5ex]
    \dlstruct & \models\standds[C(a)]             & \iffdef & a^{\dlstruct}\in\textstyle\bigcup_{\pr\in\sigma(\sts)}\interpretgp{C} \ \textcolor{lightgray}{\left( =  (\standds C)^{\dlstruct} \right)} \\[-0.5ex]
    \dlstruct & \models\standbs[R(a,\!b)]         & \iffdef & \stuple{a^{\dlstruct}\!\!,b^{\dlstruct}} \in \textstyle\bigcap_{\pr\in\sigma(\sts)}\interpretgp{R}                                        \\[-0.5ex]
    \dlstruct & \models\standds[R(a,\!b)]         & \iffdef & \stuple{a^{\dlstruct}\!\!,b^{\dlstruct}} \in \textstyle\bigcup_{\pr\in\sigma(\sts)}\interpretgp{R}
\end{align*}
\end{linenomath*}    

As usual,
$\dlstruct$ is a \define{model of $\S$} iff it satisfies every sharpening statement in $\S$;
it is a \define{model of $\T$} iff it satisfies every GCI \mbox{$\tau\in\T$};
it is a \define{model of $\A$} iff it satisfies every assertion \mbox{$\alpha\in\A$};
it is a \define{model of \mbox{$\K=\tuple{\S, \T,\A}$}} (written \mbox{$\dlstruct\models\K$}) iff it is a model of $\S$ and a model of $\T$ and a model of $\A$.

%% file: sections/reasoning-problems.tex
Our investigations regarding reasoning in $\SEL$ will focus on standpoint versions of the well-known standard reasoning tasks, and we will make use of variations of established techniques to (directly or indirectly) reduce all of them to the first.
\begin{description}
     \item[Knowledge base satisfiability:] \textsl{Given a knowledge base $\K$, is there a DL standpoint structure $\dlstruct$ such that $\dlstruct\models\K$?}
     \item[Axiom entailment:]
          \textsl{Given $\K$ and some SBox, TBox, or ABox axiom $\phi$, does $\K\models\phi$ hold, that is, is it the case that for every model $\dlstruct$ of $\K$ we have $\dlstruct\models\phi$?}
          \\
          To show that axiom entailment can be polynomially reduced to knowledge base unsatisfiability, we exhibit for every axiom type $\phi$ a knowledge base $\kb_{\neg \phi}$ such that $\kb \models \phi$ coincides with unsatisfiability of $\kb \cup \kb_{\neg \phi}$:
    \begin{linenomath*}    
          \begin{align*}
               \kb_{\st \preceq \sp}           & \eqdef \{\standds[\top \sqsubseteq \tilde A],\  \standbsp[\tilde A \sqsubseteq \bot ]\}                                               \\
               \kb_{\standvs[C \sqsubseteq D]} & \eqdef \{ \tilde A \sqsubseteq C,\ \tilde A \sqcap D \sqsubseteq \bot, \standvs^{\!\!d}[\top \sqsubseteq \exists \tilde R.\tilde A]\} \\
               \kb_{\standvs[C(a)]}            & \eqdef \{ \tilde A \sqcap C \sqsubseteq \bot,\ \standvs^{\!\!d}[\tilde A(a)]  \}                                                      \\
               \kb_{\standvs[R(a,b)]}          & \eqdef \{ \tilde B(b),\ \tilde A \sqcap \exists R.\tilde B \sqsubseteq \bot, \standvs^{\!\!d}[\tilde A(a)] \}
          \end{align*}
    \end{linenomath*}    
          Thereby, \mbox{$\standbs^{\!\!d} \eqdef \standds$} and \mbox{$\standds^{\!\!d} \eqdef \standbs$}, and \mbox{$\tilde A, \tilde B$} denote fresh concept names and $\tilde R$ a fresh role name.
     \item[Concept satisfiability (w.r.t.\ $\kb$):]
          \textsl{Given $\K$ and a mo\-da\-lised concept term $C$, is there a model $\dlstruct$ of $\K$ with \mbox{$C^\dlstruct \neq \emptyset$}?}
          \\
          This task can be solved by checking the axiom entailment \mbox{$\kb \models \standball[ C \sqsubseteq \bot ]$}.
          If the entailment holds, then $C$ is unsatisfiable w.r.t.\ $\kb$, otherwise it is satisfiable.
     \item[Instance retrieval:]
          \textsl{Given $\K$ and a mo\-da\-lised concept term $C$, obtain all \mbox{$a\in \Individuals$} with \mbox{$a^\dlstruct \in C^\dlstruct$} for every model $\dlstruct$ of $\K$.}
          \\
          This task can be solved by checking, for all individuals $a$, if the entailment \mbox{$\kb \models \standball[ C(a) ]$} holds and returning all such $a$.
          %This task can be solved by checking returning all individuals for which $\kb \models \standball[ C(a) ]$. 

\end{description}
%
%\begin{shortonly}
%    Both can be polynomially reduced to each other's complements:
%    $\tuple{\S,\T,\A}\not\models\standbs[C\dlsub D]$
%    iff
%    $\tuple{\C\cup\set{\standds[C_a\dland D\dlsub \bot]},\A\cup\set{a:C, a:C_a}}$ is satisfiable,
%    with $a\in\Individuals$ and $C_a\in\Concepts$ fresh symbols not occurring in $\K$.
%    Likewise,
%    $\tuple{\S,\T,\A}\not\models\standds[C\dlsub D]$
%    iff
%    $\tuple{\C\cup\set{\top\dlsub \exists R.C', \standbs[C'\dlsub C], \standbs[C'\dland D\dlsub \bot]},\A}$,
%    with $C'\in\Concepts$ and $R\in\Roles$ fresh.
%\end{shortonly}
%
\begin{longonly}
     For boxed GCIs, the latter problem can be reduced to the former as follows:
     \begin{align*}
           & \mathrel{\phantom{\iff}} \tuple{\S,\T,\A}\not\models\standbs[C\dlsub D]                                                                                                                                    \\
           & \iff \text{for some } \dlstruct\models\tuple{\S,\T,\A} \text{ we have } \dlstruct\not\models\standbs[C\dlsub D]                                                                                            \\
           & \iff \text{for some } \dlstruct\models\tuple{\S,\T,\A} \text{ there is a } \pr\in\sigma(\sts) \text{ such that } \interpretgp{C}\not\subseteq\interpretgp{D}                                               \\
           & \iff \text{for some } \dlstruct\models\tuple{\S,\T,\A} \text{ there is a } \pr\in\sigma(\sts) \text{ and a } \delta\in\Dom \text{ with } \delta\in\interpretgp{C} \text{ but } \delta\notin\interpretgp{D} \\
           & \iff \tuple{\C\cup\set{\standds[C_a\dland D\dlsub \bot]},\A\cup\set{a:C, a:C_a}} \text{ is satisfiable}
     \end{align*}
     where $a\in\Individuals$ and $C_a\in\Concepts$ are fresh symbols not occurring in $\K$.

     Dually, for “diamonded” GCIs we can use the following:
     \begin{align*}
           & \mathrel{\phantom{\iff}} \tuple{\S,\T,\A}\not\models\standds[C\dlsub D]                                                                                                       \\
           & \iff \text{for some } \dlstruct\models\tuple{\S,\T,\A} \text{ we have } \dlstruct\not\models\standds[C\dlsub D]                                                               \\
           & \iff \text{for some } \dlstruct\models\tuple{\S,\T,\A} \text{ we find that } \pr\in\sigma(\sts) \text{ implies } \interpretgp{C}\not\subseteq\interpretgp{D}                  \\
           & \iff \text{for some } \dlstruct\models\tuple{\S,\T,\A} \text{ we find that for each } \pr\in\sigma(\sts) \text{ there is a } \delta\in\interpretgp{C}\setminus\interpretgp{D} \\
           & \iff \tuple{\C\cup\set{\top\dlsub \exists R.C', \standbs[C'\dlsub C], \standbs[C'\dland D\dlsub \bot]},\A} \text{ is satisfiable}
     \end{align*}
     where $C'\in\Concepts$ and $R\in\Roles$ are fresh symbols not occurring in $\K$.
\end{longonly}

%\color{blue} \textbf{S.R.: some notes on reducing entailment to satisfiability (I think we hadn't spelled this out yet).}\hs{We had, see above, although it was quite some time ago.}

%We define a function mapping axioms $\alpha$ to axiom sets $\kb_{\neg \alpha}$ such that $\kb \cup \kb_{\neg \alpha}$ is unsatisfiable iff $\kb \models \alpha$.
%Thereby $\tilde A_X$ denote fresh concept names and $\tilde R$ a fresh role name. $\standvs$ can be $\standbs$ or $\standds$, and we define $\standbs^{\!\!d} = \standds$ and
%$\standds^{\!\!d} = \standbs$.
%$$
%    \begin{array}{|l|l|}\hline
%        \alpha                    & \kb_{\neg \alpha}                                                                                               \\\hline & \\[-2ex]
%        \standvs[C \sqsubseteq D] & \tilde A \sqsubseteq C,\ \tilde A \sqcap D \sqsubseteq \bot,               \\[0.5ex]
%                                 & \standvs^{\!\!d}[\top \sqsubseteq \exists \tilde R.\tilde A]                       \\\hline & \\[-2ex]
%       \standvs[C(a)]            & \tilde A \sqcap C \sqsubseteq \bot,\ \standvs^{\!\!d}[\tilde A(a)]              \\\hline & \\[-2ex]
%       \standvs[R(a,b)]          & \tilde B(b),\ \tilde A \sqcap \exists R.\tilde B \sqsubseteq \bot, \\[0.5ex]
%                                 & \standvs^{\!\!d}[\tilde A(a)]                                                                        \\\hline
%   \end{array}
%$$

%\color{black}

%% file: sections/normal-form.tex
\subsubsection{Normalisation}
\begin{shortonly}

    Before we can describe a $\PTime$ algorithm for checking satisfiability of $\SEL$ knowledge bases, we need to introduce an appropriate normal form.
    %Given a knowledge base $\kb$, we use $\BCC$ (the \define{basic concepts} of $\kb$) to denote the set of all concept names used in $\mathcal{C}$
    %plus the top concept $\top$.
    %
    %Additionally, the set $\BMCC$ of \define{basic modal concepts} is the smallest set closed under the following rules:
    %\begin{itemize}
    %    \item if $C\in\BCC$, then $C\in\BMCC$;
    %    \item if $C\in\BCC$ and $\sts\in\Stands$, then $\standds C\in\BMCC$ and $\standbs C\in\BMCC$.
    %\end{itemize}

    \begin{definition}[Normal Form of $\SEL$ Knowledge Bases]
        \label{def:normal-form}
        A TBox $\T$ is in \define{normal form} iff, for all its GCIs
        \(
        \standbs [ C \sqsubseteq  D ],
        \)\linebreak
        $C$ is of the form $A$, $\exists R.A$ or \mbox{$A \dland A'$} with \mbox{$A,\!A'{\,\in\,}\Concepts {\,\cup\,} \{\top\}$} and
        %    \item 
        $D$ is of the form $B$,  $\exists R.B$, $\standdsp B$ or $\standbsp B$ with $B{\,\in\,}\Concepts{\,\cup\,}\{\bot\}$,
        %    \end{itemize}
        where \mbox{$R\in \Roles$}, and \mbox{$\st,\sp\in\Stands$}.
        \\
        An ABox $\A$ is in \define{normal form} iff all assertions have the form $\standbs[A(a)]$ or $\standbs[R(a,b)]$ for $a,\!b{\,\in\,}\Individuals$, $A{\,\in\,}\Concepts$, and $R{\,\in\,}\Roles$.
        \\
        $\K=\tuple{\S,\T,\A}$ is in \define{normal form} whenever $\T$ and $\A$ are.
    \end{definition}

    % \begin{longonly}
    %     For a given $\SEL$ knowledge base \mbox{$\K=\tuple{\S,\T,\A}$}, normalising its ABox $\A$ is easily achieved by
    %     \begin{enumerate}
    %         \item replacing each assertion
    %               $\standds\alpha$ (with $\alpha$ of the shape $C(a)$ or $R(a,b)$) by $\standbsp \alpha$ and \mbox{$\stsp\preceq\sts$} with $\stsp$ fresh,\footnote{Note that, for this normalization step to be correct, it is crucial that standpoints are interpreted non-empty.} and then
    %         \item replacing each assertion
    %               $\standbs[C(a)]$ with $C {\,\not\in\,}\Concepts {\,\cup\,} \{\top\}$ non-basic by $\standbs[A(a)]$ and $\standbs[A\dlsub C]$ for a fresh concept name $A\in\Concepts$ not occurring in $\K$.
    %     \end{enumerate}
    %     We will next concentrate on normalising TBoxes.
    % \end{longonly}
    \renewcommand{\goesto}{\rightarrow}
    For a given $\SEL$ knowledge base \mbox{$\K=\tuple{\S,\T,\A}$}, we can compute its normal form by exhaustively applying the following transformation rules (where ``rule application'' means that the axiom on the left-hand side is replaced with the set of axioms on the right-hand side):
    %For readability, we mix sharpening statements and GCIs on right-hand sides with the understanding that each go into their respective set.

    \vspace{-2ex}\label{norm:rules}
    \begin{linenomath*}
        {\small
            \begin{align}
                \standds [ C(a) ]                      & \goesto \set{ \stv\preceq\sts, \standbv [ C(a) ] } \label{norm:diamond-concept-assertion}                            \\
                \standds [ R(a,b) ]                    & \goesto \set{ \stv\preceq\sts, \standbv [ R(a,b) ] } \label{norm:diamond-role-assertion}                             \\
                \standds [ C \dlsub D ]                & \goesto \set{ \stv\preceq\sts, \standbv [ C\dlsub D ] } \label{norm:diamond-axiom}                                   \\
                \standbs [ \bar C(a) ]                 & \goesto \set{ \standbs [ A(a) ], \standbs [ A \dlsub \bar C ]} \label{norm:complex-concept-assertion}                \\
                \standbs [ B \dlsub \exists R.\bar C ] & \goesto \set{ \standbs [ B \dlsub \exists R. A ], \standbs [ A \dlsub \bar C ]} \label{norm:existential-right}       \\
                \standbs [ B \dlsub C \dland D ]       & \goesto \set{ \standbs [ B\dlsub A], \standbs [ A\dlsub C ], \standbs [ A \dlsub D ]} \label{norm:conjunction-right} \\
                \standbs [ C \dlsub \standvu \bar D ]  & \goesto \set{ \standbs [ C \dlsub \standvu A ], \standbs [ A \dlsub \bar D ] } \label{norm:modal-right}              \\
                \standbs [ C\dlsub \top ]              & \goesto \emptyset \quad\text{and}\quad \standbs [ \bot \dlsub D ] \goesto \emptyset \label{norm:trivial}             \\
                \standbs [ \exists R.\bar C \dlsub D ] & \goesto \set{ \standbs [\bar C \dlsub A ], \standbs [ \exists R. A \dlsub D ] } \label{norm:existential-left}        \\
                \standbs [ \bar C \dland D \dlsub E ]  & \goesto \set{ \standbs [\bar C \dlsub A ], \standbs [ A \dland D \dlsub E ] } \label{norm:conjunction-left}          \\
                \standbs [ \standdu C \dlsub D ]       & \goesto \set{ \standbu [ C \dlsub \standball A ], \standbs [ A \dlsub D ] } \label{norm:diamond-left}                \\
                \standbs [ \standbu C \dlsub D ]       & \goesto \{ \stvo\preceq\stu, \stvi\preceq\stu, \standbu[C\dlsub A], \nonumber                                        \\[-0.3ex]
                                                       & \qquad\qquad\quad \standbs [ \standdvo A \dland \standdvi A \dlsub D ] \} \label{norm:box-left}
            \end{align}}
    \end{linenomath*}

    \noindent Therein, $\bar C$ and $\bar D$ stand for complex concept terms not contained in $\Concepts {\,\cup\,} \{\top\}$, and each occurrence of $A$ on a right-hand side denotes the introduction of a fresh concept name;
    likewise, $\stv$, $\stvo$, and $\stvi$ denote of a fresh standpoint name.
    Rule (\ref{norm:conjunction-left}) is applied modulo commutativity of $\sqcap$.
    Most of the transformation rules should be intuitive (keep in mind that standpoints must be nonempty).
    A notable exception is Rule (\ref{norm:box-left}), which is crucial to remove boxes occurring with negative polarity.
    It draws some high-level inspiration from existing work on non-vacuous left-hand-side universal quantifiers in Horn DLs \cite{CarralKRH14}, yet the argument for its correctness requires a more intricate model-theoretic construction and hinges on ``Hornness'' of $\mathcal{K}$ and nonemptiness of standpoints.
    A careful analysis yields that the transformation has the desired semantic and computational properties.
\end{shortonly}

\begin{lemma}
    \label{lem:normalisation}
    Every $\SEL$ knowledge base $\K$ can be transformed into a $\SEL$ knowledge base $\K'$ in normal form such that:
    \begin{itemize}
        \item $\K'$ is a $\SEL$-conservative extension of $\K$,
        \item the size of $\K'$ is at most linear in the size of $\K$, and
        \item the transformation can be computed in \PTime.
    \end{itemize}
\end{lemma}

\begin{shortonly}
    While $\K'$ being a $\SEL$-conservative extension of $\K$ brings about various valuable properties,
    what matters for our purposes is that this implies equisatisfiability of $\K$ and $\K'$, thus we will not go into details about conservative extensions.
\end{shortonly}

\begin{longonly}
    We will next show \Cref{lem:normalisation}, where we will concentrate on normalising TBoxes.
    By introducing new standpoint, concept, and role names, any TBox $\T$ can be turned into a normalised TBox $\T'$ that is a conservative extension of $\T$, i.e., every model of $\T'$ is also a model of $\T$, and every model of $\T$ can be extended to a model of $\T'$ by appropriately choosing the interpretations of the additional standpoint, concept, and role names.

    To show that this transformation can be done in polynomial time, yielding a normalised TBox $\T'$ whose size is linear in the size of $\T$, we next define the size $\size{\K}$ of a knowledge base $\K$ roughly as the number of symbols needed to write down $\K$, and define it formally as follows.
    \begin{definition}
        \label{def:size}
        Let $\K=\tuple{\S,\T,\A}$ be a Standpoint DL knowledge base.
        The \define{size} of $\K$, denoted $\size{\K}$, and its various constituents is defined inductively as follows\/:
        \begin{align*}
            \size{\sts\preceq\stu}     & \eqdef 2                                                            \\
            \size{\S}                  & \eqdef \sum_{\zeta\in\S}\size{\zeta}                                \\
            \size\top                  & \eqdef 1                                                            \\
            \size\bot                  & \eqdef 1                                                            \\
            \size A                    & \eqdef 1                               & \text{ for } A\in\Concepts \\
            \size{C\dland D}           & \eqdef 1 + \size{C} + \size{D}                                      \\
            \size{\exists R. C}        & \eqdef 1 + \size C                                                  \\
            \size{\standds C}          & \eqdef 1 + \size C                                                  \\
            \size{\standbs C}          & \eqdef 1 + \size C                                                  \\
            \size{C\dlsub D}           & \eqdef 1 + \size{C} + \size{D}                                      \\
            \size{\standbs[C\dlsub D]} & \eqdef 1 + \size{C\dlsub D}                                         \\
            \size{\standds[C\dlsub D]} & \eqdef 1 + \size{C\dlsub D}                                         \\
            \size{\T}                  & \eqdef \sum_{\tau\in\T}\size\tau                                    \\
            \size{C(a)}                & \eqdef 1+\size{C}                                                   \\
            \size{R(a_1,a_2)}          & \eqdef 3                                                            \\
            \size{\A}                  & \eqdef \sum_{\alpha\in\A}\size{\alpha}                              \\
            \size{\K}                  & \eqdef \size{\S}+\size{\T}+\size{\A}
        \end{align*}
    \end{definition}
\end{longonly}

\begin{longproof}[Proof of \Cref{lem:normalisation}]
    Let \mbox{$\K=\tuple{\S,\T,\A}$} be a Standpoint \EL knowledge base where w.l.o.g.\ $\A$ is in normal form.
    The rest of $\K$ (that is, $\T$, since there is only one form of sharpening statements) can be converted into normal form by exhaustively applying the replacement rules shown on page~\pageref{norm:rules}.
    In what follows, let $\T$ be an arbitrary TBox and let $\T'$ result from exhaustive rule application to $\T$.
    To prove the lemma, we proceed to show the following:
    \begin{enumerate}
        \item Polynomial runtime and linear output size: Application of the normalisation rules terminates after at most a polynomial (in the size of $\T$) number of steps, and the size of the resulting TBox $\T'$ is at most linear in the size of $\T$.
        \item Syntactic Correctness: $\T'$ is in normal form according to \Cref{def:normal-form}.
        \item Semantic Correctness:
              $\K'=\tuple{\S',\T',\A}$ is a conservative extension of $\K$, more specifically:
              \begin{description}
                  \item[(a)] For every model $\dlstruct$ of $\K$ there exists a DL standpoint structure $\dlstruct'$ that extends $\dlstruct$ (agrees with $\dlstruct$ on the vocabulary of $\K$) and that is a model of $\K'$.
                  \item[(b)] Every model $\dlstruct'$ of $\K'$ is also a model of $\K$.
              \end{description}
              In most of the cases, we can however restrict our attention to $\T'$.
    \end{enumerate}
    \begin{enumerate}
        \item We first show that normalisation must terminate after at most $\size{\T}$ normalisation rule applications.
              The proof plan is as follows:
              We define the multiset \mbox{$\CT:\SFCT\to\N$} that contains one copy for each occurence of a concept term occurring in $\T$.
              We observe that for each complex concept $\bar C$ causing some GCI not to be in normal form, a constant number of normalisation rule applications can be used to strictly decrease the cardinality of $\bar C$ in $\CT$.
              Together with the fact that \mbox{$\sum_{D\in\SFCT}\CT(D)\leq\size{\T}$}, the claim then follows.

              First of all, rules (\ref{norm:diamond-concept-assertion})--(\ref{norm:complex-concept-assertion}) and in particular rule~(\ref{norm:diamond-axiom}) are only applied once per axiom type, so in what follows we may safely assume that $\T$ contains only GCIs of the form $\standbs[C\dlsub D]$.
              Define the set $\SFCT$ as the least set such that:
              \begin{itemize}
                  \item If $\standbs[C\dlsub D]\in\T$, then $C,D\in\SFCT$;
                  \item if $C\in\SFCT$, then for all subconcepts $C'$ of $C$ (written $C'\in\sub(C)$) we have $C'\in\SFCT$.
              \end{itemize}
              Now for any TBox $\T'$ (typically obtained from $\T$ by applying zero or more normalisation rules), the multiset \mbox{$\CTp\colon\SFCT\to\N$} is then as follows:
              \begin{align*}
                  C       & \mapsto \sum_{\standbs[D\dlsub E]\in\T'}\left( c(C,D) + c(C,E) \right)
                  \intertext{where we define the concept-counting function $c\colon\SFCT\times\SFCT\to\N$ inductively: }
                  c(C, D) & \eqdef
                  \begin{cases}
                      1                 & \text{ if } C=D,                                                          \\
                      c(C,E)            & \text{ if } C\neq D \text{ and } [D=\exists R.E \text{ or } D=\standvs E] \\
                      c(C,E_1)+c(C,E_2) & \text{ if } C\neq D \text{ and } D=E_1\dland E_2                          \\
                      0                 & \text{ if } C\neq D \text{ and } D\in\Concepts\cup\set{\top,\bot}
                  \end{cases}
              \end{align*}
              For example if \mbox{$D=\top \dland \standds(\top\dland \exists R.(\top \dland \top))$}, then \mbox{$c(\top, D)=4$} while \mbox{$c(\top\dland \top,D)=1$}.

              We next relate the overall cardinality (sum of number of occurrences) of $\CT$ to the size $\size{\T}$ on the basis that both can be represented as disjoint unions and we can therefore sum up the individual cardinalities.
              \begin{numberclaim}
                  \label{claim:norm:size-multi}
                  We have $\sum_{C\in\SFCT}\CT(C)\leq\size{\T}$.
                  \begin{claimproof}
                      As $\size{\T}=\sum_{\tau\in\T}\size{\tau}$ and
                      \[ \sum_{C\in\SFCT}\CT(C)=\sum_{C\in\SFCT}\sum_{\standbs[D\dlsub E]\in\T}\left( c(C,D) + c(C,E) \right) = \sum_{\standbs[D\dlsub E]\in\T}\sum_{C\in\SFCT}\left( c(C,D) + c(C,E) \right)
                      \]
                      it suffices to look at a single $\standbs[D\dlsub E]\in\T$ and show
                      \[
                          \sum_{C\in\SFCT}\left( c(C,D) + c(C,E) \right) \leq \size{\standbs[D\dlsub E]}
                      \]
                      which by
                      \[
                          \sum_{C\in\SFCT}\left( c(C,D) + c(C,E) \right) = \sum_{D'\in\sub(D)}c(D',D) + \sum_{E'\in\sub(E)}c(E',E)
                      \]
                      develops into
                      \[
                          \sum_{D'\in\sub(D)}c(D',D) + \sum_{E'\in\sub(E)}c(E',E) \leq \size{\standbs[D\dlsub E]} = 1 + 1 + \size{D} + \size{E}
                      \]
                      for which it suffices to show that for any concept \mbox{$C\in\SFCT$}, we have \mbox{$\sum_{C'\in\sub(C)}c(C',C)\leq\size{C}$}, which we show by induction.
                      \begin{itemize}
                          \item The base case is clear, as $\sum_{C'\in\sub(C)}c(C',C)=c(C,C)=1\leq 1=\size{C}$ for any $C\in\Concepts\cup\set{\top,\bot}$.
                          \item $C=\exists R.D$:
                                \begin{align*}
                                     & \pheq \sum_{C'\in\sub(\exists R.D)}c(C',\exists R.D)                 \\
                                     & = c(\exists R.D,\exists R.D) + \sum_{C'\in\sub(D)}c(C',D)            \\
                                     & \stackrel{\text{(IH)}}{\leq} = c(\exists R.D,\exists R.D) + \size{D} \\
                                     & = 1 + \size{D}                                                       \\
                                     & = \size{\exists R.D}
                                \end{align*}
                          \item $C=\standvs D$: Similar.
                          \item $C=C_1\dland C_2$:
                                \begin{align*}
                                     & \pheq \sum_{C'\in\sub(C_1\dland C_2)}c(C',C_1\dland C_2)                                           \\
                                     & = c(C_1\dland C_2,C_1\dland C_2) + \sum_{C'\in\sub(C_1)}c(C',C_1) + \sum_{C'\in\sub(C_2)}c(C',C_2) \\
                                     & \stackrel{\text{(IH)}}{\leq} c(C_1\dland C_2,C_1\dland C_2) + \size{C_1} + \size{C_2}              \\
                                     & = 1 + \size{C_1} + \size{C_2}                                                                      \\
                                     & = \size{C_1\dland C_2}
                                \end{align*}
                      \end{itemize}
                      This concludes the proof of \Cref{claim:norm:size-multi}.
                  \end{claimproof}
              \end{numberclaim}
              To prove an overall linear number of rule applications, we next show that for each complex concept $\bar C$ whose occurence in a GCI $\tau\in\T'$ causes $\tau$ not to be in normal form, there is a constant number of rule applications (that is, constant for all TBoxes) such that after rule application, the number of overall occurrences of $\bar C$ has strictly decreased and additionally, any intermediate complex concepts introduced by the rule have been normalised in turn.
              To this end, we need to define some more notions.
              Let $\T$ be a TBox and $\T'$ be obtained by application of an arbitrary number of normalisation rules.
              We say that a complex concept $\bar C$ occurring in a GCI \mbox{$\tau\in\T'$} is a \define{culprit for $\tau$} iff $\tau$ is not in normal form because of $\bar C$;
              a culprit $\bar C$ is \define{top-level} for \mbox{$\tau=\standbs[D\dlsub E]$} iff \mbox{$\bar C=D$} or \mbox{$\bar C=E$}.
              We say that $\T'$ is \define{faithful to $\T$} iff every culprit occurring in $\T'$ already occurs in $\T$.
              In the proof below, we show that although normalisation rules sometimes introduce new culprits, those culprits will not lead to problems because each one can only cause a constant overhead.
              \begin{numberclaim}
                  \label{claim:norm:step-decrease}
                  Let $\T$ be a TBox and let $\T'$ be obtained from $\T$ by applying any number of rules from (\ref{norm:existential-right})--(\ref{norm:box-left}).
                  Let $\bar C$ be a top-level culprit for \mbox{$\tau\in\T'$}.
                  Then there is a constant number of rule applications leading to a TBox $\T''$ that is faithful to $\T$ and where \mbox{$\CTp(\bar C)>\CTpp(\bar C)$}.
                  \begin{claimproof}
                      We do a case distinction on the occurrence (left-hand vs.\ right-hand side) and form of $\bar C$.
                      In every case, we will explicitly give the (constantly many) rules to apply to decrease the cardinality of $\bar C$.
                      In most cases, it is easy to see that only one rule is needed, we show this exemplarily for one case and then concentrate on the two non-trivial cases.
                      \begin{itemize}
                          \item $\bar C=\exists R.\bar D$ occurs on the right-hand side:
                                Then we apply rule~(\ref{norm:existential-right}), removing one occurrence of $\bar C$.
                                The resulting $\T''$ is faithful because the only newly introduced concept terms are $\exists R.A$ on a right-hand side and $A$ on a left-hand side, and neither of these is a culprit.
                          \item $\bar C=D_1\dland D_2$ or $\bar C=\standvs D$ and $\bar C$ occurs on the right-hand side: Similar.
                          \item $\bar C=\exists R.\bar D$ or $\bar C=\standds D$ and $\bar C$ occurs on the left-hand side: Exercise.
                          \item $\bar C=D_1\dland D_2$ occurs on the left-hand side:
                                Denote $\tau=\standbs[\bar C\dlsub E]=\standbs[D_1\dland D_2\dlsub E]$.
                                We apply rule~(\ref{norm:conjunction-left}) and obtain $\standbs[D_1\dlsub A]$ and $\standbs[A\dland D_2\dlsub E]$, thus removing once occurence of $\bar C$.
                                The latter rule contains the (new) culprit $A\dland D_2$ to which we can apply the same rule again (modulo commutativity) to obtain $\standbs[D_2\dlsub A']$ and $\standbs[A\dland A'\dlsub E]$.
                                The only remaining (potential) culprits are $D_1$, $D_2$, and $E$, whence the resulting $\T''$ is faithful to $\T$.
                          \item $\bar C=\standbu D$ occurs on the left-hand side:
                                Denote $\tau=\standbs[\standbu D\dlsub E]$.
                                We apply rules (to underlined GCIs) as follows\/:
                                \begin{align*}
                                     & \mathrel{\phantom{\leadsto}} \set{ \underline{\standbs[\standbu D\dlsub E ]} }                                                                                                                                                                           \\
                                     & \stackrel{\text{(\ref{norm:box-left})}}{\leadsto} \set{ \standbu[D\dlsub A_1], \underline{\standbs[\standdvo A_1\dland\standdvi A_1\dlsub E]} }                                                                                                          \\
                                     & \stackrel{\text{(\ref{norm:conjunction-left})}}{\leadsto} \set{ \standbu[D\dlsub A_1], \standbs[\standdvo A_1\dlsub A_2], \underline{\standbs[A_2\dland\standdvi A_1\dlsub E]} }                                                                         \\
                                     & \stackrel{\text{(\ref{norm:conjunction-left})}}{\leadsto} \set{ \standbu[D\dlsub A_1], \underline{\standbs[\standdvo A_1\dlsub A_2]}, \standbs[\standdvi A_1\dlsub A_3], \standbs[A_2\dland A_3\dlsub E] }                                               \\
                                     & \stackrel{\text{(\ref{norm:diamond-left})}}{\leadsto} \set{ \standbu[D\dlsub A_1], \standb{\stvo}[A_1\dlsub\standball A_4], \standbs[A_4\dlsub A_2], \underline{\standbs[\standdvi A_1\dlsub A_3]}, \standbs[A_2\dland A_3\dlsub E] }                    \\
                                     & \stackrel{\text{(\ref{norm:diamond-left})}}{\leadsto} \set{ \standbu[D\dlsub A_1], \standb{\stvo}[A_1\dlsub\standball A_4], \standbs[A_4\dlsub A_2], \standb{\stvi}[A_1\dlsub\standball A_5], \standbs[A_5\dlsub A_3], \standbs[A_2\dland A_3\dlsub E] }
                                \end{align*}
                                It is easy to see that $D$ and $E$ are the only remaining (potential) culprits, and that one occurence of $\bar C=\standbu D$ has been removed.
                                % We apply rule~(\ref{norm:box-left}) and obtain $\standbu[D\dlsub A]$ and $\standbs[\standdvo A\dland\standdvi A\dlsub E]$.
                                %In the second GCI, we have a culprit of the form $D_1\dland D_2$ on the left-hand side and can argue as in the previous case to obtain $\standbs[\standdvo A\dlsub A']$, $\standbs[\standdvi A\dlsub A'']$, and $\standbs[A'\dland A''\dlsub E]$.
                                %To $\standbs[\standdvo A\dlsub A']$, we apply rule~(\ref{norm:diamond-left}) and obtain $\standb{\stvo}[A\dlsub\standball A''']$ and $\standbs[A'''\dlsub A']$ which are both in normal form, and do likewise for $\standbs[\standdvi A\dlsub A'']$.
                      \end{itemize}
                  \end{claimproof}
              \end{numberclaim}
              Thus each culprit will be not only be removed eventually, but removing each single occurrence will take only a constant number of steps.
              Therefore, the number of rule applications is linear in $\sum_{C\in\SFCT}\CT(C)$.
              By \Cref{claim:norm:size-multi}, the number of rule applications is linear in $\size{\T}$, thus linear in $\size{\K}$.

              It remains to show that the overall increase in size is at most linear.
              We do this by showing that for each single rule, the size increase caused by its application is constant.
              Together with the overall linear number of rule applications, it follows that the size of the resulting normalised TBox $\T'$ is at most linear in the size of the original TBox $\T$.
              \begin{description}
                  % \standds [ C \dlsub D ]                & \goesto \set{ \stv\preceq\sts, \standbv [ C\dlsub D ] } \label{norm:diamond-axiom}                             \\
                  \item[Rule~(\ref{norm:diamond-axiom}):]
                      \begin{align*}
                           & \pheq \size{\stv\preceq\sts} + \size{\standbv [ C\dlsub D ]} - \size{\standds [ C\dlsub D]} \\
                           & = 2 + 1 + 1 + \size{C} + \size{D} - \left( 1 + 1 + \size{C} + \size{D} \right)              \\
                           & = 2
                      \end{align*}
                      % \standbs [ B \dlsub \exists R.\bar C ] & \goesto \set{ \standbs [ B \dlsub \exists R. A ], \standbs [ A \dlsub \bar C ]} \label{norm:existential-right} \\
                  \item[Rule~(\ref{norm:existential-right}):]
                      \begin{align*}
                           & \pheq \size{\standbs [ B \dlsub \exists R. A ]} + \size{\standbs [ A \dlsub \bar C ]} - \size{\standbs [ B \dlsub \exists R.\bar C ]} \\
                           & = 1 + 1 + \size{B} + 2 + 1 + 1 + 1 + \size{\bar C} - \left( 1 + 1 + \size{B} + 1 + \size{\bar C} \right)                              \\
                           & = 4                                                                                                                                   \\
                      \end{align*}
                      % \standbs [ B \dlsub C \dland D ]       & \goesto \set{ \standbs [ B\dlsub A], \standbs [ A\dlsub C ], \standbs [ A \dlsub D ]} \label{norm:conjunction-right}                  \\
                  \item[Rule~(\ref{norm:conjunction-right}):]
                      \begin{align*}
                           & \pheq \size{\standbs [ B\dlsub A]} + \size{\standbs [ A\dlsub C ]} + \size{\standbs [ A\dlsub D ]} - \size{\standbs [ B \dlsub C \dland D ]} \\
                           & = (1 + 1 + \size{B} + 1) + (1 + 1 + 1 + \size{C}) + (1 + 1 + 1 + \size{D}) - \left(1 + 1 + \size{B} + 1 + \size{C} + \size{D} \right)        \\
                           & = 6
                      \end{align*}
                      % \standbs [ C \dlsub \standvu \bar D ]  & \goesto \set{ \standbs [ C \dlsub \standvu A ], \standbs [ A \dlsub \bar D ] } \label{norm:modal-right}
                  \item[Rule~(\ref{norm:modal-right}):]
                      \begin{align*}
                           & \pheq \size{\standbs [ C \dlsub \standvu A ]} + \size{\standbs [ A \dlsub \bar D ]} - \size{\standbs [ C \dlsub \standvu \bar D ]} \\
                           & = 1 + 1 + \size{C} + 2 + 1 + 1 + 1 + \size{\bar D} - \left( 1 + 1 + \size{C} + 1 + \size{\bar D} \right)                           \\
                           & = 4
                      \end{align*}
                      % \standbs [ \bot \dlsub D ]             & \goesto \emptyset \label{norm:bottom-left}                                                                     \\
                  \item[Rule~(\ref{norm:trivial}):] Clear.
                      % \standbs [ \exists R.\bar C \dlsub D ] & \goesto \set{ \standbs [\bar C \dlsub A ], \standbs [ \exists R. A \dlsub D ] } \label{norm:existential-left}  \\
                  \item[Rule~(\ref{norm:existential-left}):]
                      \begin{align*}
                           & \pheq \size{\standbs [\bar C \dlsub A ]} + \size{\standbs [ \exists R. A \dlsub D ]} - \size{\exists R.\bar C \dlsub D} \\
                           & = 1 + 1 + \size{\bar C} + 1 + 1 + 1 + 2 + \size{D} - \left( 1 + 1 + \size{\bar C} + \size{D} \right)                    \\
                           & = 5
                      \end{align*}
                      % \standbs [ \bar C \dland D \dlsub E ]  & \goesto \set{ \standbs [\bar C \dlsub A ], \standbs [ A \dland D \dlsub E ] } \label{norm:conjunction-left}    \\
                  \item[Rule~(\ref{norm:conjunction-left}):]
                      \begin{align*}
                           & \pheq \size{\standbs [\bar C \dlsub A ]} + \size{\standbs [ A \dland D \dlsub E ]} - \size{\standbs [ \bar C \dland D \dlsub E ]}  \\
                           & = 1 + 1 + \size{\bar C} + 1 + 1 + 1 + 1 + 1 + \size{D} + \size{E} - \left( 1 + 1 + 1 + \size{\bar C} + \size{D} + \size{E} \right) \\
                           & = 4
                      \end{align*}
                      % \standbs [ \standdu C \dlsub D ]       & \goesto \set{ \standbu [ C \dlsub \standball A ], \standbs [ A \dlsub D ] } \label{norm:diamond-left}          \\
                  \item[Rule~(\ref{norm:diamond-left}):]
                      \begin{align*}
                           & \pheq \size{\standbu [ C \dlsub \standball A ]} + \size{\standbs [ A \dlsub D ]} - \size{\standbs [ \standdu C \dlsub D ]} \\
                           & = 1 + 1 + \size{C} + 2 + 1 + 1 + 1 + \size{D} - ( 1 + 1 + 1 + \size{C} + \size{D} )                                        \\
                           & = 4
                      \end{align*}
                      % \standbs [ \standbu C \dlsub D ]       & \goesto \set{ \stvo\preceq\stu, \stvi\preceq\stu, \standbu[C\dlsub A], \standbs [ \standdvo A \dland \standdvi A \dlsub D ] } \label{norm:box-left} \\
                  \item[Rule~(\ref{norm:box-left}):]
                      \begin{align*}
                           & \pheq \size{\stvo\preceq\stu} + \size{\stvi\preceq\stu} + \size{\standbu[C\dlsub A]} + \size{\standbs [ \standdvo A \dland \standdvi A \dlsub D ]} - \size{\standbs [ \standbu C \dlsub D ]} \\
                           & = 2 + 2 + 1 + 1 + \size{C} + 1 + 1 + 1 + 2 + 2 + \size{D} - ( 1 + 1 + 1 + \size{C} + \size{D})                                                                                               \\
                           & = 10
                      \end{align*}
              \end{description}
        \item Assume that $\T'$ is the result of exhaustively applying the normalisation rules (\ref{norm:diamond-axiom})--(\ref{norm:modal-right}) to $\T$.
              The proof is by contradiction.
              Assume that a GCI $\standvs[C\dlsub D]$ is not in normal form;
              we do a case distinction on the possible reasons for this where in each case it will turn out that at least one of the rules is applicable in contradiction to the presumption.
              % where $C$ is of the form $A$, $\exists R.A$ or $A_1 \dland A_2$ with $A,A_1,A_2\in\BCC$ and $D$ is of the form $B$,  $\exists R.B$, $\standdsp B$ or $\standbsp B$ with $B\in\BCC\cup\set{\bot}$, and $\st,\sp\in\Stands$.
              \begin{itemize}
                  \item $\standvs\neq\standbs$: Then \mbox{$\standvs=\standds$} and Rule~(\ref{norm:diamond-axiom}) is applicable.
                  \item $D$ is of the form $\exists R.E$ with \mbox{$E\notin\BCC\cup\set{\bot}$}: Then Rule~(\ref{norm:existential-right}) is applicable.
                  \item $D$ is of the form $E_1\dland E_2$: Then Rule~(\ref{norm:conjunction-right}) is applicable.
                  \item $D$ is of the form $\standvsp E$ with \mbox{$E\notin\BCC\cup\set{\bot}$}: Then Rule~(\ref{norm:modal-right}) is applicable.
                  \item $D=\top$ or $C=\bot$: Then Rule~(\ref{norm:trivial}) is applicable.
                  \item $C$ is of the form $\exists R.B$ with \mbox{$B\notin\BCC$}: Then Rule~(\ref{norm:existential-left}) is applicable.
                  \item $C$ is of the form $B_1\dland B_2$ with \mbox{$\set{B_1,B_2}\not\subseteq\BCC$}: Then Rule~(\ref{norm:conjunction-left}) is applicable.
                  \item $C$ is of the form $\standbsp B$ with \mbox{$B\notin\BCC$}: Then Rule~(\ref{norm:box-left}) is applicable.
                  \item $C$ is of the form $\standdsp B$ with \mbox{$B\notin\BCC$}: Then Rule~(\ref{norm:diamond-left}) is applicable.
              \end{itemize}
        \item We show correctness of each rule.
              Correctness of the overall normalisation process follows by induction on the number of rule applications.
              We slightly adapt the notation to denote by \mbox{$\K'=\tuple{\S,\T',\A}$} the KB that results from application \emph{of a single rule} and do a case distinction on the rules.
              In each case, assume \mbox{$\dlstruct = \tuple{\Dom,\Precs,\gamma,\sigma}$} with \mbox{$\dlstruct\models\K$} and denote \mbox{$\dlstruct'=\tuple{\Dom',\Precs',\gamma',\sigma'}$} in case the components differ from those of $\dlstruct$.
              (In most cases, we only need to show \mbox{$\dlstruct'\models\T'$} and therefore do not mention the other KB components.)
              \begin{description}
                  \item[Rule~(\ref{norm:diamond-axiom}):]
                      (a)~
                      Let \mbox{$\dlstruct\models\standds[C\dlsub D]$}.
                      Then there is a \mbox{$\pr\in\sigma(\sts)$} such that \mbox{$\interpretgp{C}\subseteq\interpretgp{D}$}.
                      Define $\dlstruct'$ from $\dlstruct$ by \mbox{$\sigma'(\stv)\eqdef\set{\pr}$}.
                      It follows that \mbox{$\dlstruct'\models\stv\preceq\sts$} and \mbox{$\dlstruct'\models\standbv[C\dlsub D]$}.

                      (b)~
                      Let \mbox{$\dlstruct'\models\T'$} and consider any \mbox{$\pr\in\sigma'(\stv)$} (which exists due to standpoint-non-emptiness).
                      Clearly \mbox{$\dlstruct'\models\S'$} shows \mbox{$\pr\in\sigma'(\stv)\subseteq\sigma'(\sts)$}, whence \mbox{$\interpretgp{C}\subseteq\interpretgp{D}$} witnesses that \mbox{$\dlstruct'\models\standds[C\dlsub D]$}.
                  \item[Rule~(\ref{norm:existential-right}):]
                      (a)~
                      Let \mbox{$\dlstruct\models\standbs[B\dlsub\exists R.\bar C]$}.
                      Define $\dlstruct'$ from $\dlstruct$ as follows:
                      For any \mbox{$\pr\in\sigma(\sts)$}, set \mbox{$\interpretpgp{A}\eqdef\interpretgp{\bar C}$}.
                      It follows by definition that \mbox{$\dlstruct'\models\standbs[B\dlsub\exists R.A]$} and \mbox{$\dlstruct'\models\standbs[A\dlsub\bar C]$}.

                      (b)~
                      Let \mbox{$\dlstruct'\models\T'$} and consider any \mbox{$\pr\in\sigma'(\sts)$} and $\de\in\interpretpgp{B}$.
                      By \mbox{$\dlstruct'\models\standbs[B\dlsub\exists R.A]$} we get \mbox{$\de\in\interpretpgp{(\exists R.A)}$}, whence there is an \mbox{$\ve\in\Dom'$} with \mbox{$\tuple{\de,\ve}\in\interpretpgp{r}$} and \mbox{$\ve\in\interpretpgp{A}$}.
                      By \mbox{$\dlstruct'\models\standbs[A\dlsub\bar C]$}, we get \mbox{$\ve\in\interpretpgp{\bar C}$} and ultimately \mbox{$\de\in\interpretpgp{(\exists R.\bar C)}$}.
                  \item[Rule~(\ref{norm:conjunction-right}):]
                      (a)~
                      Let \mbox{$\dlstruct\models\standbs[B\dlsub C\dland D]$} and define $\dlstruct'$ from $\dlstruct$ by setting, for each \mbox{$\pr\in\sigma(\sts)$}, \mbox{$\interpretpgp{A}\eqdef\interpretgp{B}$}.
                      Then by definition we get \mbox{$\dlstruct'\models\standbs[B\dlsub A]$}, as well as \mbox{$\dlstruct'\models\standbs[A\dlsub C]$} and \mbox{$\dlstruct'\models\standbs[A\dlsub D]$}.

                      (b)~
                      Let $\dlstruct'\models\T'$ and consider any $\pr\in\sigma'(\sts)$ and $\de\in\interpretpgp{B}$.
                      By $\T'\models\standbs[B\dlsub A]$ we get $\de\in\interpretpgp{A}$.
                      By $\T'\models\standbs[A\dlsub C]$ and $\T'\models\standbs[A\dlsub D]$ we get $\de\in\interpretpgp{C}$ and $\de\in\interpretpgp{D}$, respectively.
                      In combination, $\de\in\interpretpgp{(C\dland D)}$.
                  \item[Rule~(\ref{norm:modal-right}):]
                      (a)~
                      Let \mbox{$\dlstruct\models\standbs[C\dlsub \standvu\bar D]$} and define $\dlstruct'$ from $\dlstruct$ by setting, for each \mbox{$\pr\in\sigma(\sts)$}, \mbox{$\interpretpgp{A}\eqdef\interpretgp{\bar D}$}.
                      It follows directly that \mbox{$\dlstruct'\models\standbs[C\dlsub\standvu A]$} and \mbox{$\dlstruct'\models\standbs[A\dlsub\bar D]$}.

                      (b)~
                      Let $\dlstruct'\models\T'$ and consider any $\pr\in\sigma'(\sts)$ and $\de\in\interpretpgp{C}$.
                      From $\dlstruct'\models\standbs[C\dlsub\standvu A]$ we get $\de\in\interpretpgp{(\standvu A)}$.
                      Thus, for some (all) $\pr'\in\sigma'(\stu)$, we get $\de\in\interpretpgpp{A}$;
                      and in turn, by $\dlstruct'\models\standbs[A\dlsub\bar D]$, for some (all) $\pr'\in\sigma'(\stu)$, we get $\de\in\interpretpgpp{\bar D}$.
                  \item[Rule~(\ref{norm:trivial}):]
                      Clear, as any DL interpretation satisfies \mbox{$C\dlsub\top$} and \mbox{$\bot\dlsub D$}.
                  \item[Rule~(\ref{norm:existential-left}):]
                      (a)~
                      Let \mbox{$\dlstruct\models\standbs[\exists R.\bar C \dlsub D]$}.
                      Define $\dlstruct'$ from $\dlstruct$ as follows:
                      For any \mbox{$\pr\in\sigma(\sts)$}, set \mbox{$\interpretpgp{A}\eqdef\interpretgp{\bar C}$}.
                      It follows from this definition that \mbox{$\dlstruct'\models\standbs[\bar C\dlsub A]$} and \mbox{$\dlstruct'\models\standbs[\exists R.A\dlsub D]$}.

                      (b)~
                      Let \mbox{$\dlstruct'\models\T'$} and consider \mbox{$\pr\in\sigma'(\sts)$} and any \mbox{$\de\in\interpretpgp{\exists R.\bar C}$}.
                      There is thus a \mbox{$\ve\in\Dom$} with \mbox{$\tuple{\de,\ve}\in\interpretpgp{r}$} and \mbox{$\ve\in\interpretpgp{\bar C}$}.
                      By \mbox{$\dlstruct'\models\standbs[\bar C\dlsub A]$} we get \mbox{$\ve\in\interpretpgp{A}$}, and in turn \mbox{$\de\in\interpretpgp{(\exists R.A)}$}.
                      By \mbox{$\dlstruct'\models\standbs[\exists R.A\dlsub D]$}, we get \mbox{$\de\in\interpretpgp{D}$}.
                  \item[Rule~(\ref{norm:conjunction-left}):]
                      (a)~
                      Let \mbox{$\dlstruct\models\standbs[\bar C\dland D\dlsub E]$}.
                      Define $\dlstruct'$ from $\dlstruct$ as follows:
                      For any \mbox{$\pr\in\sigma(\sts)$}, set \mbox{$\interpretpgp{A}\eqdef\interpretgp{\bar C}$}.
                      It follows from this definition that \mbox{$\dlstruct'\models\standbs[\bar C\dlsub A]$} and \mbox{$\dlstruct'\models\standbs[A\dland D\dlsub E]$}.

                      (b)~
                      Let \mbox{$\dlstruct'\models\T'$} and consider \mbox{$\pr\in\sigma'(\sts)$} and \mbox{$\de\in\interpretpgp{\bar C} \cap \interpretpgp{D}$}.
                      By \mbox{$\dlstruct'\models\standbs[\bar C\dlsub A]$} we get \mbox{$\de\in\interpretpgp{A}$}, and
                      by \mbox{$\dlstruct'\models\standbs[A\dland D\dlsub E]$} and \mbox{$\de\in\interpretpgp{D}$} we get \mbox{$\de\in\interpretpgp{E}$}.
                  \item[Rule~(\ref{norm:diamond-left}):]
                      (a)~
                      Let \mbox{$\dlstruct\models\standbs[\standdu C\dlsub D]$} and define $\dlstruct'$ from $\dlstruct$ by setting \mbox{$\interpretpgp{A} \eqdef \interpretgp{(\standdu C)}$} for every \mbox{$\pr\in\Precs$}.
                      \begin{itemize}
                          \item $\dlstruct'\models\standbu [ C \dlsub \standball A ]$:
                                Let \mbox{$\pr\in\sigma'(\stu)=\sigma(\stu)$} and \mbox{$\de\in\interpretpgp{C}=\interpretgp{C}$}.
                                Then by definition, we have \mbox{$\de\in\interpretpgp{A}$} for all \mbox{$\pr\in\Precs$}, whence \mbox{$\de\in\interpretpgp{(\standball A)}$}.
                          \item $\dlstruct'\models\standbs  [ A \dlsub D ]$:
                                Let \mbox{$\pr\in\sigma'(\sts)=\sigma(\sts)$} and \mbox{$\de\in\interpretpgp{A}$}.
                                Then by definition, \mbox{$\de\in\interpretgp{(\standdu C)}$} and by the presumption that \mbox{$\dlstruct\models\standbs[\standdu C\dlsub D]$} we get \mbox{$\de\in\interpretgp{D}=\interpretpgp{D}$}.
                      \end{itemize}
                      (b)~
                      Let \mbox{$\dlstruct'\models\T'$} and consider \mbox{$\pr\in\sigma'(\sts)$} and \mbox{$\de\in\interpretpgp{(\standdu C)}$}.
                      Then there is a \mbox{$\pr'\in\sigma'(\stu)$} such that \mbox{$\de\in\interpretpgpp{C}$}.
                      Thus \mbox{$\de\in\interpretpgpp{(\standball A)}$}, that is, \mbox{$\de\in\bigcap_{\pr''\in\Precs}\interpretpgppp{A}$}.
                      Thus in particular for \mbox{$\pr\in\sigma'(\sts)$}, we get \mbox{$\de\in\interpretpgp{A}$} and by \mbox{$\dlstruct'\models\standbs [ A \dlsub D ]$} we obtain \mbox{$\de\in\interpretpgp{D}$}.

                  \item[Rule~(\ref{norm:box-left}):]
                      (a)~Let \mbox{$\dlstruct\models\standbs[\standbu C\dlsub D]$}.
                      % this already happens across cases above
                      %Let $\dlstruct = \tuple{\Dom, \Precs, \sigma, \gamma}$ be given.
                      We define \mbox{$\dlstruct' = \tuple{\Dom', \Precs', \sigma', \gamma'}$} as follows:
                      \begin{itemize}
                          \item $\Delta'$ consists of two copies $\delta'_0, \delta'_1$ of every function \mbox{$\de'\colon\sigma(\mathsf{u}) \times \Delta\to\Delta$}.
                          \item $\Precs'$ consists of all functions \mbox{$\pr'\colon\sigma(\stu) \times \Delta\to\Precs$} plus two extra, distinct copies of the particular function \mbox{$\pr'_\mathrm{diag} \eqdef \set{ (\pr,\de) \mapsto \pr }$}, denoted $\pr'_{\stv_0}$ and $\pr'_{\stv_1}$.
                          \item \mbox{$\sigma'(\stv_b)\eqdef\set{ \pr'_{\stv_b} }$} for each \mbox{$b\in\set{0,1}$} and, for all other $\sts$,
                                $$\sigma'(\sts) \eqdef \set{ \pr' \in \Precs'\guard \mathop{\bigcup_{\pr \in \sigma(\stu),}}_{\delta \in \Dom} \set{\pr'(\pr,\de)} \subseteq \sigma(\sts) }$$
                      \end{itemize}
                      For every \mbox{$\pr'\in\Precs'$}, the DL interpretation \mbox{$\gamma'(\pr')$} over $\Dom'$ is such that
                      \begin{align*}
                          \interpret{a}{\gamma'(\pr')} & \eqdef \set{ (\pr,\de) \mapsto \delta_a}_0 & \text{ for } a\in\Individuals
                      \end{align*}
                      where \mbox{$\de_a \in \Delta$} denotes the domain element for which \mbox{$\de_a = a^{\gamma(\pr)}$} for all \mbox{$\pr \in \Precs$}.\\
                      Further, for \mbox{$\pr'\in\Precs' \setminus \{\pr'_{\stv_0},\pr'_{\stv_1}\}$}, the DL interpretation \mbox{$\gamma'(\pr')$} over $\Dom'$ is such that
                      \begin{align*}
                          \interpret{A}{\gamma'(\pr')}                                         & \eqdef \set{ \de'_b \in\Dom' \guard \de'_b (\pi,\delta)\in\interpret{A}{\gamma(\pr'(\pi,\delta))} \text{ for all } {\pr \in \sigma(\mathsf{u})},\ {\delta \in \Delta}} & \text{ for } A\in\Concepts \\
                          \interpret{R}{\gamma'(\pr')}                                         & \eqdef \set{ \tuple{\de'_b,\ve'_b}\in\Dom'\times\Dom' \guard \tuple{\de'_b(\pi,\delta),\ve'_b(\pi,\delta)}\in\interpret{R}{\gamma(\pr'(\pi,\delta))}
                          \text{ for all } {\pr \in \sigma(\mathsf{u})},\ {\delta \in \Delta}} & \text{ for } R\in\Roles
                      \end{align*}
                      and \mbox{$\gamma'(\prvb)$} assigns as follows:
                      \begin{align*}
                          \interpret{A}{\gamma'(\prvb)}                                                                                                                                                                                  & \eqdef \left( \set{ \de'_b\in\Dom' } \cap \interpret{A}{\gamma'(\pr'_\mathrm{diag})} \right)
                          \cup  \set{ \de'_{1-b}\in\Dom' \guard \bigcup_{{\pr \in \sigma(\mathsf{u}),}\atop{\delta \in \Delta}} \set{ \de'_{1-b}(\pr,\de) } \subseteq{\bigcap_{\pr \in \sigma(\mathsf{u})}}\interpret{A}{\gamma(\pr)}  } & \text{ for } A\in\Concepts                                                                                             \\
                          \interpret{R}{\gamma'(\prvb)}                                                                                                                                                                                  & \eqdef \left( \set{ \tuple{\de'_b,\ve'_b}\in\Dom'\times\Dom' } \cap \interpret{R}{\gamma'(\pr'_\mathrm{diag})} \right) \\
                                                                                                                                                                                                                                         & \quad \cup  \set{ \tuple{\de'_{1-b},\ve'_{1-b}}\in\Dom'\times\Dom' \guard
                          \bigcup_{{\pr \in \sigma(\mathsf{u}),}\atop{\delta \in \Delta}} \set{ \tuple{\de'_{1-b}(\pr,\de), \ve'_{1-b}(\pr,\de)}} \subseteq{\bigcap_{\pr \in \sigma(\mathsf{u})}}\interpret{R}{\gamma(\pr)}  }                                                                                                                                    \\
                                                                                                                                                                                                                                         & \quad \cup  \set{ \tuple{\de'_{1-b},\ve'_b}\in\Dom'\times\Dom' \guard
                          \tuple{\zeta'_b,\ve'_b}\in \interpret{R}{\gamma'(\pr'_\mathrm{diag})} \text{ for some }  \zeta'_b \approx \de'_{1-b}
                          }
                                                                                                                                                                                                                                         & \text{ for } R\in\Roles
                      \end{align*}
                      where \mbox{$\zeta'_b \approx \de'_{1-b} $} holds iff for every \mbox{$\pr \in \sigma(\mathsf{u})$} we have
                      $
                          \textstyle\bigcup_{{\delta \in \Delta}} \{\zeta'_b(\pi,\delta)\}
                          =
                          \textstyle\bigcup_{{\pr_* \in \sigma(\mathsf{u}),}\atop{\delta_* \in \Delta}} \{\de'_{1-b}(\pi_*,\delta_*) \}
                      $.
                      Finally for the fresh concept name $A$ introduced by the rule, define \mbox{$\interpretpgpp{A} \eqdef \interpretpgpp{C}$} for all \mbox{$\pr'\in\sigma'(\stu)$}.

                      We now show that \mbox{$\dlstruct'\models\K'$}.
                      For this, we start out with some useful observations.
                      \begin{numberclaim}
                          \label{claim:norm:back-and-forth}
                          For every concept $E$, \mbox{$\de'\in\Dom'$}, and \mbox{$\pr'\in\Precs'\setminus\set{\prvo,\prvi}$}\/:
                          \[
                              \de'\in\interpretpgpp{E}
                              \iff
                              \left(
                              \forall\pr\in\sigma(\stu):\forall\de\in\Dom:\de'(\pr,\de)\in\interpret{E}{\gamma(\pr'(\pr,\de))}
                              \right)
                          \]
                          \begin{claimproof}
                              We use structural induction on $E$.
                              \begin{itemize}
                                  \item $E=A\in\BCC$: By definition.
                                  \item $E=E_1\dland E_2$:
                                        \begin{align*}
                                             & \phiff \de'\in\interpretpgpp{(E_1\dland E_2)}                                                                                                                                            \\
                                             & \iff \de'\in\interpretpgpp{E_1}\cap\interpretpgpp{E_2}                                                                                                                                   \\
                                             & \iff \de'\in\interpretpgpp{E_1} \text{ and } \de'\in\interpretpgpp{E_2}                                                                                                                  \\
                                             & \stackrel{\text{(IH)}}{\iff} \left[\forall\pr\in\sigma(\stu):\forall\de\in\Dom:\de'(\pr,\de)\in\interpret{E_1}{\gamma(\pr'(\pr,\de))}\right]
                                            \text{ and } \left[\forall\pr\in\sigma(\stu):\forall\de\in\Dom:\de'(\pr,\de)\in\interpret{E_2}{\gamma(\pr'(\pr,\de))}\right]                                                                \\
                                             & \iff \forall\pr\in\sigma(\stu):\forall\de\in\Dom:\left[\de'(\pr,\de)\in\interpret{E_1}{\gamma(\pr'(\pr,\de))} \text{ and } \de'(\pr,\de)\in\interpret{E_2}{\gamma(\pr'(\pr,\de))}\right] \\
                                             & \iff \forall\pr\in\sigma(\stu):\forall\de\in\Dom:\de'(\pr,\de)\in\interpret{E_1}{\gamma(\pr'(\pr,\de))} \cap \interpret{E_2}{\gamma(\pr'(\pr,\de))}                                      \\
                                             & \iff \forall\pr\in\sigma(\stu):\forall\de\in\Dom:\de'(\pr,\de)\in\interpret{(E_1\dland E_2)}{\gamma(\pr'(\pr,\de))}
                                        \end{align*}
                                  \item $E=\exists R.B$:
                                        \begin{align*}
                                             & \phiff \de'\in\interpretpgpp{\exists R.B}                                                                                                                                                                                                                                                    \\
                                             & \iff \exists\ve'\in\Dom': \left[ \tuple{\de',\ve'}\in\interpretpgpp{R} \text{ and } \ve'\in\interpretpgpp{B}\right]                                                                                                                                                                          \\
                                             & \stackrel{\text{(IH)}}{\iff} \exists\ve'\in\Dom': \left[ \tuple{\de',\ve'}\in\interpretpgpp{R} \text{ and } \forall\pr\in\sigma(\stu):\forall\de\in\Dom:\ve'(\pr,\de)\in\interpret{B}{\gamma(\pr'(\pr,\de))} \right]                                                                         \\
                                             & \iff \exists\ve'\in\Dom': \left[ \left[ \forall\pr\in\sigma(\stu):\forall\de\in\Dom:\tuple{\de'(\pr,\de),\ve'(\pr,\de)}\in\interpret{R}{\gamma(\pr'(\pr,\de))} \right] \text{ and } \forall\pr\in\sigma(\stu):\forall\de\in\Dom:\ve'(\pr,\de)\in\interpret{B}{\gamma(\pr'(\pr,\de))} \right] \\
                                             & \iff \exists\ve'\in\Dom': \left[ \forall\pr\in\sigma(\stu):\forall\de\in\Dom: \left[ \tuple{\de'(\pr,\de),\ve'(\pr,\de)}\in\interpret{R}{\gamma(\pr'(\pr,\de))} \text{ and } \ve'(\pr,\de)\in\interpret{B}{\gamma(\pr'(\pr,\de))} \right] \right]                                            \\
                                             & \stackrel{\ddagger}{\iff} \forall\pr\in\sigma(\stu):\forall\de\in\Dom:\exists\ve'\in\Dom': \left[ \tuple{\de'(\pr,\de),\ve'(\pr,\de)}\in\interpret{R}{\gamma(\pr'(\pr,\de))} \text{ and } \ve'(\pr,\de)\in\interpret{B}{\gamma(\pr'(\pr,\de))} \right]                                       \\
                                             & \stackrel{\dagger}{\iff} \forall\pr\in\sigma(\stu):\forall\de\in\Dom:\exists\ve\in\Dom: \left[ \tuple{\de'(\pr,\de),\ve}\in\interpret{R}{\gamma(\pr'(\pr,\de))} \text{ and } \ve\in\interpret{B}{\gamma(\pr'(\pr,\de))} \right]                                                              \\
                                             & \iff \forall\pr\in\sigma(\stu):\forall\de\in\Dom:\de'(\pr,\de)\in\interpret{(\exists R.B)}{\gamma(\pr'(\pr,\de))}
                                        \end{align*}
                                        Note that for the ``$\impliedby$'' direction, there is always an \mbox{$\ve'\in\Dom'$} such that for all \mbox{$\pr\in\sigma(\stu)$} and \mbox{$\de\in\Dom$} we find \mbox{$\ve'(\pr,\de)=\ve$} for the appropriate \mbox{$\ve\in\Dom$} (step $\dagger$), and thus this $\ve'$ is in this sense independent of concrete \mbox{$\pr\in\sigma(\stu)$} and \mbox{$\de\in\Dom$} (step $\ddagger$).
                                  \item $E=\standdsp B$:
                                        \begin{align*}
                                             & \phiff \de'\in\interpretpgpp{(\standdsp B)}                                                                                                                   \\
                                             & \iff \exists\pr''\in\sigma'(\stsp):\de'\in\interpretpgppp{B}                                                                                                  \\
                                             & \stackrel{\text{(IH)}}{\iff} \exists\pr''\in\sigma'(\stsp):\forall\pr\in\sigma(\stu):\forall\de\in\Dom: \de'(\pr,\de)\in\interpret{B}{\gamma(\pr''(\pr,\de))} \\
                                             & \iff \forall\pr\in\sigma(\stu):\forall\de\in\Dom:\exists\pr''\in\sigma'(\stsp):\de'(\pr,\de)\in\interpret{B}{\gamma'(\pr''(\pr,\de))}                         \\
                                             & \iff \forall\pr\in\sigma(\stu):\forall\de\in\Dom:\exists\pr_\stsp\in\sigma(\stsp):\de'(\pr,\de)\in\interpret{B}{\gamma(\pr_\stsp)}                            \\
                                             & \iff \forall\pr\in\sigma(\stu):\forall\de\in\Dom:\de'(\pr,\de)\in\interpret{(\standdsp B)}{\gamma(\pr'(\pr,\de))}
                                        \end{align*}
                                        Again, for ``$\impliedby$'', we find a \mbox{$\pr''\in\sigma(\stsp)$} for which always \mbox{$\pr''(\pr,\de)=\pr_\stsp$} as desired.
                                  \item $E=\standbsp B$:
                                        \begin{align*}
                                             & \phiff \de'\in\interpretpgpp{(\standbsp B)}                                                                                                                   \\
                                             & \iff \forall\pr''\in\sigma'(\stsp):\de'\in\interpretpgpp{B}                                                                                                   \\
                                             & \stackrel{\text{(IH)}}{\iff} \forall\pr''\in\sigma'(\stsp):\forall\pr\in\sigma(\stu):\forall\de\in\Dom: \de'(\pr,\de)\in\interpret{B}{\gamma(\pr''(\pr,\de))} \\
                                             & \iff \forall\pr\in\sigma(\stu):\forall\de\in\Dom:\forall\pr''\in\sigma'(\stsp):\de'(\pr,\de)\in\interpret{B}{\gamma'(\pr''(\pr,\de))}                         \\
                                             & \stackrel{\dagger}{\iff} \forall\pr\in\sigma(\stu):\forall\de\in\Dom:\forall\pr_\stsp\in\sigma(\stsp):\de'(\pr,\de)\in\interpret{B}{\gamma(\pr_\stsp)}        \\
                                             & \iff \forall\pr\in\sigma(\stu):\forall\de\in\Dom:\de'(\pr,\de)\in\interpret{(\standbsp B)}{\gamma(\pr'(\pr,\de))}
                                        \end{align*}
                                        The equivalence marked $\dagger$ holds because for every \mbox{$\pr_\stsp\in\sigma(\stsp)$} we can find a \mbox{$\pr''\in\sigma'(\stsp)$} such that \mbox{$\pr''(\pr,\de)=\pr_\stsp$} for all \mbox{$\pr\in\sigma(\stu)$} and \mbox{$\de\in\Dom$} (e.g.\ the constant function \mbox{$\set{(\pr,\de)\mapsto \pr_\stsp}$});
                                        conversely, for every \mbox{$\pr''\in\sigma'(\stsp)$}, \mbox{$\pr\in\sigma(\stu)$}, and \mbox{$\de\in\Dom$}, we have \mbox{$\pr''(\pr,\de)\in\sigma(\stsp)$} by definition.
                              \end{itemize}
                          \end{claimproof}
                      \end{numberclaim}
                      The claim above does not consider the newly introduced precisifications.
                      They are however covered by the next claim.
                      \begin{numberclaim}
                          \label{claim:norm:back-and-forth:v}
                          For every concept $E$, \mbox{$\de'\in\Dom'$}, and \mbox{$\pr'\in\set{\prvo,\prvi}$}\/:
                          \begin{enumerate}
                              \item\label{claim:norm:back-and-forth:v:back} $\de'\in\interpretpgpp{E}                                                                                                        \implies \forall\pr\in\sigma(\stu):\forall\de\in\Dom:\de'(\pr,\de)\in\interpret{E}{\gamma(\pr'(\pr,\de))}$
                              \item\label{claim:norm:back-and-forth:v:forth} $\forall\pr\in\sigma(\stu):\forall\de\in\Dom:\forall\pr_*\in\sigma(\stu): \de'(\pr,\de)\in\interpret{E}{\gamma(\pr'(\pr_*,\de))} \implies \de'\in\interpretpgpp{E}$
                              % \item $\de'\in\interpretpgpp{E}$ implies [$\de'(\pr,\de)\in\interpret{E}{\gamma(\pr'(\pr,\de))}$ for all $\pr\in\sigma(\stu)$ and $\de\in\Dom$];
                              % \item{} [$\de'(\pr,\de)\in\interpret{E}{\gamma(\pr'(\pr_*,\de))}$ for all $\pr,\pr_*\in\sigma(\stu)$ and $\de\in\Dom$] implies $\de'\in\interpretpgpp{E}$.
                          \end{enumerate}
                          \begin{claimproof}
                              We use structural induction on $E$.
                              \begin{itemize}
                                  \item $E=A\in\BCC$:
                                        Assume \mbox{$\pr'=\prvo$} (the remaining case is symmetric).
                                        For direction (\ref{claim:norm:back-and-forth:v:back}) we have\/:
                                        \begin{align*}
                                             & \phiff \de'_b\in\interpret{A}{\gamma'(\prvo)} \text{ for some } b\in\set{0,1}                                                                                                                                                                              \\
                                             & \iff \de'_0\in\interpret{A}{\gamma'(\pr'_{\mathrm{diag}})} \text{ or } \mathop{\bigcup_{\pr\in\sigma(\stu),}}_{\de\in\Dom}\set{\de'_1(\pr,\de)}\subseteq\bigcap_{\pr\in\sigma(\stu)}\interpretgp{A}                                                        \\
                                             & \iff \forall\pr\in\sigma(\stu):\forall\de\in\Dom:\de'_0(\pr,\de)\in\interpret{A}{\gamma(\pr'_{\mathrm{diag}}(\pr,\de))} \text{ or } \forall\pr\in\sigma(\stu):\forall\de\in\Dom:\forall\pr_*\in\sigma(\stu):\de'_1(\pr,\de)\in\interpret{A}{\gamma(\pr_*)} \\
                                             & \iff \forall\pr\in\sigma(\stu):\forall\de\in\Dom:\de'_0(\pr,\de)\in\interpret{A}{\gamma(\pr)} \text{ or } \forall\pr\in\sigma(\stu):\forall\de\in\Dom:\forall\pr_*\in\sigma(\stu):\de'_1(\pr,\de)\in\interpret{A}{\gamma(\pr_*)}                           \\
                                             & \implies \forall\pr\in\sigma(\stu):\forall\de\in\Dom:\de'_0(\pr,\de)\in\interpret{A}{\gamma(\pr)} \text{ or } \forall\pr\in\sigma(\stu):\forall\de\in\Dom:\de'_1(\pr,\de)\in\interpret{A}{\gamma(\pr)}                                                     \\
                                             & \iff \forall\pr\in\sigma(\stu):\forall\de\in\Dom:\de'_b(\pr,\de)\in\interpret{A}{\gamma(\pr)} \text{ for some } b\in\set{0,1}                                                                                                                              \\
                                             & \iff \forall\pr\in\sigma(\stu):\forall\de\in\Dom:\de'_b(\pr,\de)\in\interpret{A}{\gamma(\pr'_{\mathrm{diag}}(\pr,\de))} \text{ for some } b\in\set{0,1}                                                                                                    \\
                                             & \iff \forall\pr\in\sigma(\stu):\forall\de\in\Dom:\de'_b(\pr,\de)\in\interpret{A}{\gamma(\prvo(\pr,\de))} \text{ for some } b\in\set{0,1}
                                        \end{align*}
                                        In direction (\ref{claim:norm:back-and-forth:v:forth}), we have\/:
                                        \begin{align*}
                                             & \phiff \forall\pr\in\sigma(\stu):\forall\de\in\Dom:\forall\pr_*\in\sigma(\stu): \de'_b(\pr,\de)\in\interpret{E}{\gamma(\pr'(\pr_*,\de))} \text{ for some } b\in\set{0,1}                                                                                                             \\
                                             & \iff \forall\pr\in\sigma(\stu):\forall\de\in\Dom:\forall\pr_*\in\sigma(\stu): \de'_0(\pr,\de)\in\interpret{E}{\gamma(\pr'(\pr_*,\de))} \text{ or } \forall\pr\in\sigma(\stu):\forall\de\in\Dom:\forall\pr_*\in\sigma(\stu): \de'_1(\pr,\de)\in\interpret{E}{\gamma(\pr'(\pr_*,\de))} \\
                                             & \implies \forall\pr\in\sigma(\stu):\forall\de\in\Dom:\de'_0(\pr,\de)\in\interpret{A}{\gamma(\pr)} \text{ or } \forall\pr\in\sigma(\stu):\forall\de\in\Dom:\forall\pr_*\in\sigma(\stu):\de'_1(\pr,\de)\in\interpret{A}{\gamma(\pr_*)}
                                        \end{align*}
                                        In the last line, we continue with the equivalences from direction (\ref{claim:norm:back-and-forth:v:back}) above the line with ``$\implies\!$''.
                                  \item $E=\exists R.B$: We show the claim for \mbox{$\pr'=\prvo$}, as the other case is symmetric.
                                        In direction (\ref{claim:norm:back-and-forth:v:back}), let \mbox{$\de'_b\in\interpretpgpp{(\exists R.B)}$}.
                                        By definition of the semantics, there is a \mbox{$\ve'_c\in\Dom'$} such that \mbox{$\tuple{\de'_b,\ve'_c}\in\interpretpgpp{R}$} and \mbox{$\ve'_c\in\interpretpgpp{B}$}.
                                        Applying the induction hypothesis to \mbox{$\ve'_c\in\interpretpgpp{B}$} yields that
                                        \begin{gather}
                                            \label{eq:back-and-forth:v:exists:ih}
                                            \forall\pr\in\sigma(\stu):\forall\de\in\Dom:\ve'_c(\pr,\de)\in\interpret{B}{\gamma(\pr'(\pr,\de))}=\interpret{B}{\gamma(\pr)}
                                        \end{gather}
                                        By definition of $\interpretpgpp{R}$, we have three cases for \mbox{$\tuple{\de'_b,\ve'_c}\in\interpretpgpp{R}$}, in each of which we show the claim.
                                        \begin{enumerate}
                                            \item $\tuple{\de'_0,\ve'_0}\in\interpret{R}{\gamma'(\pr'_{\mathrm{diag}})}$.
                                                  That is, \mbox{$\tuple{\de'_0(\pr,\de),\ve'_0(\pr,\de)}\in\interpret{R}{\gamma(\pr)}$} for all \mbox{$\pr\in\sigma(\stu)$} and \mbox{$\de\in\Dom$}.
                                                  Together with \Cref{eq:back-and-forth:v:exists:ih}, we obtain
                                                  \[
                                                      \forall\pr\in\sigma:\forall\de\in\Dom:\left( \tuple{\de'_0(\pr,\de),\ve'_0(\pr,\de)}\in\interpret{R}{\gamma(\pr)} \text{ and } \ve'_c(\pr,\de)\in\interpret{B}{\gamma(\pr)} \right)
                                                  \]
                                                  Now since $\ve'_0$ and $\ve'_1$ are two distinct copies of the same function, this means that
                                                  \[
                                                      \forall\pr\in\sigma:\forall\de\in\Dom:\left( \tuple{\de'_0(\pr,\de),\ve'_0(\pr,\de)}\in\interpret{R}{\gamma(\pr)} \text{ and } \ve'_0(\pr,\de)\in\interpret{B}{\gamma(\pr)} \right)
                                                  \]
                                                  Thus
                                                  \[
                                                      \forall\pr\in\sigma:\forall\de\in\Dom:\exists\ve\in\Dom:\left( \tuple{\de'_0(\pr,\de),\ve}\in\interpret{R}{\gamma(\pr)} \text{ and } \ve\in\interpret{B}{\gamma(\pr)} \right)
                                                  \]
                                                  namely we set \mbox{$\ve\eqdef\ve'(\pr,\de)$} in each case.
                                                  By the definition of the semantics we get
                                                  \[
                                                      \forall\pr\in\sigma:\forall\de\in\Dom:\de'(\pr,\de)\in\interpret{(\exists R.B)}{\gamma(\pr)}
                                                  \]
                                                  which, by \mbox{$\pr'=\prvo=\pr'_{\mathrm{diag}}$} means
                                                  \[
                                                      \forall\pr\in\sigma:\forall\de\in\Dom:\de'(\pr,\de)\in\interpret{(\exists R.B)}{\gamma(\pr'(\pr,\de))}
                                                  \]
                                            \item $\forall\pr\in\sigma(\stu):\forall\de\in\Dom:\forall\pr_*\in\sigma(\stu):\tuple{\de'_1(\pr,\de),\ve'_1(\pr,\de)}\in\interpret{R}{\gamma(\pr_*)}$.
                                                  In combination with \Cref{eq:back-and-forth:v:exists:ih}, we get
                                                  \[
                                                      \forall\pr\in\sigma(\stu):\forall\de\in\Dom:\forall\pr_*\in\sigma(\stu):
                                                      \left(
                                                      \tuple{\de'_1(\pr,\de),\ve'_1(\pr,\de)}\in\interpret{R}{\gamma(\pr_*)}
                                                      \text{ and }
                                                      \ve'_c(\pr,\de)\in\interpret{B}{\gamma(\pr)}
                                                      \right)
                                                  \]
                                                  which, by virtue of $\ve'_0$ and $\ve'_1$ being the same function, means
                                                  \[
                                                      \forall\pr\in\sigma(\stu):\forall\de\in\Dom:\forall\pr_*\in\sigma(\stu):
                                                      \left(
                                                      \tuple{\de'_1(\pr,\de),\ve'_1(\pr,\de)}\in\interpret{R}{\gamma(\pr_*)}
                                                      \text{ and }
                                                      \ve'_1(\pr,\de)\in\interpret{B}{\gamma(\pr)}
                                                      \right)
                                                  \]
                                                  which in particular implies
                                                  \[
                                                      \forall\pr\in\sigma(\stu):\forall\de\in\Dom:
                                                      \left(
                                                      \tuple{\de'_1(\pr,\de),\ve'_1(\pr,\de)}\in\interpret{R}{\gamma(\pr)}
                                                      \text{ and }
                                                      \ve'_1(\pr,\de)\in\interpret{B}{\gamma(\pr)}
                                                      \right)
                                                  \]
                                                  which implies the claim as in the case above.
                                            \item there is some \mbox{$\zeta'_0\in\Dom'$} with \mbox{$\zeta'_0\approx\de'_1$} and \mbox{$\tuple{\zeta'_0,\ve'_0}\in\interpret{R}{\gamma'(\pr'_{\mathrm{diag}})}$}.
                                                  The latter property implies
                                                  \[
                                                      \forall\pr\in\sigma(\stu):\forall\de\in\Dom:\tuple{\zeta'_0(\pr,\de),\ve'_0(\pr,\de)}\in\interpret{R}{\gamma(\pr)}
                                                  \]
                                                  In turn, from \mbox{$\zeta'_0\approx\de'_0$} we obtain
                                                  \[
                                                      \forall\pr\in\sigma(\stu):\bigcup_{\de\in\Dom}\set{\zeta'_0(\pr,\de)}=\mathop{\bigcup_{\pr_*\in\sigma(\stu),}}_{\de_*\in\Dom}\set{\de'_1(\pr_*,\de_*)}
                                                  \]
                                                  This means in particular that
                                                  \[
                                                      \forall\pr,\pr_*\in\sigma(\stu):\forall\de_*\in\Dom:\exists\de_+\in\Dom:\de'_1(\pr_*,\de_*)=\zeta'_0(\pr,\de_+)
                                                  \]
                                                  Thus
                                                  \[
                                                      \forall\pr,\pr_*\in\sigma(\stu):\forall\de_*\in\Dom:\exists\de_+\in\Dom:\tuple{\de'_1(\pr_*,\de_*),\ve'_0(\pr,\de_+)}\in\interpret{R}{\gamma(\pr)}
                                                  \]
                                                  which we combine with \Cref{eq:back-and-forth:v:exists:ih} as usual to obtain
                                                  \[
                                                      \forall\pr,\pr_*\in\sigma(\stu):\forall\de_*\in\Dom:\exists\de_+\in\Dom:
                                                      \left(
                                                      \tuple{\de'_1(\pr_*,\de_*),\ve'_0(\pr,\de_+)}\in\interpret{R}{\gamma(\pr)}
                                                      \text{ and }
                                                      \ve'_c(\pr,\de_+)\in\interpret{B}{\gamma(\pr)}
                                                      \right)
                                                  \]
                                                  which in particular implies
                                                  \[
                                                      \forall\pr\in\sigma(\stu):\forall\de\in\Dom:\exists\de_+\in\Dom:
                                                      \left(
                                                      \tuple{\de'_1(\pr,\de),\ve'_0(\pr,\de_+)}\in\interpret{R}{\gamma(\pr)}
                                                      \text{ and }
                                                      \ve'_0(\pr,\de_+)\in\interpret{B}{\gamma(\pr)}
                                                      \right)
                                                  \]
                                                  whence we get
                                                  \[
                                                      \forall\pr\in\sigma(\stu):\forall\de\in\Dom:\de'_1(\pr,\de)\in\interpretgp{(\exists R.B)}
                                                  \]
                                        \end{enumerate}
                                        In direction (\ref{claim:norm:back-and-forth:v:forth}), assume that
                                        \[
                                            \forall\pr\in\sigma:\forall\de\in\Dom:\forall\pr_*\in\sigma(\stu):\de'_0(\pr,\de)\in\interpret{(\exists R.B)}{\gamma(\pr'(\prvo,\de))}
                                        \]
                                        By the definition of the semantics, we obtain
                                        \[
                                            \forall\pr\in\sigma:\forall\de\in\Dom:\forall\pr_*\in\sigma(\stu):\exists\ve\in\Dom:
                                            \left(
                                            \tuple{\de'_0(\pr,\de),\ve}\in\interpret{R}{\gamma(\pr'(\pr_*,\de))} \text{ and } \ve_{\pr,\de}\in\interpret{B}{\gamma(\pr'(\pr_*,\de))}
                                            \right)
                                        \]
                                        Since \mbox{$\pr'=\prvo=\pr'_{\mathrm{diag}}$}, this means that
                                        \[
                                            \forall\pr\in\sigma:\forall\de\in\Dom:\forall\pr_*\in\sigma(\stu):\exists\ve_{\pr,\de}\in\Dom:
                                            \left(
                                            \tuple{\de'(\pr,\de),\ve}\in\interpret{R}{\gamma(\pr_*)} \text{ and } \ve\in\interpret{B}{\gamma(\pr_*)}
                                            \right)
                                        \]
                                        Since $\Dom'$ contains every function \mbox{$\zeta'\colon\sigma(\stu)\times\Dom\to\Dom$}, there in particular exists a function \mbox{$\ve'_0\in\Dom'$} such that \mbox{$\ve'_0(\pr,\de)=\ve_{\pr,\de}$} for all \mbox{$\pr\in\sigma(\stu)$} and \mbox{$\de\in\Dom$}.
                                        Thus
                                        \[
                                            \exists\ve'_0\in\Dom':\forall\pr\in\sigma(\stu):\forall\de\in\Dom:\forall\pr_*\in\sigma(\stu):
                                            \left(
                                            \tuple{\de'_0(\pr,\de),\ve'_0(\pr,\de)}\in\interpret{R}{\gamma(\pr_*)} \text{ and } \ve'_0(\pr,\de)\in\interpret{B}{\gamma(\pr_*)}
                                            \right)
                                        \]
                                        which means
                                        \begin{multline*}
                                            \left(
                                            \exists\ve'_0\in\Dom':\forall\pr\in\sigma(\stu):\forall\de\in\Dom:\forall\pr_*\in\sigma(\stu):
                                            \tuple{\de'_0(\pr,\de),\ve'_0(\pr,\de)}\in\interpret{R}{\gamma(\pr_*)}
                                            \right)
                                            \text{ and } \\
                                            \left(
                                            \exists\ve'_0\in\Dom':\forall\pr\in\sigma(\stu):\forall\de\in\Dom:\forall\pr_*\in\sigma(\stu):
                                            \ve'_0(\pr,\de)\in\interpret{B}{\gamma(\pr_*)}
                                            \right)
                                        \end{multline*}
                                        where we can apply the induction hypothesis to the second line to obtain
                                        \begin{multline*}
                                            \left(
                                            \exists\ve'_0\in\Dom':\forall\pr\in\sigma(\stu):\forall\de\in\Dom:\forall\pr_*\in\sigma(\stu):
                                            \tuple{\de'_0(\pr,\de),\ve'_0(\pr,\de)}\in\interpret{R}{\gamma(\pr_*)}
                                            \right)
                                            \text{ and } \\
                                            \ve'_0\in\interpret{B}{\gamma'(\prvo)}
                                        \end{multline*}
                                        which in particular means that
                                        \[
                                            \exists\ve'_0\in\Dom':
                                            \left(
                                            \left(
                                                \forall\pr\in\sigma(\stu):\forall\de\in\Dom:
                                                \tuple{\de'_0(\pr,\de),\ve'_0(\pr,\de)}\in\interpret{R}{\gamma(\pr)}
                                                \right)
                                            \text{ and } \ve'_0\in\interpret{B}{\gamma'(\prvo)}
                                            \right)
                                        \]
                                        which (employing among other things that \mbox{$\prvo(\pr,\de)=\pr$} for all \mbox{$\pr\in\sigma(\stu)$} and \mbox{$\de\in\Dom$}) implies
                                        \[
                                            \exists\ve'_0\in\Dom':\left( \tuple{\de'_0,\ve'_0}\in\interpret{R}{\gamma'(\prvo)} \text{ and } \ve'\in\interpret{B}{\gamma'(\prvo)} \right)
                                        \]
                                        that is,
                                        \(
                                        \de'_0\in\interpret{(\exists R.B)}{\gamma'(\prvo)}
                                        \).
                                  \item $E=\standdsp B$:
                                        In direction (\ref{claim:norm:back-and-forth:v:back}), assume \mbox{$\de'\in\interpretpgpp{(\standdsp B)}$}.
                                        Thus there exists some \mbox{$\pr''\in\sigma'(\stsp)$} such that \mbox{$\de'\in\interpretpgppp{B}$}.
                                        The induction hypothesis yields
                                        \[
                                            \forall\pr\in\sigma(\stu):\forall\de\in\Dom:\de'(\pr,\de)\in\interpret{B}{\gamma(\pr''(\pr,\de))}
                                        \]
                                        Thus for all \mbox{$\pr\in\sigma(\stu)$} and \mbox{$\de\in\Dom$} there exists a \mbox{$\pr_{\pr,\de}^{\stsp}\in\sigma(\stsp)$}, namely \mbox{$\pr_{\pr,\de}^\stsp\eqdef\pr''(\pr,\de)$}, such that \mbox{$\de'(\pr,\de)\in\interpret{B}{\gamma(\pr_{\pr,\de}^\stsp)}$}.
                                        This directly yields
                                        \(
                                        \forall\pr\in\sigma(\stu):\forall\de\in\Dom:\de'(\pr,\de)\in\interpretpgpp{(\standdsp B)}
                                        \).

                                        In direction (\ref{claim:norm:back-and-forth:v:forth}), assume
                                        \[
                                            \forall\pr\in\sigma(\stu):\forall\de\in\Dom:\forall\pr_*\in\sigma(\stu): \de'(\pr,\de)\in\interpret{(\standdsp B)}{\gamma(\prvo(\pr_*,\de))}
                                        \]
                                        This means that
                                        \[
                                            \forall\pr\in\sigma(\stu):\forall\de\in\Dom:\forall\pr_*\in\sigma(\stu):\exists\pr_{\pr_*,\de}^\stsp\in\sigma(\stsp): \de'(\pr,\de)\in\interpret{B}{\gamma(\pr_{\pr_*,\de}^\stsp)}
                                        \]
                                        Choose \mbox{$\pr'_\stsp\in\sigma'(\stsp)$} such that for all \mbox{$\pr_*\in\sigma(\stu)$} and \mbox{$\de\in\Dom$}, we have \mbox{$\pr'_\stsp(\pr_*,\de)=\pr_{\pr_*,\de}^\stsp$}.
                                        Thus we obtain
                                        \[
                                            \exists\pr'_\stsp\in\sigma'(\stsp):\forall\pr\in\sigma(\stu):\forall\de\in\Dom:\forall\pr_*\in\sigma(\stu): \de'(\pr,\de)\in\interpret{B}{\gamma(\pr'_\stsp(\pr_*,\de))}
                                        \]
                                        We apply the induction hypothesis and get
                                        \(
                                        \exists\pr'_\stsp\in\Dom':\de'\in\interpret{B}{\gamma'(\pr'_\stsp)}
                                        \), that is,
                                        \(
                                        \de'\in\interpretpgpp{(\standdsp B)}
                                        \).
                                  \item $E=\standbsp B$:
                                        In direction (\ref{claim:norm:back-and-forth:v:back}), we have
                                        \begin{align*}
                                             & \phimplies \de'\in\interpretpgpp{(\standbsp B)}                                                                                                                  \\
                                             & \implies \forall\pr''\in\sigma'(\stsp):\de'\in\interpretpgppp{B}                                                                                                 \\
                                             & \stackrel{\text{(IH)}}{\implies} \forall\pr''\in\sigma'(\stsp):\forall\pr\in\sigma(\stu):\forall\de\in\Dom:\de'(\pr,\de)\in\interpret{B}{\gamma(\pr''(\pr,\de))} \\
                                             & \implies \forall\pr\in\sigma(\stu):\forall\de\in\Dom:\forall\pr''\in\sigma'(\stsp):\de'(\pr,\de)\in\interpret{B}{\gamma(\pr''(\pr,\de))}                         \\
                                             & \stackrel{\dagger}{\implies} \forall\pr\in\sigma(\stu):\forall\de\in\Dom:\forall\pr_\stsp\in\sigma(\stsp):\de'(\pr,\de)\in\interpret{B}{\gamma(\pr_\stsp)}       \\
                                             & \implies \forall\pr\in\sigma(\stu):\forall\de\in\Dom:\de'(\pr,\de)\in\interpret{(\standbsp B)}{\gamma(\pr)}
                                        \end{align*}
                                        and for the converse direction, (\ref{claim:norm:back-and-forth:v:forth}), we consider
                                        \begin{align*}
                                             & \phimplies \forall\pr\in\sigma(\stu):\forall\de\in\Dom:\forall\pr_*\in\sigma(\stu): \de'(\pr,\de)\in\interpret{(\standbs B)}{\gamma(\pr'(\pr_*,\de))}                                     \\
                                             & \implies \forall\pr\in\sigma(\stu):\forall\de\in\Dom:\forall\pr_*\in\sigma(\stu):\forall\pr_\stsp\in\sigma(\stsp): \de'(\pr,\de)\in\interpret{B}{\gamma(\pr_\stsp)}                       \\
                                             & \stackrel{\dagger}{\implies} \forall\pr\in\sigma(\stu):\forall\de\in\Dom:\forall\pr_*\in\sigma(\stu):\forall\pr''\in\sigma'(\stsp): \de'(\pr,\de)\in\interpret{B}{\gamma(\pr''(\pr,\de))} \\
                                             & \implies \forall\pr''\in\sigma'(\stsp):\forall\pr\in\sigma(\stu):\forall\de\in\Dom:\forall\pr_*\in\sigma(\stu): \de'(\pr,\de)\in\interpret{B}{\gamma(\pr''(\pr,\de))}                     \\
                                             & \stackrel{\text{(IH)}}{\implies} \forall\pr''\in\sigma'(\stsp):\de'\in\interpret{B}{\gamma'(\pr'')}                                                                                       \\
                                             & \implies \de'\in\interpretpgpp{(\standbsp B)}
                                        \end{align*}
                                        Here, in both directions the implications marked $\dagger$ can be justifed just like in the proof of \Cref{claim:norm:back-and-forth}.
                              \end{itemize}
                          \end{claimproof}
                      \end{numberclaim}
                      The two previous claims can be combined to obtain the following “global” property:
                      \begin{numberclaim}
                          \label{claim:norm:forth:all}
                          For every concept $E$, \mbox{$\de'\in\Dom'$}, and \mbox{$\pr'\in\Precs'$}\/:
                          \[
                              \de'\in\interpretpgpp{E}
                              \implies
                              \left(
                              \forall\pr\in\sigma(\stu):\forall\de\in\Dom:\de'(\pr,\de)\in\interpret{E}{\gamma(\pr'(\pr,\de))}
                              \right)
                          \]
                          \begin{claimproof}
                              Follows from direction ``$\implies\!$'' of \Cref{claim:norm:back-and-forth} and \Cref{claim:norm:back-and-forth:v:back} of \Cref{claim:norm:back-and-forth:v}.
                          \end{claimproof}
                      \end{numberclaim}
                      For the new precisifications, we sometimes also have to pay attention to which copy of the functions in $\Dom'$ we are currently considering, especially for the main case we will need below, namely where both copies are contained in the same concept's extension in both precisifications.
                      \begin{numberclaim}
                          \label{claim:norm:forth:b}
                          For every concept $E$, function \mbox{$\de'\colon\sigma(\stu)\times\Dom\to\Dom$}, and \mbox{$b\in\set{0,1}$}:
                          \begin{align}
                              \de'_b \in \interpret{E}{\prvb}    & \implies \forall\pr\in\sigma(\stu):\forall\de\in\Dom:\de'(\pr,\de)\in\interpretgp{E}                                          \label{eq:same} \\
                              \de'_{1-b}\in \interpret{E}{\prvb} & \implies \forall\pr\in\sigma(\stu):\forall\de\in\Dom:\forall\pr_*\in\sigma(\stu):\de'(\pr,\de)\in\interpret{E}{\gamma(\pr_*)} \label{eq:diff}
                          \end{align}
                          \begin{claimproof}
                              We use structural induction on $E$.
                              \begin{itemize}
                                  \item $E=A\in\BCC$: By definition.
                                  \item $E=\exists R.B$:
                                        \begin{description}
                                            \item[\normalfont(\ref{eq:same}):]
                                                Let \mbox{$\de'_b\in\interpret{(\exists R.B)}{\gamma'(\prvb)}$} for some \mbox{$b\in\set{0,1}$}.
                                                By definition of the semantics, there is an \mbox{$\ve'_c\in\Dom'$} such that \mbox{$\tuple{\de'_b,\ve'_c}\in\interpret{R}{\gamma'(\prvb)}$} and \mbox{$\ve'_c\in\interpret{B}{\gamma'(\prvb)}$}.
                                                By definition of \mbox{$\interpret{R}{\gamma'(\prvb)}$} the only possible reason for \mbox{$\tuple{\de'_b,\ve'_c}\in\interpret{R}{\gamma'(\prvb)}$} is that
                                                \mbox{$b=c$} and \mbox{$\tuple{\de'_b,\ve'_c}\in\interpret{R}{\gamma'(\pr'_{\mathrm{diag}})}$}, that is,
                                                \[
                                                    \forall\pr\in\sigma(\stu):\forall\de\in\Dom:\tuple{\de'_b(\pr,\de),\ve'_c(\pr,\de)}\in\interpret{R}{\gamma(\pr)}
                                                \]
                                                Due to \mbox{$b=c$}, we can apply the induction hypothesis of (\ref{eq:same}) to \mbox{$\ve'_c\in\interpret{B}{\gamma'(\prvb)}$} and obtain
                                                \[
                                                    \forall\pr\in\stu:\forall\de\in\Dom:\ve'_c(\pr,\de)\in\interpretgp{B}
                                                \]
                                                Togher with the above, we obtain
                                                \[
                                                    \forall\pr\in\sigma:\forall\de\in\Dom:\left( \tuple{\de'_b(\pr,\de),\ve'_c(\pr,\de)}\in\interpret{R}{\gamma(\pr)} \text{ and } \ve'_c(\pr,\de)\in\interpret{B}{\gamma(\pr)} \right)
                                                \]
                                                Thus
                                                \[
                                                    \forall\pr\in\sigma:\forall\de\in\Dom:\exists\ve_{\pr,\de}\in\Dom:\left( \tuple{\de'_b(\pr,\de),\ve_{\pr,\de}}\in\interpret{R}{\gamma(\pr)} \text{ and } \ve_{\pr,\de}\in\interpret{B}{\gamma(\pr)} \right)
                                                \]
                                                namely we set \mbox{$\ve_{\pr,\de}\eqdef\ve'_c(\pr,\de)$} in each case.
                                                By the definition of the semantics we get
                                                \[
                                                    \forall\pr\in\sigma:\forall\de\in\Dom:\de'(\pr,\de)\in\interpret{(\exists R.B)}{\gamma(\pr)}
                                                \]
                                            \item[\normalfont(\ref{eq:diff}):]
                                                Let \mbox{$\de'_{1-b}\in\interpret{(\exists R.B)}{\gamma'(\prvb)}$}.
                                                By the definition of the semantics, there exists some \mbox{$\ve'_c\in\Dom'$} such that \mbox{$\tuple{\de'_{1-b},\ve'_c}\in\interpret{R}{\gamma'(\prvb)}$} and \mbox{$\ve'_c\in\interpret{B}{\gamma'(\prvb)}$}.
                                                According to the definition of $\interpret{R}{\gamma'(\prvb)}$, there are two possible cases:
                                                \begin{enumerate}
                                                    \item \mbox{$c=1-b$} and \mbox{$\forall\pr\in\sigma(\stu):\forall\de\in\Dom:\forall\pr_*\in\sigma(\stu):\tuple{\de'_{1-b}(\pr,\de),\ve'_{c}(\pr,\de)}\in\interpretgps{R}$}.
                                                          We can apply the induction hypothesis of (\ref{eq:diff}) to the fact that \mbox{$\ve'_c=\ve'_{1-b}\in\interpret{B}{\gamma'(\prvb)}$} and obtain
                                                          \[
                                                              \forall\pr\in\stu:\forall\de\in\Dom:\forall\pr_*\in\sigma(\stu):\ve'_c(\pr,\de)\in\interpretgps{B}
                                                          \]
                                                          In combination with the presumption of this case, we obtain
                                                          \[
                                                              \forall\pr\in\sigma(\stu):\forall\de\in\Dom:\forall\pr_*\in\sigma(\stu):
                                                              \left(
                                                              \tuple{\de'_b(\pr,\de),\ve'_c(\pr,\de)}\in\interpretgps{R}
                                                              \text{ and }
                                                              \ve'_c(\pr,\de)\in\interpretgps{B}
                                                              \right)
                                                          \]
                                                          thus proving the claim.
                                                    \item $c=b$ and there is some \mbox{$\ze'_b\in\Dom'$} with \mbox{$\ze'_b\approx\de'_{1-b}$} and \mbox{$\tuple{\ze'_b,\ve'_c}\in\interpret{R}{\gamma'(\pr'_{\mathrm{diag}})}$}.
                                                          The latter property implies
                                                          \[
                                                              \forall\pr_*\in\sigma(\stu):\forall\de\in\Dom:\tuple{\ze'_b(\pr_*,\de),\ve'_c(\pr_*,\de)}\in\interpret{R}{\gamma(\pr_*)}
                                                          \]
                                                          In turn, from \mbox{$\ze'_b\approx\de'_{1-b}$} we obtain
                                                          \[
                                                              \forall\pr_*\in\sigma(\stu):\bigcup_{\de\in\Dom}\set{\ze'_b(\pr_*,\de)}=\mathop{\bigcup_{\pr\in\sigma(\stu),}}_{\de\in\Dom}\set{\de'_{1-b}(\pr,\de)}
                                                          \]
                                                          This means in particular that
                                                          \[
                                                              \forall\pr\in\sigma(\stu):\forall\de\in\Dom:\forall\pr_*\in\sigma(\stu):\exists\de_+\in\Dom:\ze'_b(\pr_*,\de_+)=\de'_{1-b}(\pr,\de)
                                                          \]
                                                          whence
                                                          \[
                                                              \forall\pr\in\sigma(\stu):\forall\de\in\Dom:\forall\pr_*\in\sigma(\stu):\exists\de_+\in\Dom:\tuple{\de'_{1-b}(\pr,\de),\ve'_c(\pr_*,\de_+)}\in\interpretgps{R}
                                                          \]
                                                          Now \mbox{$b=c$} means that we can apply the induction hypothesis of (\ref{eq:same}) to \mbox{$\ve'_c=\ve'_b\in\interpret{B}{\gamma'(\prvb)}$}, from which we get
                                                          \[
                                                              \forall\pr_*\in\sigma(\stu):\forall\de\in\Dom:\ve'_b(\pr_*,\de)\in\interpretgps{B}
                                                          \]
                                                          which we combine with the line above as usual to obtain
                                                          \[
                                                              \forall\pr,\pr_*\in\sigma(\stu):\forall\de\in\Dom:\exists\de_+\in\Dom:
                                                              \left(
                                                              \tuple{\de'_{1-b}(\pr,\de),\ve'_c(\pr_*,\de_+)}\in\interpretgps{R}
                                                              \text{ and }
                                                              \ve'_c(\pr_*,\de_+)\in\interpretgps{B}
                                                              \right)
                                                          \]
                                                          which in particular implies
                                                          \[
                                                              \forall\pr\in\sigma(\stu):\forall\de\in\Dom:\forall\pr_*\in\sigma(\stu):\exists\ve\in\Dom:
                                                              \left(
                                                              \tuple{\de'_{1-b}(\pr,\de),\ve}\in\interpretgps{R}
                                                              \text{ and }
                                                              \ve\in\interpretgps{B}
                                                              \right)
                                                          \]
                                                          namely in each case \mbox{$\ve\eqdef\ve'_c(\pr_*,\de_+)$},
                                                          whence we get
                                                          \[
                                                              \forall\pr\in\sigma(\stu):\forall\de\in\Dom:\forall\pr_*\in\sigma(\stu):\de'_{1-b}(\pr,\de)\in\interpret{(\exists R.B)}{\gamma(\pr_*)}
                                                          \]
                                                          thus proving the claim.
                                                \end{enumerate}
                                        \end{description}
                                  \item $E=\standdsp B$:
                                        Let \mbox{$\pr'\in\set{\prvo,\prvi}$} and consider \mbox{$\de'_b\in\interpretpgpp{(\standdsp B)}$}.
                                        Then there is some \mbox{$\pr''\in\sigma'(\stsp)$} such that \mbox{$\de'_{1-b}\in\interpretpgppp{B}$}.
                                        By \Cref{claim:norm:forth:all} we obtain
                                        \[
                                            \forall\pr\in\sigma(\stu):\forall\de\in\Dom:
                                            \de'_{b}(\pr,\de)\in\interpret{B}{\gamma(\pr''(\pr,\de))}
                                        \]
                                        whence for all \mbox{$\pr\in\sigma(\stu)$} and \mbox{$\de\in\Dom$} we can set \mbox{$\pr_{\pr,\de}\eqdef\pr''(\pr,\de)$} to rewrite this into
                                        \[
                                            \forall\pr\in\sigma(\stu):\forall\de\in\Dom:\exists\pr_{\pr,\de}\in\sigma(\stsp):
                                            \de'_{b}(\pr,\de)\in\interpret{B}{\gamma(\pr_{\pr,\de})}
                                        \]
                                        and we conclude the claim:
                                        \[
                                            \forall\pr\in\sigma(\stu):\forall\de\in\Dom:\forall\pr_*\in\sigma(\stu):
                                            \de'_{b}(\pr,\de)\in\interpretgps{(\standdsp B)}
                                        \]
                                  \item $E=\standbsp B$:
                                        Let \mbox{$\pr'\in\set{\prvo,\prvi}$} and consider \mbox{$\de'_b\in\interpretpgpp{(\standbsp B)}$}.
                                        Then for all \mbox{$\pr''\in\sigma'(\stsp)$}, we have \mbox{$\de'_b\in\interpretpgppp{B}$}, which by \Cref{claim:norm:forth:all} means that
                                        \[
                                            \forall\pr''\in\sigma'(\stsp):\forall\pr\in\sigma(\stu):\forall\de\in\Dom:\de'_b(\pr,\de)\in\interpret{B}{\gamma(\pr''(\pr,\de))}
                                        \]
                                        Let \mbox{$\pr_\stsp\in\sigma(\stsp)$} be arbitrary.
                                        Consider the function \mbox{$\pr'_\stsp = \set{ (\pr,\de)\mapsto \pr_\stsp }$}, for which we have \mbox{$\pr'_\stsp\in\sigma'(\stsp)$} by definition
                                        and thus also
                                        \[
                                            \forall\pr\in\sigma(\stu):\forall\de\in\Dom:\de'_b(\pr,\de)\in\interpret{B}{\gamma(\pr'_\stsp(\pr,\de))}=\interpret{B}{\gamma(\pr_\stsp)}
                                        \]
                                        Since $\pr_\stsp$ was arbitrarily chosen, we get
                                        \[
                                            \forall\pr\in\sigma(\stu):\forall\de\in\Dom:
                                            \forall\pr_\stsp\in\sigma(\stsp):
                                            \de'_b(\pr,\de)\in\interpret{B}{\gamma(\pr_\stsp)}
                                        \]
                                        which shows the claim.
                              \end{itemize}
                          \end{claimproof}
                      \end{numberclaim}
                      We now continue the main proof, showing that \mbox{$\dlstruct'\models\K'$}.
                      \begin{itemize}
                          \item $\dlstruct'\models\standbu[C\dlsub A]$: By definition.
                          \item $\dlstruct'\models\standbs[\standdvo A \dland \standdvi A \dlsub D ]$:
                                Let \mbox{$\pr'\in\sigma'(\sts)$} and \mbox{$\de'_b\in\interpretpgpp{(\standdvo A\dland \standdvi A)}$} for some \mbox{$b\in\set{0,1}$}.
                                We have to show \mbox{$\de'_b\in\interpretpgpp{D}$}.
                                By definition of $\dlstruct'$, we conclude that \mbox{$\de'_b\in\interpret{A}{\gamma'(\prvo)} \cap \interpret{A}{\gamma'(\prvi)}$};
                                we furthermore have \mbox{$\prvo,\prvi\in\sigma'(\stu)$} and
                                \mbox{$\interpret{A}{\gamma'(\prvy)}=\interpret{C}{\gamma'(\prvy)}$} for all \mbox{$y\in\set{0,1}$};
                                thus it follows that \mbox{$\de'_b\in\interpret{C}{\gamma'(\prvo)} \cap \interpret{C}{\gamma'(\prvi)}$}.
                                So in any case we obtain that $\de'_b\in\interpret{C}{\gamma'(\pr_{\stv_{1-b}})}$, whence by \Cref{claim:norm:forth:b} we get
                                \[
                                    \forall\pr\in\sigma(\stu):\forall\de\in\Dom:\forall\pr_*\in\sigma(\stu):
                                    \de'(\pr,\de)\in\interpret{C}{\gamma(\pr_*)}
                                \]
                                Hence, for all \mbox{$\pr\in\sigma(\stu)$} and \mbox{$\de\in\Dom$} we have \mbox{$\de'(\pr,\de)\in\interpretgp{(\standbu C)}$}.
                                From \mbox{$\dlstruct\models\standbs[\standbu C\dlsub D]$} it therefore follows that \mbox{$\de'(\pr_*,\de)\in\interpretgp{D}$} for all \mbox{$\pr_*\in\sigma(\stu)$}, \mbox{$\de\in\Dom$}, and \mbox{$\pr\in\sigma(\sts)$}.
                                Now by \mbox{$\pr'\in\sigma'(\sts)$} we obtain that for any \mbox{$\pr_*\in\sigma(\stu)$} and \mbox{$\de\in\Dom$} we get \mbox{$\pr'(\pr_*,\de)\in\sigma(\sts)$}.
                                Thus it follows that \mbox{$\de'(\pr,\de)\in\interpret{D}{\gamma(\pr'(\pr_*,\de))}$} for all \mbox{$\pr,\pr_*\in\sigma(\stu)$} and \mbox{$\de\in\Dom$}.
                                By \Cref{claim:norm:back-and-forth}/\Cref{claim:norm:back-and-forth:v} above, we get \mbox{$\de'\in\interpretpgpp{D}$}.
                          \item $\dlstruct'\models\tau$ for all $\tau\in\T'\setminus\set{\standbs[C\dlsub A],\standbs[\standdvo A \dland \standdvi A \dlsub D]}$:
                                Let w.l.o.g.\ \mbox{$\tau=\standbsp[E\dlsub F]$}.
                                By presumption \mbox{$\dlstruct\models\T$} we get that for every \mbox{$\pr\in\sigma(\sts')$} we have \mbox{$\interpretgp{E}\subseteq\interpretgp{F}$}.
                                Let \mbox{$\pr'\in\sigma'(\sts')$} and \mbox{$\de'\in\interpretpgpp{E}$}.
                                From \Cref{claim:norm:forth:all} we get that \mbox{$\de'(\pr,\de)\in\interpret{E}{\gamma(\pr'(\pr,\de))}$} for all \mbox{$\pr\in\sigma(\stu)$} and \mbox{$\de\in\Dom$}.
                                From \mbox{$\pr'\in\sigma'(\sts')$} we get that \mbox{$\pr'(\pr_*,\de)\in\sigma(\sts')$} for all \mbox{$\pr_*\in\sigma(\stu)$} and \mbox{$\de\in\Dom$}.
                                It follows that \mbox{$\interpret{E}{\gamma(\pr'(\pr_*,\de))}\subseteq\interpret{F}{\gamma(\pr'(\pr_*,\de))}$} for all \mbox{$\pr_*\in\sigma(\stu)$} and \mbox{$\de\in\Dom$}.
                                In particular, \mbox{$\de'(\pr,\de)\in\interpret{F}{\gamma(\pr'(\pr_*,\de))}$} for all \mbox{$\pr,\pr_*\in\sigma(\stu)$} and \mbox{$\de\in\Dom$}.
                                Therefore, \Cref{claim:norm:back-and-forth}/\Cref{claim:norm:back-and-forth:v} yields \mbox{$\de'\in\interpretpgpp{F}$}.
                          \item $\dlstruct'\models\alpha$ for all $\alpha\in\A$: We do a case distinction on the possible forms of assertions.
                                \begin{itemize}
                                    \item $\alpha=\standbsp[B(a)]$:
                                          It follows that \mbox{$B\in\BCC$} since $\A$ is in normal form.
                                          From \mbox{$\dlstruct\models\A$} we get that for any \mbox{$\pr\in\sigma(\sts')$} we have \mbox{$\interpretgp{a}=\de_a\in\interpretgp{B}$}.
                                          Let \mbox{$\pr'\in\sigma'(\sts')$}.
                                          By definition, we get that \mbox{$\interpretpgpp{a}(\pr,\de)=\de_a$} for all \mbox{$\de\in\Dom$} and all \mbox{$\pr\in\Precs$} and in particular for all \mbox{$\pr\in\sigma(\stu)$}.
                                          From \mbox{$\pr'\in\sigma'(\sts')$} we get that for any \mbox{$\pr\in\sigma(\stu)$} and \mbox{$\de\in\Dom$} we find \mbox{$\pr'(\pr,\de)\in\sigma(\sts')$}.
                                          It thus follows that \mbox{$\interpretpgpp{a}(\pr,\de)=\de_a\in\interpret{B}{\gamma(\pr'(\pr,\de))}$} for all \mbox{$\pr\in\sigma(\stu)$} and \mbox{$\de\in\Dom$}.
                                          Consequently, \mbox{$\interpretpgpp{a}\in\interpretpgpp{B}$}.
                                    \item $\alpha=\standbsp[R(a,c)]$:
                                          From \mbox{$\dlstruct\models\A$} it follows that for every \mbox{$\pr\in\sigma(\sts')$} we have \mbox{$\tuple{\interpretgp{a},\interpretgp{c}}\in\interpretgp{R}$}.
                                          Let \mbox{$\pr'\in\sigma'(\sts')$}.
                                          As above, we get \mbox{$\interpretpgpp{a}(\pr,\de)=\de_a$} and \mbox{$\interpretpgpp{c}(\pr,\de)=\de_c$} for every \mbox{$\de\in\Dom$} and \mbox{$\pr\in\sigma(\stu)$}.
                                          From \mbox{$\pr'\in\sigma'(\sts')$}, we get \mbox{$\pr'(\pr,\de)\in\sigma(\sts')$} for every \mbox{$\de\in\Dom$} and \mbox{$\pr\in\sigma(\stu)$}.
                                          Thus \mbox{$\tuple{\interpretpgpp{a}(\pr,\de),\interpretpgpp{a}(\pr,\de)}\in\interpret{R}{\gamma(\pr'(\pr,\de))}$} for all \mbox{$\de\in\Dom$} and \mbox{$\pr\in\sigma(\stu)$} whence
                                          \mbox{$\tuple{\interpretpgpp{a},\interpretpgpp{a}}\in\interpretpgpp{R}$}.
                                \end{itemize}
                          \item \mbox{$\dlstruct'\models\zeta$} for all \mbox{$\zeta\in\S$}:
                                Let \mbox{$\sts'\preceq\sts''\in\S$} and \mbox{$\pr'\in\sigma'(\sts')$}.
                                Then by definition of $\sigma'$, we get \mbox{$\mathop{\bigcup_{\pr\in\sigma(\stu),}}_{\de\in\Dom}\set{\pr'(\pr,\de)}\subseteq\sigma(\sts')$}.
                                By \mbox{$\sigma(\sts')\subseteq\sigma(\sts'')$} it follows that \mbox{$\pr'\in\sigma'(\sts'')$}.
                      \end{itemize}
                      (b)~
                      Let \mbox{$\dlstruct'\models\T'$} and consider \mbox{$\pr'\in\sigma'(\sts)$} and \mbox{$\de'\in\interpretpgpp{(\standbu C)}$}.
                      Since \mbox{$\prvo,\prvi\in\sigma'(\stu)$}, we thus get \mbox{$\de'\in\interpret{C}{\gamma'(\prvo)} \cap \interpret{C}{\gamma'(\prvi)}$}.
                      By \mbox{$\dlstruct'\models\standbu[C\dlsub A]$}, we obtain that \mbox{$\de'\in\interpret{A}{\gamma'(\prvo)} \cap \interpret{A}{\gamma'(\prvi)}$};
                      that is, \mbox{$\de'\in\interpretpgpp{(\standdvo A\dland \standdvi A)}$}.
                      Since also \mbox{$\dlstruct'\models\standbs[\standdvo A \dland \standdvi A \dlsub D]$}, it follows that \mbox{$\de'\in\interpretpgpp{D}$}.
              \end{description}
    \end{enumerate}
    %This concludes the proof of \Cref{lem:normalisation}.
\end{longproof}

%% file: sections/tableau-algorithm.tex
We present a \textsc{PTime} tableau decision algo\-rithm for $\SEL$.\!
Complexity-optimal tableau algorithms have been proposed for description logics with modal operators applied to concepts and axioms such as $\mathbf{K}_{\mathcal{ALC}}$ \cite{LutzSWZ02}, which is known to be in $\NExpTime$.
Our case cannot be treated in the same way, as we need to take greater care to show tractability in the end.
\citeauthor{LutzSWZ02}~[\citeyear{LutzSWZ02}] build a ``quasi-model'' from a tree of ``quasi-worlds'', which is not as easily applicable in our case,\iflong\footnote{Note, e.g.\ that rigidity of a concept~$C$ can be globally enforced via \mbox{$\standball[C\dlsub\standball C]$}.}\fi\xspace so we follow a dual approach:
we will build a \emph{quasi-model} from a completion graph of \emph{(quasi) domain elements}, where each of the latter is associated to a constraint system with assembled information regarding one individual's specifics in each precisification.
%We follow a dual approach: rather than building a `quasi-model' from a tree of `quasi-worlds', we will build a \emph{quasi-model} from a completion graph of \emph{(quasi) domain elements}, where each of the latter is associated to a constraint system with assembled information regarding one individual's specifics in each precisification.
%
%Let us define formally what we mean by a constraint system for a given standpoint $\EL$-knowledge base $\kb$.
We begin with some definitions. Given a $\SEL$ knowledge base $\kb$, denote by
\begin{itemize}
    \item $\SC$ the elements of $\Stands$ occurring in $\kb$ together with $*$,
    \item $\objC$ the set of all individual names occurring in $\kb$,
    \item $\BCC$ (\define{basic concepts}) the %elements of $\Concepts$ 
          concept names
          used in $\kb$, plus \!$\top$\!,
    \item $\CC$  the set of concept terms used in $\kb$ (with $\BCC\subseteq\CC$),
    \item $\forC$ the set of \define{subformulas} of $\kb$, consisting of all axioms of $\kb$ with and without their outer standpoint modality.
\end{itemize}
A \emph{constraint} for $\kb$ is of the form $(x\col C)$, $(x\col a)$, $(x\col \phi)$, or $(x\col \st)$,\footnote{For better legibility, we will sometimes omit the parentheses.}
where $x$ is a variable, $C\in \CC$ a concept, $a\in \objC$ an individual, $\phi\in\forC$ a formula, and $\st \in \SC$ a standpoint name.
\emph{Constraint systems} are finite sets of constraints.

\begin{definition}[(Initial) Constraint System for $\kb$]\label{def:constraint-system} %\label{def:ini-constraint-system}
    The \emph{initial constraint system} for $\kb$, called $\initialcs$, is the set
    \begin{linenomath*}
        $$ \{x_{\st}\col *,\ x_{\st}\col \!\top,\ x_{\st}\col \phi,\  x_{\st}\col \st \mid \phi\in\kb, \st\in\SC \}$$
    \end{linenomath*}

    % A \emph{constraint system} for $\kb$ is a finite set $S$ of constraints for $\kb$ that includes
    %     % \begin{enumerate}
    %     %            \item 
    %     \mbox{$\initialcs$}, and contains
    %     %		\item 
    %      $x\col *$ and $\ x\col\! \top$ for each $x$ in $S$.
    %     %	\end{enumerate}
    %  %a variable \emph{fresh} if it does not (yet) occur in $S$, and 
    % For a variable $x$, let \mbox{$\stlabel{S}{x}=\set{\st\guard(x\col \st)\in S}$} be the \emph{standpoint signature} of $x$ in~$S$.

    A \emph{constraint system} for $\kb$ is a finite set $S$ of constraints for $\kb$ such that
    % \begin{enumerate}
    %            \item 
    \mbox{$\initialcs\subseteq S$} and
    %		\item 
    \mbox{$\{x\col *,\ x\col\! \top\}\subseteq S$} for each $x$ in $S$.
    %	\end{enumerate}
    %a variable \emph{fresh} if it does not (yet) occur in $S$, and 
    For a variable $x$, let \mbox{$\stlabel{S}{x}=\set{\st\guard(x\col \st)\in S}$} be the \emph{standpoint signature} of $x$ in~$S$.
\end{definition}

%Hello! :-) I see you are still busy with that part. Currently we are 2 lines overlength. Let me know when you are ready, then I can condense.
%Sure : )
%Hi Sebastian :)

Intuitively, each constraint system $S$ produced by the algorithm corresponds to a domain element \mbox{$\qelem\in\Delta$} and each variable $x$ in $S$ corresponds to some precisification $\pi$.
Moreover, each constraint $x\col X$ in $S$ encodes information of $\qelem$ in $\pi$.
Namely, $X$ may be an axiom that holds in $\pi$, a standpoint that contains $\pi$, or a concept expression of which $\qelem$ is an instance in $\pi$.
Initialising one variable per standpoint in the initial constraint system guarantees non-empty standpoints.

%Intuitively, each constraint system $S$ encodes information relevant to a domain element \mbox{$\qelem\in\Delta$}, and each variable $x$ specifies the constraints of $\qelem$ in some precisification $\pi$.
%Namely, $x$ is paired with axioms that hold in $\pi$, standpoints to which $\pi$ belongs to, and the concept expressions that $\qelem$ belongs to in $\pi$.
%The initialisation of a variable per standpoint in the initial constraint system guarantees non-empty standpoints.

A constraint system is \define{complete} iff it satisfies every local completion rule from \Cref{fig:rules}.
Local completion rules operate on constraint systems, while global rules involve more than one constraint system and operate on completion graphs.

%\begin{definition}[Element label]\label{def:element-label}
%    An element label $L$ is a set of triples $(C,s,x)\in L$ where $C\in\BCC$ is a concept, $s\subseteq\SC$ is a set of standpoints and $x$ is a variable.  
%\end{definition}

%\begin{definition}[Quasi-role]\label{def:quasi-role}
%	A \emph{quasi-role} is a tuple $\langle \qelem, v, \qelem', v', R\rangle$ where $\qelem$ and $\qelem'$ are elements, $v$ and $v'$ are variables and $R$ is a role name in $\kb$.
%\end{definition}

\begin{definition}[Completion Graph]\label{def:completion-set}
    An \emph{element label} is a set $L$ of triples of the form $(C,s,x)$, where $C\in\BCC$ is a con\-cept, $s\subseteq\SC$ is a set of standpoints, and $x$ is a variable.\\
    A \emph{quasi-role} for a set $\Delta$ is a tuple $\langle \qelem, v, \qelem', v', R\rangle$ where $v$ and $v'$ are variables, $\qelem,\qelem'\in \Delta$,  and $R$ is a role name in $\kb$.\\
    A \emph{completion graph} for $\kb$ is a tuple \mbox{$\mathbf{CG}=\tuple{\Delta,\consisf,\elabelf,\rolesf}$},
    %%% SOLVED:
    %\hsn{There are two notation clashes here, as $\Dom$ is also the domain of an interpretation, and $\S$ is also the SBox in $\K$.} 
    %%%
    with $\Delta$ a non-empty set of elements; $\consisf$ a map from $\Delta$ into constraint systems; $\elabelf$ a map from $\Delta$ into element labels; and $\rolesf$ a set of quasi-roles such that
    \begin{itemize}
        \item for all $\langle \qelem, v, \qelem'\!, v'\!, R\rangle\!\in\!\rolesf$,
              $(v\col \st)\in \consis{\qelem}$ iff $(v'\col \st)\in \consis{\qelem'}$;
        \item if $(C,s,x)\!\in\!\elabel{\qelem}$, then $\{x\col C\}\!\cup\!\{x\col \st\mid \st\in s\}\subseteq \consis{\qelem}$.
    \end{itemize}
\end{definition}

\begin{figure*}
    \small
    %\begin{mdframed}
    \!\!\begin{minipage}{0.38\textwidth}
        %		\textbf{Tableau completion rules} (Assume $C \in\CC\cup\{\bot\}$)
        \begin{itemize}
            \item[]\hspace{-1.4ex}\textbf{Local labelling (LL) rule}:
            \begin{enumerate}[leftmargin=0.4cm, itemsep=0.1em, labelsep=0.6ex]
                \item[$\mathbf{R}_{\preceq}$] If $\{x \col \st,\ x' \col \st\!\preceq\!\sp\} \subseteq S$ but $(x\col\sp) \notin S$,\\ then set $S\eqdef S\cup \{x \col \sp\}$.
            \end{enumerate}
            \item[]\hspace{-1.4ex}\textbf{Local content (LC) rules}:
            \begin{enumerate}[leftmargin=0.4cm, itemsep=0.1em, labelsep=0.6ex]
                \item[$\mathbf{R}_\sqcap$] If \mbox{$\{x \col C,\ x \col D\}\subseteq S$}, \mbox{$(x \col C\sqcap D)\notin S$} and \mbox{$C\sqcap D\in\CC$}, then set \mbox{$S\eqdef S\cup \{x \col C \sqcap D\}$}.
                \item[$\mathbf{R}_\sqsubseteq$] If \mbox{$\{x \col C,\  x \col C\sqsubseteq D\}\subseteq S$} but \mbox{$(x \col D)\notin S$}, then set \mbox{$S\eqdef S\cup \{x \col D\}$}.
                \item[$\mathbf{R}_{\Box}$] If \mbox{$\{x \col \!\standbs\Phi,\  x' \col \st\}\subseteq S$} but \mbox{$(x' \col \Phi)\notin S$},\\ then set \mbox{$S\eqdef S\cup \{x' \col \Phi\}$}.
                \item[$\mathbf{R}_{g}$] If \mbox{$(x \col \mathbf{G}) \in S$} but \mbox{$(x' \col \mathbf{G}) \notin S$},\\ then set \mbox{$S\eqdef S\cup \{x' \col \mathbf{G}\}$}.
                \item[$\mathbf{R}_{a}$] If \mbox{$\{x \col a,\ x \col C(a)\}\subseteq S$} but \mbox{$(x \col C)\notin S$},\\ then set \mbox{$S\eqdef S\cup \{x \col C\}$}.
                \item[$\mathbf{R}_{\Diamond}$] If \mbox{$(x \col \standds C)\in S$} and \mbox{$\{x' \col \st,\ x' \col C\}\nsubseteq S$} for all $x'$ in $S$, then create a fresh variable $x'$ and set \mbox{$S\eqdef S\cup \{x'\col C,\ x'\col \st,\ x'\col *,\ x'\col \top\}$}.
            \end{enumerate}
        \end{itemize}
    \end{minipage} \ \
    \begin{minipage}{0.61\textwidth}
        %\ \\
        \begin{itemize}
            \item[]\hspace{-1.4ex}\textbf{Global non-generating (GN) rules}:
            \begin{enumerate}[leftmargin=0.4cm, itemsep=0.1em, labelsep=0.5ex]
                \item[$\mathbf{R}_{\downarrow}$] \mbox{If %$(x \col \exists R.C)\in \consis{\qelem}$ and 
                        \mbox{$(x\col C){\,\in\,} \consis{\qelem}$}, %for some \mbox{$\qelem'\in\Delta$} 
                        \mbox{$\langle \qelem'\!, x'\!, \qelem, x, R\rangle {\,\in\,} \rolesf$}, and \mbox{$\exists R.C {\,\in\,} \CC$}, but \mbox{$(x'\! \col \exists R.C) {\,\notin\,} \consis{\qelem'}$}}, then set \mbox{$\consis{\qelem'}\eqdef \consis{\qelem'}\cup \{x' \col \exists R.C\}$}.
                \item[$\mathbf{R}_{r}$] If \mbox{$\{x \col a,\ x \col R(a,b)\}\subseteq \consis{\qelem}$} and \mbox{$(x'\col b)\in\consis{\qelem'}$}, but \mbox{$\langle \qelem, x, \qelem', x, R\rangle\notin\rolesf$},
                    then set \mbox{$\consis{\qelem'}\eqdef \consis{\qelem'}{\,\cup\,} \{x \col \top\}{\,\cup\,} \{x\col \st \mid \st\in\stlabel{\qelem}{x} \}$}\\ $\left.\right.\quad\quad\ \ $ and \mbox{$\rolesf\eqdef \rolesf{\,\cup\,}  \{\langle \qelem, x, \qelem'\!, x, R\rangle\}$}.

                \item[$\mathbf{R}_{r'}\!$] If \mbox{$\{x \col b,\ x \col R(a,b)\}\subseteq \consis{\qelem}$} and \mbox{$(x'\col a)\in\consis{\qelem'}$}, but \mbox{$\langle \qelem', x, \qelem, x, R\rangle\notin\rolesf$}, then set \mbox{$\consis{\qelem'}\eqdef \consis{\qelem'}\cup \{x \col \top\}\cup\{x\col \st \mid \st\in\stlabel{\qelem}{x} \}$}\\$\left.\right.\quad\quad\ \ $ and \mbox{$\rolesf\eqdef \rolesf\cup \{\langle \qelem', x, \qelem, x, R\rangle\}$}.

                \item[$\mathbf{R}_{\exists'}\!$] If %there are some $\qelem,\qelem' \in\Delta$ with 
                    \mbox{$(x \col \exists R.C)\in \consis{\qelem}$}, \mbox{$(C,\stlabel{\qelem}{x},x')\in\elabel{\qelem'}$} with \mbox{$\qelem\neq\qelem'$} or \mbox{$x=x'$}, but \mbox{$\langle \qelem, x, \qelem''\!, x''\!, R\rangle\notin\rolesf$} whenever \mbox{$(C,\stlabel{\qelem}{x},x'')\in\elabel{\qelem''}$} and \mbox{$\qelem{\,\neq\,}\qelem''$} or \mbox{$x{\,=\,}x''$}\!, then set \mbox{$\rolesf\eqdef \rolesf\cup \{\langle \qelem, x, \qelem', x', R\rangle\}$}.
            \end{enumerate}
            \item[]\hspace{-1.4ex}\textbf{Global generating (GG) rule}:
            \begin{enumerate}[leftmargin=0.4cm, itemsep=0.1em, labelsep=0.6ex]
                \item[$\mathbf{R}_{\exists}$] If% there is some $\qelem \in\Delta$ with 
                    \mbox{$(x \col \exists R.C)\in \consis{\qelem}$}, %but no \mbox{$\qelem'$} has $(C,\stlabel{\qelem}{x},x')\in\elabel{\qelem'}$, $\langle \qelem, x, \qelem', x', R\rangle\in\rolesf$ and either $\qelem'\neq\qelem$ or $x'=x$, 
                    %%% NEW
                    but\\ \mbox{$\langle \qelem, x, \qelem''\!, x''\!, R\rangle\notin\rolesf$} whenever \mbox{$(C,\stlabel{\qelem}{x},x'')\in\elabel{\qelem''}$} and \mbox{$\qelem{\,\neq\,}\qelem''$} or \mbox{$x{\,=\,}x''$}\!,\\
                    %%% NEW
                    then create $\qelem'$ and a fresh variable $x'$, and then set \mbox{$\elabel{\qelem'}\eqdef \{(C,\stlabel{\qelem}{x},x')\}$}, \mbox{$\consis{\qelem'}{\eqdef}  \initialcs{\,\cup\,}\{x'\!\col C,\, x'\!\col \!\top\}{\,\cup\,}\{x'\!\col \st \mid \st{\,\in\,}\stlabel{\qelem}{x} \}$}, %and 
                    \mbox{$\rolesf{\eqdef} \rolesf{\,\cup\,} \{\langle \qelem, x, \qelem'\!\!, x'\!\!, R\rangle\}$}.
            \end{enumerate}
        \end{itemize}
    \end{minipage}
    %\end{mdframed}
    \vspace{-0.5em}
    \caption{The tableau completion rules. $\mathbf{G}$ can be of the form $a$, $(\st\preceq\sp)$, or $\standbs \phi$. $\Phi$ may denote any element of {\small $\forC\cup\CC$}.%, all of which share the property of applying globally to all precisifications in a model. 
    }\label{fig:rules}
    \vspace{-0.5em}
\end{figure*}

\pagebreak

For convenience of presentation, we use the shortcut $\stlabel{\qelem}{v}$ for $\stlabel{\consis{\qelem}}{v}$ and for any \mbox{$\mathbf{CG}=\tuple{\Delta,\consisf,\elabelf,\rolesf}$}, we will refer to all \mbox{$\qelem\in\Delta$} simply as elements of $\mathbf{CG}$.

$\mathbf{CG}$ is said to be \define{locally complete} iff for every element $\qelem$ in $\mathbf{CG}$, $\consis{\qelem}$ is complete, and we call $\mathbf{CG}$ \define{globally complete} iff it is {locally complete} and no global completion rule (see \Cref{fig:rules}) is applicable to $\mathbf{CG}$ as a whole.%
% Guess everybody has read this by now... ;-)
%\sru{Don't change spacing in the figure any more, it's deliberate and fine-tuned.}

Intuitively, the next definition poses some global requirements for $\mathbf{CG}$ to warrant its eligibility as a model-substitute.
\begin{definition}[Coherence]\label{def:coherence}
    Let \mbox{$\mathbf{CG}=\tuple{\Delta,\consisf,\elabelf,\rolesf}$} be a completion graph for $\kb$.
    $\mathbf{CG}$ is called \define{coherent} iff%\hsn{For the sake of uniformity, I'll adapt the definitional if to iff.}
    \begin{itemize}
        \item for each \mbox{$a\in\objC$} there is a unique element \mbox{$\qelem_a\in\Delta$} such that \mbox{$(v\col a)\in\consisu{\qelem_a}$} for all variables $v$ in $\consisu{\qelem_a}$,
        \item for each %\hsn{For each $\qelem,\qelem'\in\Delta$ and for each $v\in\consis{\qelem}$ \ldots?}% \sru{Solved.}%
              $\qelem,\qelem'\in\Delta$ and each variable $v$ contained in $\consis{\qelem}$, \mbox{$\consis{\qelem'}$ contains} some $v'$ such that \mbox{$\stlabel{\qelem}{v}=\stlabel{\qelem'}{v'}$}, and
        \item if $(v\col\phi){\in}\consis{\qelem}$ and $\stlabel{\qelem}{v}{=}\stlabel{\qelem'}{v'}$, then $(v'\!\col\phi){\in}\consis{\qelem'}$.
              %\item \textcolor{orange}{for each $\tuple{\qelem, v, \qelem', v', R}\in\rolesf$ with $(v'\col C)\in\consisu{\qelem'}$, we have $(v \col \exists R.C) \in\consisu{\qelem}$ in case $\exists R.C \in\CC$.}
    \end{itemize}
    %\textcolor{blue}{S.R.: The 4th bullet point is redundant, IMHO, because it seems to coincide with being closed under $\mathbf{R}_{\downarrow}$.}\\
    %\textcolor{blue}{S.R.: Is ``homogenous'' a good term? Maybe "coherent" is better?}
\end{definition}

As usual in tableaux, inconsistencies emerge as clashes.

\begin{definition}[Clash]
    A \define{clash} is a constraint of the form  $(x\col \bot)$.
    %    A constraint system $S$ is said to contain a \emph{clash} if $(x\col \bot) \in S$ for some variable $x$. 
    A completion graph $\mathbf{CG}$ is said to \define{contain a clash} iff $\consis{\qelem}$ does for some $\qelem$ in $\mathbf{CG}$.
    Constraint systems or completion graphs not containing clashes are called \define{clash-free}.
\end{definition}

\subsubsection{The Algorithm}

To decide whether a given $\SEL$ knowledge base $\kb$ in normal form is satisfiable, we form the initial completion graph $\mathbf{CG}_{I}$ with \mbox{$\rolesf=\emptyset$} and $\Delta$ consisting of an element $\qelem_{\top}$ with \mbox{$\elabel{\qelem_{\top}}=\emptyset$} and \mbox{$\consis{\qelem_{\top}}=\initialcs$},
and for every \mbox{$a\in\objC$} an element $\qelem_a$ with \mbox{$\elabel{\qelem_{a}}=\emptyset$} and \mbox{$\consis{\qelem_{a}}= \initialcs \cup \{(x_{\st}\col a)\mid (x_{\st}\col \top)\in \initialcs\}$}.

After that, we repeatedly apply the local and global completion rules from \Cref{fig:rules}, %in the order specified in the pseudo-code of Fig.~\ref{fig:pseudocode}
where LL rules have the highest priority, followed by LC, GN, and GG rules, in that order.
After each rule application, we check if $\mathbf{CG}$ contains a clash and terminate with answer ``unsatisfiable'' should this be the case.
If we arrive at a clash-free $\mathbf{CG}$ with no more rules applicable, the algorithm terminates and returns ``satisfiable''.

\subsubsection{Quasi-Models and Quasi-Satisfiability}
This section sketches how special structures, called
(dual) quasi-models can serve as proxies for proper $\SEL$ models.
%Our construction is based on a set of (quasi-)elements that characterise the domain and its interpretation in the different precisifications. 

%\begin{definition}[Run]\label{def:run}
%	Let $\mathbf{CG}=\tuple{\Delta,\consisf,\elabelf,\rolesf}$ be a completion graph. A \emph{run} $r$ in $\mathbf{CG}$ is a map from each element $\qelem\in\Delta$ to a variable of $\consis{\qelem}$, such that
%	\begin{enumerate}[leftmargin=0.5cm, label={\rm\small C\arabic*}, ref={(C\arabic*)}]
%		\item if $(r(\qelem)\col \st)\in\consisu{\qelem}$, then $(r(\qelem')\col \st) \in \consisu{\qelem'}$  for all $\qelem,\qelem'\in\Delta$ and $\st\in\SC$,\label{condition:run-standpoint}
%            \item if $\tuple{\qelem,\!v,\!\qelem',\!v',\!R}\in\rolesf$ and $r(\qelem)=v$, then $r(\qelem')=v'$, and \label{condition:run-role}
%            \item if $(r(\qelem)\col \exists R.C) \in\consisu{\qelem}$, there is a $\qelem'\in\Delta$ with $\tuple{\qelem, r(\qelem), \qelem', r(\qelem'), R}\in\rolesf$ and $(r(\qelem')\col C) \in\consisu{\qelem'}$. %\label{condition:run-existential}
% \end{enumerate}
%\end{definition}

\begin{definition}[Run, Quasi-model]\label{def:run}%\label{def:quasi-model}
    %\sru{Don't change spacing of the figure any more, it's deliberate and fine-tuned.}
    %\hsn{Allowing $\Phi\in\objC$ (caption of \Cref{fig:rules}) would allow for $\Phi=a$ and thus some $x:\standbs a$. Is this intended?}%
    %\sru{Solved.}%
    Let \mbox{$\mathbf{CG}{\,=\,}\tuple{\Delta,\consisf,\elabelf,\!\rolesf}$} be a completion graph.
    A \define{run} $r$ in $\mathbf{CG}$ is a function mapping each element $\qelem\in\Delta$ to a variable of $\consis{\qelem}$, such that
    \begin{enumerate}[leftmargin=0.5cm, label={\rm\small C\arabic*}, ref={(C\arabic*)}]
        \item if \mbox{$(r(\qelem)\col \st)\in\consisu{\qelem}$}, then \mbox{$(r(\qelem')\col \st) \in \consisu{\qelem'}$}\\
              for all \mbox{$\qelem,\qelem'\in\Delta$} and \mbox{$\st\in\SC$},\label{condition:run-standpoint}
        \item if \mbox{$\tuple{\qelem,v,\qelem'\!,v'\!,R}\in\rolesf$} and \mbox{$r(\qelem)=v$}, then \mbox{$r(\qelem')=v'\!$}, and \label{condition:run-role}
        \item if \mbox{$(r(\qelem)\col \exists R.C) \in\consisu{\qelem}$}, there exists some \mbox{$\qelem'\in\Delta$} with \mbox{$\tuple{\qelem, r(\qelem), \qelem'\!, r(\qelem'), R}\in\rolesf$} and \mbox{$(r(\qelem')\col C) \in\consisu{\qelem'}$}. \label{condition:run-existential}
    \end{enumerate}
    A \define{quasi-model} of $\kb$ is a tuple \mbox{$\mathcal{Q}=\tuple{\Delta,\consisfu,\elabelf,\rolesf,\runsf}$} where $\tuple{\Delta,\consisf,\elabelf,\rolesf}$ is a globally complete, coherent and clash-free completion graph for $\kb$, and $\runsf$ a set of runs in $\tuple{\Delta,\consisf,\elabelf,\rolesf}$ such that
    for every \mbox{$\qelem\in\Delta$} and variable $v$ in $\consis{\qelem}$, there is a run $r$ in $\runsf$ such that \mbox{$r(\qelem)=v$}.
    $\kb$ is called \define{quasi-satisfiable} iff $\kb$ has a quasi-model.
\end{definition}

In a nutshell, runs serve the purpose of lining up ``compatible'' variables, one from each individual constraint system, in such a way that precisifications can be constructed. With these notions in place, we can establish the desired result.

\begin{theorem}\label{theorem:satisfiable-iff-quasisatisfiable}
    A $\SEL$ knowledge base $\kb$ is satisfiable iff it is quasi-satisfiable.
\end{theorem}

\begin{proof} \emph{(sketch)} We prove the correspondence by showing that every quasi-model gives rise to a model and vice versa.\\
    ($\Leftarrow$) Given a quasi-model $\mathcal{Q}=\tuple{\Delta,\consisfu,\elabelf,\rolesf, \runsf}$ of $\kb$, we obtain a model $\dlstruct = \tuple{\Dom, \Precs, \sigma, \gamma}$ by letting $\Pi = \runsf$%, i.e. the set of precisifications of $\dlstruct$ coincides with the set of runs in $\mathcal{Q}$
    , $\sigma(\st)=\{r \mid  (r(\qelem)\col \st)\in\consisu{\qelem}\}$, $C^{\gamma(r)} = \{ \qelem \mid (r(\qelem)\col C)\in\consisu{\qelem}\}$ for $C\in\BCC$, $R^{\gamma(r)} = \{(\qelem,\qelem')\mid \langle \qelem,r(\qelem),\qelem',r(\qelem'),R\rangle\in\rolesf\}$ and $a^{\gamma(r)} = \qelem_a$ for all $a\in\objC$.\\
    ($\Rightarrow$) Given a model $\dlstruct = \tuple{\Dom, \Precs, \sigma, \gamma}$ of $\kb$, we obtain a quasi-model $\mathcal{Q}=\tuple{\Delta,\consisfu,\elabelf,\rolesf,\runsf}$ as follows:
    Let $P \eqdef  \Precs \cup \SC$ and for $p \in P$, let $\bar{p}$ denote some arbitrary but fixed $\pi \in \sigma(p)$ if $p \in \SC$, and otherwise $\bar{p}=p$. Finally, let
    \vspace{-2ex}

    {\small\vspace{-1ex}
        \begin{align*} \consisu{\qelem} \eqdef &
               \{v_{p}\col C \mid \qelem \in C^{\gamma(\bar{p})}, p \in P\}\cup \{v_{p}\col a \mid \qelem=a^{\gamma(\bar{p})}, p \in P\}                                                                      \\
                                       & \cup \{v_{p}\col \phi \mid \dlstruct, \bar{p}\models \phi,\ \phi\in\forC\} \cup \{v_{p}\col \st \mid \bar{p}\in\sigma(\st)\}                                         \\
               \elabel{\qelem} \eqdef  & \{(C,\!\{\st \mid \bar{p}{\in}\sigma(\st)\}\!,\!v_{p}) \mid \exists R.C {\in}\CC,(\qelem'\!\!,\qelem){\in} R^{\gamma(\bar{p})}\!\!,\qelem{\in} C^{\gamma(\bar{p})}\} \\
               \rolesf \eqdef          & \{\tuple{\qelem, v_{p}, \qelem', v_{p}, R} \mid (\qelem,\qelem')\in R^{\gamma(\bar{p})} \}                                                                           \\
               \runsf \eqdef           & \{ \{ \qelem \mapsto v_p \mid \qelem \in \Delta \}  \mid p \in P \}.                                                                                                 \\[-6.5ex]
        \end{align*}
    }
\end{proof}
\vspace{1ex}

\subsubsection{Polytime Termination and Correctness}

Next, we give an overview of our argument why our algorithm runs in polynomial time with respect to $\size{\kb}$, the size of its input $\kb$.
We observe that the number $|\Delta|$ of domain elements of any completion graph $\mathbf{CG}$ constructed by our algorithm is bounded by $3\size{\kb}^2$ ($\dag$).
We also find that the number of variables used in any single $\consis{\qelem}$ is bounded by $2\size{\kb}^2$ and the number of constraints in $\consis{\qelem}$ by $2\size{\kb}^3$ ($\ddag$).
Now, the number of applications of $\mathbf{R}_{\exists}$ is bounded by the number of elements in each completion graph, i.e. at most $3\size{\kb}^2$ in view of ($\dag$). Since the rules $\mathbf{R}_{\preceq}$, $\mathbf{R}_{\sqcap}$, $\mathbf{R}_{\sqsubseteq}$, $\mathbf{R}_{\Box}$, $\mathbf{R}_{\Diamond}$, $\mathbf{R}_{g}$,
$\mathbf{R}_{a}$, $\mathbf{R}_{r}$, $\mathbf{R}_{r'}$ and $\mathbf{R}_{\downarrow}$ produce one or more new constraints in an element, the number of applications of such rules per element is bounded by $2\size{\kb}^3$ due to ($\ddag$).
$\mathbf{R}_{\exists'}$ can add, for each $\qelem$ with $(C,s,x)\in\elabel{\qelem}$, at most one quasi-role from every variable in every element, thus we have at most $6\size{\kb}^4$ rule applications.
The total number of rule applications is bounded by the rule applications per element multiplied by the bound on elements, together with the bound on $\mathbf{R}_{\exists}$, which gives us%
\begin{linenomath*}%    
    $$ (6\!\size{\kb}^4)(3\!\size{\kb}^2) + (2\!\size{\kb}^3)(3\!\size{\kb}^2) + 3\!\size{\kb}^2 \leq (27\!\size{\kb}^6).$$
\end{linenomath*}

\begin{theorem}\label{theorem:termination}
    The completion algorithm terminates after at most $c\size{\kb}^6$ steps, where $c$ is a constant.
\end{theorem}

As every single rule application can be clearly executed in polynomial time with respect to $\kb$, we can conclude that our algorithm runs in polynomial time.

We are now ready to establish correctness of our decision algorithm, by showing its soundness and completeness.
For both directions, \Cref{theorem:satisfiable-iff-quasisatisfiable} will come in handy.
As usual, the soundness part of our argument is the easier one.

\begin{theorem}[Soundness]\label{theorem:soundness}
    If there is a globally complete, coherent and clash-free completion graph $\mathbf{CG}$ for a knowledge base $\kb$, then $\kb$ is satisfiable.
\end{theorem}

\begin{proof} \emph{(sketch)}
    Given $\mathbf{CG}=\tuple{\Delta,\consisf,\elabelf,\rolesf}$, let $\runsf$ consist of all runs on $\mathbf{CG}$.
    Then we can show that $\mathcal{Q}=\tuple{\Delta,\consisf,\elabelf,\rolesf,\runsf}$ constitutes a quasi-model for $\kb$, so we can conclude by Theorem \ref{theorem:satisfiable-iff-quasisatisfiable} that $\kb$ is satisfiable.
\end{proof}

Proving completeness requires significantly more work.
We make use of a notion that, intuitively, formalizes the idea that a completion graph $\mathbf{CG}$ under development is ``in sync'' with a quasi-model $\mathcal{Q}$ of the same knowledge base,
where $\mathcal{Q}$ can be conceived as a model-theoretic ``upper bound'' of $\mathbf{CG}$.

\begin{definition}[$\mathcal{Q}$-compatibility]
    Let $\kb$ be a $\SEL$ knowledge base and \mbox{$\mathcal{Q}=\tuple{\Delta^q,\consisfq,\elabelf^q,\rolesf^q,\runsf^q}$} be a quasimodel for $\kb$.
    A completion graph \mbox{$\mathbf{CG}=\tuple{\Delta,\consisf,\elabelf,\rolesf}$} for $\kb$ is called \define{$\mathcal{Q}$-compatible} iff there is a left-total relation \mbox{$\mu \subseteq \Delta\times\Delta^q$} where
    \begin{itemize}
        %    \item  $(\qelemg,\qelem)\in\crel$ if both (a) $\elabel{\qelemg}\subseteq\elabelq{\qelem}$ and (b) $(v\col a)\in\consisq{\qelem}$ if $(x\col a)\in\consis{\qelemg}$,
        %    \item  $(\qelemg,\qelem)\in\crel$ if $\{a \mid (x\col a){\in}\consis{\qelemg}\} \subseteq \{ a \mid (v\col a){\in}\consisq{\qelem}\}$ and $\elabel{\qelemg}{\subseteq}\elabelq{\qelem}$,
        \item for all \mbox{$\qelemg \in \Delta$} and \mbox{$\qelem \in \Delta^q$}, if both \mbox{$\elabel{\qelemg}\subseteq\elabelq{\qelem}$} and \mbox{$\{a\mid(x\col a){\,\in\,}\consis{\qelemg}\} {\,\subseteq\,} \{ a \mid (v\col a){\,\in\,}\consisq{\qelem}\}$}, then \mbox{$(\qelemg,\!\qelem){\,\in\,}\crel$},
        \item for each \mbox{$(\qelemg,\qelem)\in\crel$} there is a surjective function $\cfuncf{\qelemg,\qelem}$ from the variables in $\consisq{\qelem}$ to the variables in $\consis{\qelemg}$ such that
              \begin{itemize}
                  \item $(\cfunc{\qelemg,\qelem}{v}\col \st)\in\consis{\qelemg}$ implies $(v\col \st)\in\consisq{\qelem}$,
                  \item $(\cfunc{\qelemg,\qelem}{v}\col \Phi)\in\consis{\qelemg}$ implies $(v\col \Phi)\in\consisq{\qelem}$,
                  \item if $\tuple{\qelemg, x, \qelemg', x', R}\in\rolesf$ then $\tuple{\qelem,\! y,\! \qelem',\! y',\! R}\!\in\!\rolesf^q$ for some $(\qelemg,\qelem), (\qelemg'\!,\qelem')\in\crel$ with $\cfunc{\qelemg,\qelem}{y}{=}x$ and $\cfunc{\qelemg',\qelem'}{y'}{=}x'$\!.
              \end{itemize}
    \end{itemize}
\end{definition}
\pagebreak
With this definition, we can establish two important insights:
\begin{itemize}
    \item The tableau algorithm's initial completion graph $\mathbf{CG}_I$ is $\mathcal{Q}$-compatible for any quasimodel $\mathcal{Q}$ of $\kb$.
    \item Applications of tableau rules preserve $\mathcal{Q}$-compatiblility.
\end{itemize}
This entails that the completion graph maintained in the algorithm will be $\mathcal{Q}$-compatible at all times, thus also upon termination.
We exploit this insight to show completeness.

\begin{theorem}[Completeness]\label{theorem:completeness}
    If a $\SEL$ knowledge base $\kb$ is satisfiable,
    the tableau algorithm will construct a globally complete, coherent, and clash-free completion graph for $\kb$.
\end{theorem}
\begin{proof}
    If $\kb$ is satisfiable then by \Cref{theorem:satisfiable-iff-quasisatisfiable}, there is a quasi-model $\mathcal{Q}$ for $\kb$.
    According to \Cref{theorem:termination}, we can obtain a globally complete completion graph $\mathbf{CG}$ after polynomially many applications of the tableau rules, which, as just discussed, is $\mathcal{Q}$-compatible.
    It must thus also be clash-free, because otherwise there were an element $\qelemg$ and variable $x$, with $(x\col\bot)\in\consis{\qelemg}$, and thus there is $(\qelemg,\qelem)\in\crel$ and $\cfuncf{\qelemg,\qelem}$ such that $(\cfunc{\qelemg,\qelem}{x}\col\bot)\in\consisq{\qelem}$, which is a contradiction because $\mathcal{Q}$ is a quasi-model.
    It is not hard to show that $\mathbf{CG}$ is also coherent, whence we can conclude that $\mathbf{CG}$ is a globally complete, coherent, and clash-free completion graph for $\kb$.
\end{proof}

Together with the well-known \textsc{PTime}-hardness of the satisfiability problem in (standpoint-free) $\EL$, we have therefore established \textsc{PTime}-completeness of $\SEL$ and exhibited a worst-case optimal algorithm for it.

%% file: sections/intractable-extensions.tex
While the shown tractability of reasoning in $\SEL$ is good news, one might ask if one could include more modelling features or relax certain side conditions and still preserve tractability.
This section shows that tractability can be easily lost (at least under standard complexity-theoretic assumptions).

\subsubsection{Empty standpoints}
While it may make sense on a philosophical level, one might wonder whether the constraint that $\sigma(\st)$ needs to be nonempty for every \mbox{$\st \in \SC$} has an impact on tractability.
In fact, dropping this constraint, obtaining a logic $\SEL^\emptyset$ with the same syntax but modified semantics, would increase expressivity (standpoint non-emptiness could still be enforced in $\SEL^{\smash\emptyset}$ by asserting \mbox{$\top \sqsubseteq \standds\! \top$} for every \mbox{$\st\in\SC$}).
However, satisfiability in $\SEL^{\smash\emptyset}$ turns out to be \NP-hard, even when disallowing usage of concept and role names entirely.
The key insight that both $\standd{s}\! \top$ and its negation $\standb{s} \bot$ can be expressed as $\SEL^{\smash\emptyset}$ concepts gives rise to the following reduction from 3SAT:
Assume an instance \mbox{$\phi= \bigvee\!\clause_1 \land \ldots \land \bigvee\!\clause_n$} of 3SAT containing $n$ clauses (i.e., disjunctions of literals) $\bigvee\!\clause_j$  over the propositional variables \mbox{$P = \set{ p_1,\ldots, p_k }$}.
We note that $\phi$ is equivalent to \mbox{$(\bigwedge\!\overline{\clause}_1 \to \mathbf{false}) \wedge \ldots \wedge  (\bigwedge \!\overline{\clause}_k \to \mathbf{false})$}, where $\overline{\clause}_j$ is obtained from ${\clause}_j$ by replacing every literal by its negated version.
Let now $\{\st_1,\ldots,\st_k\}$ be a set of standpoint names and, for any literal $\ell$ over $P$, define
\begin{linenomath*}
    $$
        L_\ell = \begin{cases}
            \standd{s_i}\! \top & \text{ if } \ell = p_i,      \\
            \standb{s_i} \bot   & \text{ if } \ell = \neg p_i. \\
        \end{cases}
    $$
\end{linenomath*}
Then, $\phi$ is satisfiable iff the following $\SEL^\emptyset$ knowledge base is:
\begin{linenomath*}
    $$
        \KB_\phi = \big\{  L_{\ell} \sqcap L_{\ell'} \sqcap L_{\ell''} \sqsubseteq \bot  \mid  \{\ell,\ell',\ell''\} =  \overline{\clause}_j, 1 \leq j \leq n \big\}.
    $$
\end{linenomath*}

\subsubsection{Rigid roles}
$\SEL$ allows to globally enforce rigidity of specific concepts through axioms of the shape \mbox{$A \sqsubseteq \standball A$}.
(This is in contrast to e.g.\ $\mathbf{K}_n\mathcal{ALC}$, where rigidity of concepts can only be expressed \emph{relative} to a given formula.)
In a similar manner, rigidity of roles (i.e., the interpretation of certain distinguished roles being the same throughout all precisifications) would represent a desirable modelling feature.
Other modal extensions of DLs have easily been shown to even become undecidable when this feature is permitted, but as $\SEL$ uses a much simplified semantics on the modal dimension, these results do not carry over to $\SEL$.
Yet, we will show that just the presence of \emph{one} distinguished rigid role $\dot{R}$ causes $\SEL$ to become intractable as satisfiability turns \coNP-hard.
To demonstrate this, we reduce 3SAT to KB unsatisfiability.
\begin{shortonly}
    As above, assume an instance $\phi= \bigvee\!\clause_1 \land \ldots \land \bigvee\!\clause_n$ of 3SAT over propositional variables \mbox{$P = \set{ p_1,\ldots, p_k }$}.
    Then $\phi$ is satisfiable iff the following $\SEL$ TBox is unsatisfiable (with all axioms instantiated for $1 \leq i \leq k$):
    \begin{linenomath*}
        $$
            \begin{array}{@{}r@{\ }l@{}r@{\ }l}
                \top    & \sqsubseteq \exists T. L_0                                             & L_k                                                   & \sqsubseteq \standdall\mathit{S}             \\
                L_{i-1} & \sqsubseteq \exists \dot{R}.\standball(L_{i} \sqcap T_{p_i} )          & \exists \dot{R}.(T_{p_i} \sqcap \mathit{S})           & \sqsubseteq (T_{p_i} \sqcap \mathit{S})      \\
                L_{i-1} & \sqsubseteq \exists \dot{R}.\standball (L_{i} \sqcap T_{\neg p_i})     & \exists \dot{R}.(T_{\neg p_i} \sqcap \mathit{S})      & \sqsubseteq (T_{\neg p_i} \sqcap \mathit{S}) \\
                T_\ell  & \sqsubseteq T_{\clause_j} \text{\ \ for all } \ell \in \clause_j \quad & T_{\clause_1} \! \sqcap \ldots \sqcap T_{\clause_n}\! & \sqsubseteq \bot
            \end{array}
        $$
    \end{linenomath*}
\end{shortonly}
\begin{supplementary}

\end{supplementary}

\subsubsection{Nominal Concepts}
\newcommand{\ELO}{\mathcal{E\hspace{-1pt}LO}}
\newcommand{\SELO}{\mathbb{S}_\ELO}
Nominal concepts are a modelling feature widely used in ontology languages.
For an individual $o$, the nominal concept $\{o\}$ refers to the singleton set $\{o^\mathcal{I}\}$.
Let $\ELO$ denote $\EL$ extended by nominal concepts.
Several formalisms subsuming $\ELO$, including OWL~2~EL, are known to allow for tractable reasoning \cite{Baader05ELenvelope,Krotzsch10OWLELreasoning}.
However, in the presence of standpoints, nominals prove to be detrimental for the reasoning complexity: satisfiability of $\SELO$ TBoxes using just one nominal concept $\{o\}$ turns out to be \ExpTime-hard and thus definitely harder than for $\SEL$. This can be shown by a \PTime reduction of satisfiability for Horn-$\mathcal{ALC}$ TBoxes (which is known to be \ExpTime-complete \cite{Krotzsch13HornDLs}) to satisfiability of $\SELO$ TBoxes with just one standpoint (the global one) and one nominal concept $\{o\}$.
To this end, recall that any Horn-$\mathcal{ALC}$ TBox can be normalised in \PTime to consist of only axioms of the following shapes:
\begin{linenomath*}
    $$
        A \sqsubseteq B \quad\! A \sqcap B \sqsubseteq C \quad\! \exists R.A \sqsubseteq B \quad\! A \sqsubseteq \exists R.B \quad\! A \sqsubseteq \forall R.B
    $$
\end{linenomath*}
where $A$ and $B$ can be concept names, $\top$, or $\bot$.
From a normalised Horn-$\mathcal{ALC}$ TBox $\mathcal{T}$, we obtain the target $\SELO$ TBox $\mathcal{T}'$ by
(i) declaring every original concept name as rigid via the axiom $A \sqsubseteq \standball A$ as well as
(ii) replacing every axiom of the shape $A \sqsubseteq \exists R.B$ by the axiom
\begin{linenomath*}
    $$A \sqsubseteq \standdall ((\exists \mathit{Src}.\{o\}) \sqcap (\exists R.(B \sqcap \exists \mathit{Tgt}.\{o\})))$$
\end{linenomath*}
(introducing two fresh role names $\mathit{Src}$ and $\mathit{Tgt}$), and replacing every axiom of the shape $A \sqsubseteq \forall r.B$ by the two axioms
\begin{linenomath*}
    $$A \sqcap \exists R.\top \sqsubseteq (\exists \mathit{Src}.(\{o\} \sqcap \tilde{B})) \quad\text{and}\quad \exists \mathit{Tgt}.\tilde{B} \sqsubseteq B,$$
\end{linenomath*}
introducing a copy $\tilde{A}$ for every original concept name $A$.
%
%$$
%\begin{array}{r@{\ }lrr@{\ }l}
%	A           & \sqsubseteq B           & \Longrightarrow &  A           & \sqsubseteq B  \\
%	A \sqcap B  & \sqsubseteq C           & \Longrightarrow &  A \sqcap B  & \sqsubseteq C \\
%	A           & \sqsubseteq \exists r.B & \Longrightarrow &  A           & \sqsubseteq \standdall ((\exists src.\{o\}) \sqcap (\exists r.(B \sqcap \exists tgt.\{o\})))  \\
%	\exists r.A & \sqsubseteq B           & \Longrightarrow & \exists r.A  & \sqsubseteq B \\
%	A           & \sqsubseteq \forall r.B & \Longrightarrow & A \sqcap \exists r.\top & \sqsubseteq (\exists src.(\{o\} \sqcap \tilde{B})) \qquad   \exists tgt.\tilde{B} \sqsubseteq B\\
%\end{array}
%$$
With this polytime translation, satisfiability of the Horn-$\mathcal{ALC}$ TBox $\mathcal{T}$ and the $\SELO$ TBox $\mathcal{T}'$ coincide.

\begin{supplementary}
    \begin{proposition}
        $\mathcal{T}$ and $\mathcal{T}'$ are equisatisfiable.
    \end{proposition}
    \begin{proof}
        We show equisatisfiability of $\mathcal{T}$ and $\mathcal{T}'$ by providing constructions to obtain a model of $\mathcal{T}$ from a model of $\mathcal{T}'$ and vice versa.
        Let $\mathcal{I}=(\Dom,\cdot^\mathcal{I})$ be a model of $\mathcal{T}$. Then we construct a model
        $\dlstruct= \tuple{\Dom, \Precs, \sigma, \gamma}$ of $\mathcal{T}'$ as follows:
        \begin{itemize}
            \item $\Precs = \Roles \times \Dom \times \Dom$ with $\sigma(*) = \Precs$,
            \item pick one arbitrary but fixed $\delta_o \in \Dom$ and let $o^{\gamma(\pi)}=\delta_o$ for all $\pi \in \Precs$,
            \item $A^{\gamma(\pi)}=A^\mathcal{I}$ for all $\pi \in \Precs$, whereas
                  $\tilde{A}^{\gamma((r,\delta_1,\delta_2))}= \{ \delta_o \mid \delta_2 \in A^\mathcal{I}\}$ for $A \in \Concepts$,
            \item $r^{\gamma((r,\delta_1,\delta_2))}=\{(\delta_1,\delta_2)\} \cap r^\mathcal{I}$, while
                  $src^{\gamma((r,\delta_1,\delta_2))}=\{(\delta_1,\delta_o)\}$ and
                  $tgt^{\gamma((r,\delta_1,\delta_2))}=\{(\delta_2,\delta_o)\}$.
        \end{itemize}

        \medskip

        For the other direction, let $\dlstruct= \tuple{\Dom, \Precs, \sigma, \gamma}$ be a model of $\mathcal{T}'$. Then we construct a model $\mathcal{I}=(\Dom,\cdot^\mathcal{I})$ of $\mathcal{T}$ as follows:
        \begin{itemize}
            \item $A^\mathcal{I} = A^{\gamma(\pi)}$ for any $\pi \in \Precs$ (the choice is irrelevant because all $A$ are rigid)
            \item $r^\mathcal{I} = \bigcup_{\pi \in \Precs}\{(\delta_1,\delta_2) \in r^{\gamma(\pi)} \mid (\delta_1,o^{\gamma(\pi)}) \in  src^{\gamma(\pi)},\ (\delta_2,o^{\gamma(\pi)}) \in tgt^{\gamma(\pi)}\}$
        \end{itemize}
    \end{proof}
\end{supplementary}

%% file: sections/conclusion.tex
In this paper we introduced Standpoint $\EL$, a new, lightweight member of the emerging family of standpoint logics. We described the new modelling and reasoning capabilities it brings to large-scale ontology management
and established a \textsc{PTime} (and thus worst-case optimal) tableau-based decision procedure for standard reasoning tasks. We also demonstrated that certain extensions of $\SEL$, which would be desirable from a expressivity point of view, inevitably come with a loss of tractability (sometimes under the assumption $\textsc{P} \neq \textsc{NP}$).

Yet several modelling features can be accommodated into $\SEL$ without endangering tractability.
For instance, from a usability perspective, it would seem very advantageous if not just single axioms, but whole axiom sets (up to whole knowledge bases) could be preceded by standpoint modalities. By definition, an axiom of the type $\standbs\kb$ can be equivalently rewritten into the axiom set $\{\standbs\phi \mid \phi \in \kb\}$. While something alike is not immediately possible for axioms of the type $\standds\kb$, our normalization rule for diamond-preceded axioms can be lifted and thus $\standds\kb$ can be rewritten to $\standb{s'}\kb$ (and further to $\{\standb{s'}\phi \mid \phi \in \kb\}$) upon introducing a fresh standpoint name $\sp$ and asserting $\sp \preceq \st$. Thus standpoint-modality-annotated knowledge bases come essentially for free in $\SEL$. In fact, we already made use of this modelling feature in Axiom~\ref{formula:instance-patient1} and Axiom~\ref{formula:instance-tumour} of our initial example.

Moreover, we are confident that, as opposed to nominal concepts, other modelling features of OWL~2~EL can be added to $\SEL$ without harming tractability. These include complex role inclusions (also called role-chain axioms) such as $\pred{FindingSite} \circ \pred{PartOf} \sqsubseteq \pred{FindingSite}$, and the self-concept as in $ \pred{ApoptoticCell} \sqsubseteq \exists \pred{Destroys}.\mathsf{Self}$. It also seems plausible that the sharpening statements can be extended to incorporate intersection and disjointness of standpoints.

Beyond exploring the tractability boundaries, next endeavours include to investigate feasible strategies for developing a $\SEL$ reasoner. Options include
\begin{itemize}
    \item to implement our tableau algorithm from scratch or by modifying existing open-source tableaux systems,
    \item to design a deduction calculus over normalised axioms that can be translated into a datalog program, akin to the approach of \citeauthor{Krotzsch10OWLELreasoning} [\citeyear{Krotzsch10OWLELreasoning}], then utilizing a state-of-the-art datalog engine like VLog \cite{vlog16}, or
    \item to find a reduction to reasoning in standpoint-free (\textsc{PTime} extensions of) $\EL$ that is supported by existing reasoners (such as ELK \cite{KazakovKS14} or Snorocket \cite{Metke-JimenezL13}).
\end{itemize}
With one or several reasoners in place, appropriate experiments will be designed and conducted to assess practical feasibility and scalability.

Beyond the $\EL$ family, further popular and computationally lightweight formalisms exist, such as the tractable profiles OWL~2~RL and OWL~2~QL \cite{owl2-profiles}. It would be interesting to investigate options to extend these by standpoint reasoning without sacrificing tractability. More generally, we intend to research the effect of adding standpoints to KR languages -- light- or heavyweight -- in terms of computational properties and expressivity as well as avenues for implementing efficient reasoners for them.
%-Easy extensions that add paragraphing such 
% - as diamonds in front of full ABox

%Future work:
%- Standpoint conjunction and disjointness in axiom statements.
%- Role chain axioms.
%- Implementations and optimisation, hint at complexity optimal deduction calculus. 